\RequirePackage{fix-cm}
\documentclass{article}
\usepackage{iclr2027_conference,times}
\usepackage{amsmath,amssymb,amsthm}
\usepackage{array}
\usepackage{enumitem}
\usepackage{float}
\usepackage{tcolorbox}
\usepackage{tikz}
\usetikzlibrary{arrows.meta,calc}
\usepackage[T1,OT1]{fontenc}
\usepackage{fvextra}
\usepackage[varqu,varl]{zi4}

\newsavebox{\figuretermbox}
\setlist[itemize]{labelindent=0pt,leftmargin=20pt,topsep=1pt,itemsep=1pt,parsep=0pt,partopsep=0pt}

\ifdefined\XeTeXversion
  \errmessage{Build with pdflatex: xelatex silently replaces the ICLR Times fonts}
\fi
\ifdefined\directlua
  \errmessage{Build with pdflatex: lualatex silently replaces the ICLR Times fonts}
\fi

\usepackage{hyperref}

\usepackage{amsmath,amsfonts,bm}

\def\eqref#1{equation~\ref{#1}}

\def\1{\bm{1}}

\DeclareMathAlphabet{\mathsfit}{\encodingdefault}{\sfdefault}{m}{sl}
\SetMathAlphabet{\mathsfit}{bold}{\encodingdefault}{\sfdefault}{bx}{n}

\RequirePackage{xcolor}

\definecolor{todoRed}{HTML}{C1121F}
\definecolor{authorLuis}{HTML}{0057B8}
\definecolor{authorNafiseh}{HTML}{00838F}
\definecolor{authorTina}{HTML}{6F42C1}
\definecolor{authorOmid}{HTML}{007A5E}
\definecolor{authorKeyvan}{HTML}{A64B00}
\definecolor{authorHannah}{HTML}{AD1457}

\RequirePackage{amsmath}
\RequirePackage{amssymb}
\RequirePackage{amsthm}
\RequirePackage{booktabs}
\RequirePackage{aliascnt}
\RequirePackage[nameinlink,noabbrev]{cleveref}

\theoremstyle{definition}
\newtheorem{definition}{Definition}
\newaliascnt{construction}{definition}

\aliascntresetthe{construction}
\newaliascnt{assumption}{definition}

\aliascntresetthe{assumption}

\theoremstyle{plain}
\newtheorem{theorem}{Theorem}
\newaliascnt{proposition}{theorem}

\aliascntresetthe{proposition}
\newaliascnt{lemma}{theorem}

\aliascntresetthe{lemma}
\newaliascnt{corollary}{theorem}

\aliascntresetthe{corollary}

\theoremstyle{remark}
\newaliascnt{remark}{definition}

\aliascntresetthe{remark}
\newaliascnt{example}{definition}

\aliascntresetthe{example}

\crefname{definition}{definition}{definitions}
\Crefname{definition}{Definition}{Definitions}
\crefname{construction}{construction}{constructions}
\Crefname{construction}{Construction}{Constructions}
\crefname{assumption}{assumption}{assumptions}
\Crefname{assumption}{Assumption}{Assumptions}
\crefname{theorem}{theorem}{theorems}
\Crefname{theorem}{Theorem}{Theorems}
\crefname{proposition}{proposition}{propositions}
\Crefname{proposition}{Proposition}{Propositions}
\crefname{lemma}{lemma}{lemmas}
\Crefname{lemma}{Lemma}{Lemmas}
\crefname{corollary}{corollary}{corollaries}
\Crefname{corollary}{Corollary}{Corollaries}
\crefname{remark}{remark}{remarks}
\Crefname{remark}{Remark}{Remarks}
\crefname{example}{example}{examples}
\Crefname{example}{Example}{Examples}
\crefname{enumi}{part}{parts}
\Crefname{enumi}{Part}{Parts}
\crefname{appendix}{Appendix}{Appendices}
\Crefname{appendix}{Appendix}{Appendices}

\newcommand{\Agent}{\mathcal{A}}
\newcommand{\Env}{\mathcal{E}}
\newcommand{\AskPolicy}{\pi_{\mathrm{ask}}}

\newcommand{\AskPolicyFull}{\pi_{\mathrm{full}}}
\newcommand{\AskPolicyNo}{\pi_{\mathrm{none}}}
\newcommand{\AskPolicyAdvFull}{\AskPolicyFull^{\mathrm{adv}}}
\newcommand{\AskPolicyAdvNo}{\AskPolicyNo^{\mathrm{adv}}}
\newcommand{\AskPolicySpecRand}{\AskPolicy^{Q}}

\newcommand{\SolutionPolicy}{\pi_{\mathrm{sol}}}
\newcommand{\Oracle}{\mathcal{O}}
\newcommand{\OraclePerf}{\Oracle^{\star}}
\newcommand{\Verif}{V}
\newcommand{\Text}{\mathrm{Text}}
\newcommand{\traj}{\tau}
\newcommand{\hist}{\mathit{hist}}
\newcommand{\qst}{\mathit{qst}}
\newcommand{\ans}{\mathit{ans}}
\newcommand{\cred}{\mathit{cred}}
\newcommand{\LeakJ}{\mathrm{Leak}_{J}}
\newcommand{\LeakRec}{\mathrm{Leak}_{\Rcl}}
\newcommand{\DiagResp}{\Delta J_{\mathrm{resp}}}
\newcommand{\DiagFull}{\Delta J_{\mathrm{full}}}

\newcommand{\DiagRecNone}{\Delta\Rcl_{\mathrm{none}}}
\newcommand{\DiagSpec}{\Delta\Rcl_{\mathrm{spec}}}
\newcommand{\DiagSpecQ}{\DiagSpec^{Q}}
\newcommand{\Amb}{\mathrm{Amb}}
\newcommand{\res}{\rho}

\newcommand{\mch}{\mathrm{match}}
\newcommand{\rsp}{\mathrm{resp}}
\newcommand{\mchPerf}{\mch^{\star}}
\newcommand{\rspPerf}{\rsp^{\star}}

\newcommand{\Prc}{\mathrm{Prec}}
\newcommand{\Rcl}{\mathrm{Rec}}

\newcommand{\TV}{\mathrm{TV}}

\newcommand{\JFull}{J_{\mathrm{full}}}
\newcommand{\JNo}{J_{\mathrm{none}}}

\newcommand{\benchmarkname}{AmbiTab}
\newcommand{\benchmarkshort}{AT}

\tcbuselibrary{skins}
\definecolor{theoryPanelBg}{HTML}{F8F8F5}
\colorlet{theoryPanelEdge}{black!40}
\tcbset{
  theorypanel style/.style={
    enhanced,
    colback=theoryPanelBg,
    frame hidden,
    borderline={0.4pt}{0pt}{theoryPanelEdge, dash pattern=on 1.3pt off 1.1pt},
    boxrule=0pt,
    arc=1mm,
    outer arc=1mm,
    left=1.5mm,
    right=1.5mm,
    top=1mm,
    bottom=1mm,
    boxsep=0pt,
    before skip=6pt,
    after skip=6pt,
    fontupper=\normalsize,
  },
}
\newcommand{\theorypanelsetup}{\linespread{1.05}\selectfont\setlist[itemize,enumerate]{labelindent=0pt,leftmargin=*,topsep=2pt,itemsep=3pt,parsep=0pt,partopsep=0pt}\setlength{\abovedisplayskip}{4pt}\setlength{\belowdisplayskip}{4pt}}
\newtcolorbox{theorypanel}[1]{
  theorypanel style,
  coltitle=black,
  fonttitle=\bfseries\scshape,
  title=#1,
  detach title,
  before upper={\tcbtitle\par\theorypanelsetup},
}
\newtcolorbox{statementpanel}{
  theorypanel style,
  before upper={\theorypanelsetup},
}

\newif\ifshowaskeradditions
\showaskeradditionstrue
\newcommand{\askeradditioncolor}{\ifshowaskeradditions\color{authorTina}\fi}

\usepackage{url}
\usepackage{graphicx}

\hypersetup{hidelinks}
\title{Conversational Task Disambiguation over Tabular Data: Leakage-Aware Formulation, Benchmark Suite, and Training}

\author{\parbox[t]{\dimexpr\textwidth-2\tabcolsep\relax}{%
\centering\normalfont\normalsize
\textbf{Nafiseh Ghoroghchian}\thanks{Contributed equally and are joint first authors.}\quad \textbf{Luis Scoccola}\footnotemark[1]\quad \textbf{Tina Sedaghat}\\[0.35em]
\textbf{Omid Vaheb}\quad \textbf{Hannah Chen}\thanks{This work was done during her internship at Layer 6 AI.}\quad \textbf{Dino D'Agostino}\quad \textbf{Keyvan Golestan}\\[0.6em]
Layer 6 AI
}}
\iclrfinalcopy

\begin{document}

\maketitle
\lhead{}
\renewcommand{\headrulewidth}{0pt}

\begin{abstract}
Conversational task disambiguation over tabular data uses dialogue to resolve missing information about a user's intended task before producing a solution over tables or databases.
Existing evaluation and training lack a leakage-aware foundation.
Task success mixes the agent's disambiguation and solution-generation capabilities and can also reflect \emph{oracle leakage}, that is, information that a user simulator reveals beyond what a real user would.
Existing datasets also lack a shared representation of ambiguities and access boundaries.
We introduce the notion of an \emph{ambiguous verifiable task}, which formalizes ambiguities and resolutions, decomposing the agent into an asking policy and a solution policy, and the environment into an oracle and verifier.
This framework provides baselines and metrics for evaluating task disambiguation separately from solution generation, formal definitions of oracle leakage, judge-free leakage diagnostics, and a training objective for the asking policy.
We instantiate the framework in text-to-SQL with \emph{\benchmarkname{}}, a benchmark suite that unifies six ambiguous datasets under a common representation specifying what the agent, oracle, and verifier may access.
We evaluate clarification strategies and oracle leakage, and train an asking policy with reinforcement learning.
The trained asker improves our disambiguation metrics on all six datasets and task success on five, and our leakage diagnostics measure how training affects oracle leakage.
\end{abstract}

\section{Introduction}

\subsection{Context and problem statement}
Language models are increasingly used to solve tasks over tables and relational databases~\citep{tian2025tableagents,wu2025tablebench,li2025treb}.
When the solution is an executable query, execution provides a trusted verifier.
Reinforcement Learning (RL) using such rewards~\citep{pourreza2025reasoningsql,ma2025sqlr1,lei2025reasoningtable} has improved single-turn systems on text-to-SQL benchmarks such as BIRD~\citep{li2023bird} and Spider~2.0~\citep{lei2025spider}.
These benchmarks assume that the task specification determines the intended solution.
Real task specifications often do not: for example ``list the top earners'' leaves open both how many earners to return and what counts as earnings (\Cref{fig:main}(c)).
We study \emph{conversational task disambiguation over tabular data}: an agent asks clarifying questions to resolve missing information about the intended task before producing a solution over tables or databases.
We focus empirically on text-to-SQL, where clarification identifies the user's intended interpretation and the solution policy generates a corresponding query.
Conversational benchmarks such as SParC and CoSQL~\citep{yu2019sparc,yu2019cosql} established the value of evaluating complete interactions, and interactive benchmarks such as BIRD-Interact~\citep{huo2025birdinteract} show that this type of ambiguous task remains difficult even for capable models.

Evaluating and training conversational task disambiguation requires a leakage-aware foundation, which current practice lacks, for three reasons.
First, success rate, the natural metric, depends on both asking and solving: a failure may reflect a missing question or a wrong solution, and a success may reflect a task that needs no disambiguation.
Second, at scale the user is replaced by an oracle, often a language model, and an agent selected or trained against it may exploit its errors, as it might exploit a misspecified reward~\citep{skalse2022rewardhacking,pan2022rewardmisspecification}, for example by eliciting details of the reference solution that no real user would supply.
We call this \emph{oracle leakage}; prior work separates, constrains, or audits user simulators against it~\citep{jain2025autoeval,huo2025birdinteract,sullutrone2026abiss}, and diagnosing it with a language model judge adds another untrusted component.
Both problems carry over to training: a success-based reward can attribute gains to the asker that actually arise from the solver’s capabilities or from oracle leakage.
Third, existing datasets lack a shared representation of ambiguities, resolutions, and what the agent, the oracle, and the verifier may access.

We address these difficulties with a formalism that makes explicit what is trusted and what the agent, the oracle, and the verifier may each access.
We use it to address the following questions:
\begin{itemize}
	\item[\textbf{Q1.}] How can task disambiguation be evaluated separately from solution generation?
	\item[\textbf{Q2.}] How can oracle leakage be formally defined and diagnosed without relying on an untrusted language model judge?
	\item[\textbf{Q3.}] How can an asking policy be trained within the same formalism, and how does training affect task success rate, asking behavior, and leakage?
\end{itemize}

\subsection{Contributions}

\begin{figure}[!t]
\centering
\includegraphics[width=1.0\linewidth]{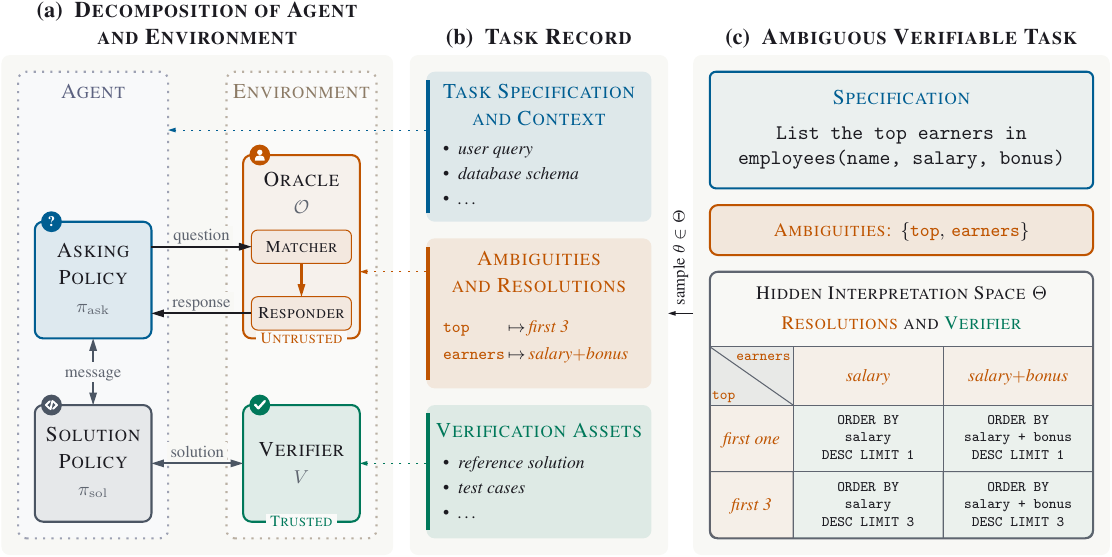}
\caption{
Agent and environment decompositions, task record, and example of a text-to-SQL AVT with two ambiguities and four interpretations.}
\label{fig:main}
\end{figure}

We introduce the notion of an \emph{Ambiguous Verifiable Task} (AVT) (\Cref{def:avt}), a general representation of underspecified tasks with verifiable solutions.
An AVT consists of a public task specification, a probability space of hidden interpretations, a finite set of ambiguities and their resolutions, and a verifier (\Cref{fig:main}c); nothing in the definition is specific to tables or SQL.
We study conversational task disambiguation through an ask-then-solve protocol for AVTs and instantiate it in text-to-SQL over relational databases.
In this setup, the agent consists of an \emph{asking policy} (or \emph{asker}) that asks the oracle questions and a \emph{solution policy} (or \emph{solver}) that submits a solution.
The environment consists of an \emph{oracle}, whose \emph{matcher} credits the ambiguities each question addresses and whose \emph{responder} returns information about their resolutions, and a trusted \emph{verifier} that scores the solution.
Because only the asker interacts with the oracle, we can vary the asker while holding the solver fixed.
We address questions \textbf{Q1}-\textbf{Q3} through three main contributions:

\begin{enumerate}[leftmargin=20pt]
    \item \textbf{Theory.}
    We develop a leakage-aware formalism for AVTs and instantiate it in an agent-environment interaction protocol (\Cref{sec:theory}).
    For \textbf{Q1}, it separates asking from solving through no-disambiguation and full-disambiguation baselines and asking precision and recall.
    Splitting the oracle into a matcher and a responder, as some user simulators already do~\citep{huo2025birdinteract,sullutrone2026abiss,zhang2026pibench,trinh2026hilbench}, separates two leakage channels: crediting ambiguities that a question does not address, and revealing more than the credited resolutions.
    For \textbf{Q2}, we use our framework to give mathematical formulas quantifying the two types of leakage.
    The formalism also introduces judge-free, computable diagnostics and proves how they witness or approximate leakage under explicit assumptions.

    \item \textbf{\benchmarkname{} Benchmark Suite.}
    We introduce \benchmarkname{} for conversational task disambiguation over tabular data, comprising more than 8{,}000 task records across six ambiguous text-to-SQL datasets (\Cref{section:data}).
    AmbiTab-Bird and \benchmarkname{}-EW-Bird derive from BIRD~\citep{li2023bird}; \benchmarkname{}-BIRD-Interact and \benchmarkname{}-Ambrosia adapt BIRD-Interact~\citep{huo2025birdinteract} and Ambrosia~\citep{saparina2024ambrosia}; and \benchmarkname{}-AmbiQT and \benchmarkname{}-TACO adapt AmbiQT~\citep{bhaskar2023benchmarking} and TACO~\citep{deng2026taco}.
    All six share a record that separates the public specification, the resolutions available to the oracle, and the private reference used by the verifier (\Cref{fig:main}b), supporting the metrics for \textbf{Q1} and, with the reference-response and sampling procedures in \Cref{sec:theory-estimation}, the diagnostics for \textbf{Q2}.
    AmbiTab-Bird and \benchmarkname{}-EW-Bird also include training splits relevant to \textbf{Q3}.

    \item \textbf{Experiments.}
    We address \textbf{Q1} by comparing no-disambiguation, evaluated-asker, and full-disambiguation conditions and reporting both success and asking metrics (\Cref{sec:results-ladder}).
    Reporting all three success rates together with macro precision and recall distinguishes the benefit of clarification from gaps in question quality.
    For \textbf{Q2}, we apply the diagnostics to deterministic and LLM oracles and test them on a responder that reads the reference SQL (\Cref{sec:results-oracle-comparison}).
    For \textbf{Q3}, we train the asking policy with a multi-turn GRPO variant with DAPO-style asymmetric clipping and zero-advantage group filtering~\citep{shao2024grpo,yu2025dapo} on AmbiTab-Bird, evaluate its coverage and task success with a fixed solution policy across datasets, and compare training oracles and reward objectives (\Cref{sec:results-rl}).
\end{enumerate}

\subsection{Related Work}
\label{sec:related}

\paragraph{Conversational task disambiguation.}
Ambiguous text-to-SQL benchmarks represent alternative readings~\citep{bhaskar2023benchmarking,saparina2024ambrosia,deng2026taco}.
BIRD-Interact and ABISS evaluate clarification with simulated users~\citep{huo2025birdinteract,sullutrone2026abiss}, whereas DySQL-Bench studies follow-up requests in database exploration~\citep{sun2026rethinking}.
Information-gain methods select questions that distinguish interpretations~\citep{kobalczyk2025active,qiu2025interactive}; IN3, AskBench, and $\pi$-Bench assess coverage of missing information~\citep{qian2024in3,zhao2026askbench,zhang2026pibench}.
We focus on eliciting the missing information that resolves the current task's intended interpretation.
Our task representation makes task ambiguities, resolutions, and oracle access explicit, with text-to-SQL as the evaluated instantiation over tabular data.
See \Cref{sec:appendix-related} for broader comparisons of benchmarks, clarification methods, and reward formulations.

\paragraph{Oracle reliability.}
Prior work restricts solution access or audits simulated responses~\citep{jain2025autoeval,huo2025birdinteract,sullutrone2026abiss}, but simulator fidelity and policy overfitting remain concerns~\citep{seshadri2026lost,yu2026frozen}.
We define leakage relative to a perfect oracle and derive judge-free diagnostics for undeserved question credit and excess response information.

\paragraph{Learning to disambiguate.}
RL for SQL generation and table reasoning directly rewards solution correctness~\citep{pourreza2025reasoningsql,ma2025sqlr1,lei2025reasoningtable}.
ACT and SGP learn clarification policies~\citep{chen2025act,berant2025sgp}, while IntentRL learns to enumerate interpretations and SQL queries in one response~\citep{saparina2026intentrl}.
InfoPO combines information gain with outcome rewards~\citep{kong2026infopo}; HiL-Bench combines blocker coverage with per-question bonuses and penalties~\citep{trinh2026hilbench}.
Our question-F1 reward balances ambiguity coverage and question efficiency without additional shaping coefficients.
Training improves downstream accuracy through better disambiguation.

\section{Theory}
\label{sec:theory}

\paragraph{Agent, environment, and oracle leakage.}
We model the interactions between an agent $\Agent$ and an environment $\Env$ as a stochastic process from which we sample trajectories $\traj \sim (\Agent; \Env)$.
Both $\Agent$ and $\Env$ are conditional distributions: $\Agent$ models the probability distribution of the next action given the observable history, and $\Env$ that of the next state given the history and the action taken by $\Agent$.
The sets of states and actions depend on the setup.

We are interested in situations where the environment is determined by two kinds of distributions: (1) Trusted distributions, such as a programmatic verifier $\Verif$,
which could even be deterministic.
(2) Untrusted distributions, such as an \emph{oracle} $\Oracle$ implemented by a language model, intended to approximate a \emph{perfect oracle} $\OraclePerf$, which could be defined to be, e.g., the behavior of a user base.

The agent is usually evaluated using its \emph{success rate} $J(\Agent; \Env) := \mathbb{E}_{\traj \sim (\Agent; \Env)} [\Verif(\traj)]$.
However, maximizing $J(\Agent; \Env)$ can lead the agent to exploit deviations of $\Oracle$ from $\OraclePerf$,
e.g., by eliciting information from $\Oracle$ that $\OraclePerf$ would not provide.
We refer to any such deviation as \emph{oracle leakage}.

\paragraph{Ambiguous verifiable tasks.}
We represent task specifications, ambiguities, clarification questions, responses, and solutions as text; we write $\Text$ for this set.
Task disambiguation elicits information about the intended interpretation of a specification.
In our text-to-SQL instantiation, this interpretation determines what the user intends to compute over the database, and the solution is a corresponding SQL query (\Cref{fig:main}c).

\begin{statementpanel}
\begin{definition}
\label{def:avt}
An \emph{ambiguous verifiable task} is a tuple $(x, \Theta, \Amb, \res, \Verif)$ consisting of
\begin{itemize}
	\itemsep0em
	\item A \emph{specification} $x\in\Text$;
	\item A probability space $\Theta$ of \emph{interpretations} of $x$;
	\item A finite set $\Amb\subseteq\Text$ of \emph{ambiguities};
	\item A \emph{resolution map} $\res\colon\Theta\times\Amb\to\Text$;
	\item A \emph{verifier} $\Verif\colon\Theta\times\Text\to\{0,1\}$.
\end{itemize}
\end{definition}
\end{statementpanel}

Proofs, technical assumptions, a table of notation, and other details are in \Cref{sec:appendix-proofs}.

\paragraph{Leakage-aware evaluation of agents on ambiguous verifiable tasks.}
\Cref{fig:avt-protocol} defines the
ask-then-solve protocol,
a simple interaction protocol between an agent and an environment associated to an AVT.
This setup can be generalized to other interaction protocols.\footnote{For context, note that an AVT equipped with an interaction protocol determines a partially observable Markov decision process: the state is the hidden interpretation $\theta$ plus the dialogue history, and the agent observes the history. Since $\theta$ is drawn once and fixed, this is the mixed-observability case of \citet{ong2010momdp}.}

\begin{figure}[ht]
\centering
\scalebox{1.0}{
\begingroup

\definecolor{avtInk}{HTML}{1C1B21}
\definecolor{avtMuted}{HTML}{5F6570}
\definecolor{avtPaper}{HTML}{F8F8F5}
\definecolor{avtAsk}{HTML}{005E92}    
\definecolor{avtSol}{HTML}{4B5563}    
\definecolor{avtOrc}{HTML}{BF5400}    
\definecolor{avtVer}{HTML}{00785A}    

\newcommand{\avtFTitle}{\fontsize{6.8}{8}\selectfont\scshape}
\newcommand{\avtFSym}{\fontsize{6.3}{7.4}\selectfont}
\newcommand{\avtFSmall}{\fontsize{5.8}{6.8}\selectfont}
\newcommand{\avtFKernel}{\fontsize{5.6}{6.6}\selectfont}

\tikzset{
  avtBox/.style={draw=#1, fill=#1!10!avtPaper, line width=0.7pt, rounded corners=1mm,
    minimum width=22mm, minimum height=11mm, align=center, inner xsep=0.8mm, inner ysep=0.8mm},
  avtGraphFlow/.style={draw=#1, line width=0.65pt, -{Triangle[length=1.35mm, width=1.0mm]}},
  avtEdgeValue/.style={fill=avtPaper, inner xsep=0.55mm, inner ysep=0.15mm, font=\avtFSmall, text=avtMuted},
  avtTopEdgeValue/.style={fill=avtPaper, inner xsep=0.05mm, inner ysep=0.1mm, font=\avtFKernel, text=avtMuted},
  avtThetaLabel/.style={anchor=south east, font=\avtFSmall, text=avtOrc},
}

\begin{tikzpicture}[x=1mm, y=1mm, every node/.style={inner sep=0pt, outer sep=0pt, text=avtInk}]
\path[use as bounding box] (0,7.0) rectangle (124.0,43.5);
\fill[avtPaper, rounded corners=1mm] (6.0,7.0) rectangle (110.0,43.5);

\node[anchor=north east, font=\avtFSmall, text=avtOrc] at (108.0,42.6) {$\theta\sim\Theta$};

\node[avtBox=avtAsk] (ask) at (21.0,30.0) {
  {\avtFTitle Asking Policy}\\[1.0mm]
  {\avtFSym\color{avtMuted} $\AskPolicy$}
};

\node[avtBox=avtOrc] (match) at (58.0,30.0) {
  {\avtFTitle Matcher}\\[1.0mm]
  {\avtFSym\color{avtMuted} $\mch$}
};
\node[avtThetaLabel] at ($(match.south east)+(-0.8mm,0.6mm)$) {$\theta$};

\node[avtBox=avtOrc] (resp) at (95.0,30.0) {
  {\avtFTitle Responder}\\[1.0mm]
  {\avtFSym\color{avtMuted} $\rsp$}
};
\node[avtThetaLabel] at ($(resp.south east)+(-0.8mm,0.6mm)$) {$\theta$};

\node[anchor=west, font=\avtFSym] (spec) at (0.5,30.0) {$x$};
\draw[avtGraphFlow=avtInk] ([xshift=0.8mm]spec.east) -- (ask.west);
\draw[avtGraphFlow=avtInk] (ask.east) -- (match.west)
  node[midway, above=0.55mm, avtTopEdgeValue] {$\qst_i$};
\draw[avtGraphFlow=avtOrc] (match.east) -- (resp.west)
  node[midway, above=0.55mm, avtTopEdgeValue] {$\cred_i$};
\draw[avtGraphFlow=avtInk, rounded corners=1.8pt] (resp.north) -- (95.0,39.5) -- (21.0,39.5)
  node[midway, fill=avtPaper, inner xsep=1.0mm, inner ysep=0.2mm,
    font=\avtFSmall, text=avtMuted] {$\hist_i$}
  -- (ask.north);

\node[avtBox=avtSol] (solve) at (58.0,15.0) {
  {\avtFTitle Solution Policy}\\[1.0mm]
  {\avtFSym\color{avtMuted} $\SolutionPolicy$}
};

\node[avtBox=avtVer] (verify) at (95.0,15.0) {
  {\avtFTitle Verifier}\\[1.0mm]
  {\avtFSym\color{avtMuted} $V$}
};
\node[avtThetaLabel] at ($(verify.south east)+(-0.8mm,0.6mm)$) {$\theta$};

\draw[avtGraphFlow=avtSol, rounded corners=1.8pt] (ask.south) -- (21.0,15.0) -- (solve.west);
\node[above=0.55mm, avtTopEdgeValue] at ({$(ask.east)!0.5!(match.west)$} |- solve.west) {$\hist_n$};
\draw[avtGraphFlow=avtSol] (solve.east) -- (verify.west)
  node[midway, above=0.55mm, avtTopEdgeValue] {$\mathit{solution}$};
\node[anchor=east, font=\avtFSym] (verdict) at (124.0,15.0) {$\{0,1\}$};
\draw[avtGraphFlow=avtVer] (verify.east) -- ([xshift=-0.8mm]verdict.west);

\end{tikzpicture}
\endgroup}
\caption{
The \emph{ask-then-solve protocol} (\Cref{sec:avt-protocol}).
Oracle leakage can arise from undeserved matcher credit or responses revealing more than the credited resolutions.}

\label{fig:avt-protocol}
\end{figure}
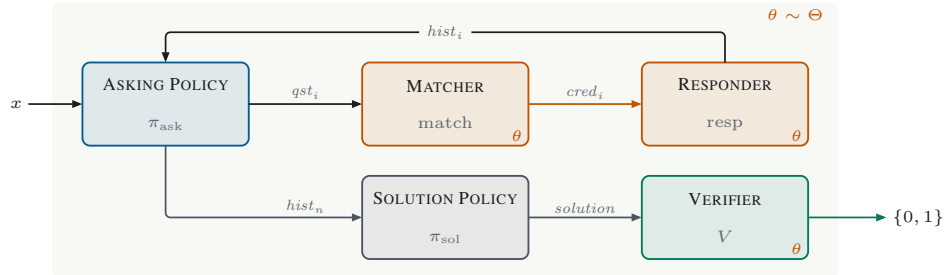

We now define success metrics\footnote{These parallel precision and recall in HiL-Bench~\citep{trinh2026hilbench}; in our protocol, these evaluate $\AskPolicy$ independently of $\SolutionPolicy$.
Our no- and full-disambiguation baselines, below,
parallel their blocked and full-information conditions.
Checkpoint coverage in AskBench~\citep{zhao2026askbench} is related.}, leakage, and leakage diagnostics for the ask-then-solve protocol associated to an AVT (\Cref{fig:avt-protocol}).
See \Cref{sec:theory-definitions} for formal definitions, and \Cref{sec:theory-estimation} for estimation.
The definitions and theoretical guarantees below concern this protocol.

Since we care about the interactions between the asker and the oracle, we fix a solution policy, a verifier, and a perfect oracle $\OraclePerf=(\mchPerf,\rspPerf)$.
We abbreviate $(\AskPolicy,\SolutionPolicy;\Oracle,\Verif)$ as $(\AskPolicy;\Oracle)$.

\begin{theorypanel}{Success metrics and baselines}
\begin{itemize}
\item \emph{Success rate}. $J(\AskPolicy;\Oracle):=\mathbb E[\Verif(\theta,\mathit{solution})]$.
\item \emph{Recall} and \emph{Precision}. In an episode with $n$ questions, the matcher credits a set $\cred_i\subseteq\Amb$ of ambiguities to question $i$; write $\cred_{\leq i}:=\bigcup_{j=1}^i\cred_j$, so that $\cred_{\leq 0}=\emptyset$.
\[
\Rcl(\AskPolicy;\Oracle):=\mathbb E\left[\frac{|\cred_{\leq n}|}{|\Amb|}\right] \;\;\;\;\;\;
\Prc(\AskPolicy;\Oracle):=\mathbb E\left[\frac{|\{i:\cred_i\not\subseteq\cred_{\leq i-1}\}|}{n}\;\middle|\;n\geq1\right].
\]
\item
\emph{Full-disambiguation} and \emph{no-disambiguation baselines}.
Fix a \emph{full-disambiguation policy} $\AskPolicyFull$, which acquires every ambiguity and asks no unproductive question ($\Rcl(\AskPolicyFull;\OraclePerf)=\Prc(\AskPolicyFull;\OraclePerf)=1$), and let $\AskPolicyNo$ stop without asking any questions.
\[
\JFull(\Oracle):=J(\AskPolicyFull;\Oracle)
\;\;\;\;\;\;
\JNo(\Oracle):=J(\AskPolicyNo;\Oracle).
\]
\end{itemize}
\end{theorypanel}

Recall that, by \emph{oracle leakage}, we mean any exploitation by the agent of deviations of $\Oracle$ from $\OraclePerf$.
This motivates the following definitions\footnote{$J$-leakage is a signed, fixed-policy analogue of policy-value gaps used to assess simulators~\citep{dann2026policyaware} and is related to generalized-gain leakage~\citep{alvim2012gain}. Here the reference is a task-defined perfect oracle, and a positive gap means extra task success enabled by the actual oracle.}:

\begin{theorypanel}{Leakage}
\begin{enumerate}[label=\arabic*.]
	\item \emph{$J$-leakage}. $\LeakJ(\AskPolicy;\Oracle) := J(\AskPolicy;\Oracle) - J(\AskPolicy;\OraclePerf)$.
	\item \emph{Recall leakage}. $\LeakRec(\AskPolicy;\mch) := \Rcl(\AskPolicy;\mch,\rspPerf)-\Rcl(\AskPolicy;\OraclePerf)$.
\end{enumerate}
\end{theorypanel}



The above notions of leakage cannot generally be computed directly.
For example, estimating $J$-leakage requires trajectories from $(\AskPolicy; \OraclePerf)$, which in practice can require an expensive back-and-forth between sampling from the asker and human annotators.


\begin{theorypanel}{Leakage diagnostics}
\begin{enumerate}[label=\arabic*.,itemsep=4pt]
\item
	\emph{Responder-leakage diagnostic} $\DiagResp$. Replaces the actual responder by the perfect responder while holding the asking policy and actual matcher fixed (\Cref{def:responder-leakage-diagnostic}).
\item
	\emph{Specification-randomization diagnostic} $\DiagSpecQ$.
	Recall after replacing the task specification seen by the asking policy with an independently sampled specification (\Cref{def:specification-randomization-diagnostic}).
\item
	\emph{Adversarial diagnostics} $\DiagFull$, $\DiagRecNone$.
	Evaluate $J$-leakage under an adversarial policy that covers all of the ambiguities, and recall leakage under one that covers none (\Cref{def:full-no-disambiguation-diagnostics}).
\end{enumerate}
\end{theorypanel}

Next is the main result, \Cref{thm:leakage-diagnostics}.
Informally:
\Cref{item:diagnostics-responder} says that the responder-leakage diagnostic approximates $J$-leakage when swapping the actual matcher for the perfect one results in a small change in the distribution of trajectories, as measured by the total variation distance $\TV$ (\Cref{sec:theory-conventions});
and \Cref{item:diagnostics-mixture} says that the specification-randomization diagnostic approximates recall leakage of the randomized policy when its perfect-oracle recall is small, and that a simple mixture model relates this to recall leakage of the original policy.
See \Cref{sec:proof-diagnostic-theorem} for the proof.

\begin{theorem}
\label{thm:leakage-diagnostics}
Fix an AVT, asking policy $\AskPolicy$, solution policy, and oracle $\Oracle=(\mch,\rsp)$.
\begin{enumerate}[label=\arabic*.,ref=\arabic*,leftmargin=*]
\item \label{item:diagnostics-responder}
Let $P_{\mch}$ and $P_{\mchPerf}$ be the joint distributions of the hidden interpretation $\theta$ and the complete visible trajectory $\traj$ under $(\AskPolicy;\mch,\rspPerf)$ and $(\AskPolicy;\OraclePerf)$, respectively.
Then
\[
\left|\,\LeakJ(\AskPolicy;\Oracle)-\DiagResp(\AskPolicy;\mch,\rsp)\,\right|
\,\le\, \TV(P_{\mch},P_{\mchPerf}).
\]

\item \label{item:diagnostics-mixture}
We have
$\DiagSpecQ(\AskPolicy;\mch)-\LeakRec(\AskPolicySpecRand;\mch)=\Rcl(\AskPolicySpecRand;\OraclePerf)$.

Moreover, if $\AskPolicy=p\,\pi_{\mathrm{inv}}+(1-p)\,\pi'$ for some $p\in[0,1]$ and asking policies $\pi_{\mathrm{inv}}$ and $\pi'$, where $\pi_{\mathrm{inv}}$ is invariant under replacing the task specification, then
\[
\left|\,\LeakRec(\AskPolicy;\mch)-\LeakRec(\AskPolicySpecRand;\mch)\,\right|
\,\le\, 2\,(1-p^B),
\]
where $B$ is the question budget and $Q$ is any distribution over task specifications.
\end{enumerate}
\end{theorem}

\section{\benchmarkname{} benchmark suite}
\label{section:data}

We introduce \benchmarkname{}, a suite of six ambiguous text-to-SQL datasets for evaluating conversational task disambiguation over tabular data.
Each dataset consists of task records with explicit ambiguity annotations, fixed resolutions, and verification assets.

\noindent\textbf{Task record.}
Abstractly, a task record implements a fixed interpretation of an ambiguous verifiable task $(x,\Theta,\Amb,\res,\Verif)$, so it can be seen as a sample of an AVT.
More precisely, given an interpretation $\theta\in\Theta$, its corresponding \emph{task record} is the tuple $(x,\Amb,\res_\theta,\Verif_\theta)$, where $\res_\theta(\alpha)=\res(\theta,\alpha)$ is the resolution map for an ambiguity $\alpha$ evaluated on the fixed interpretation, and $\Verif_\theta(s)=\Verif(\theta,s)$ is the verifier evaluated on the fixed interpretation.

In practice, this is implemented as follows (see \Cref{fig:main}(b) for an illustration):
\begin{itemize}
\item \textbf{Task specification and context.} The user request, and optionally the database schema, the column meanings, external knowledge, and other context.
\item \textbf{Ambiguities and resolutions.} A list pairing each ambiguity with its resolution as text.
Optionally, it might include the ambiguity type, a SQL snippet for grounding the responder, etc.
\item \textbf{Verification assets.} Usually, an executable database and a reference query.
Optionally, test cases, statements that prepare and restore the database, a policy for comparing result sets, etc.
\end{itemize}

\noindent\textbf{Dataset construction and details.}
We adapt BIRD~\citep{li2023bird}, BIRD-Interact~\citep{huo2025birdinteract}, Ambrosia~\citep{saparina2024ambrosia}, AmbiQT~\citep{bhaskar2023benchmarking}, and TACO~\citep{deng2026taco} (\Cref{sec:appendix-conversions}).
All datasets use SQLite except for \benchmarkname{}-BIRD-Interact, which uses PostgreSQL.
\Cref{tab:datasets} shows where each corpus's missing information comes from; \Cref{sec:appendix-data} gives the construction and oracle details.

\begin{table}[H]
\centering
\small
\setlength{\tabcolsep}{4pt}
\caption{The six \benchmarkname{} datasets.
\benchmarkname{}-Bird and \benchmarkname{}-EW-Bird counts combine their train and development splits.}
\label{tab:datasets}
\begin{tabular}{@{}>{\raggedright\arraybackslash}p{0.23\linewidth}>{\raggedright\arraybackslash}p{0.65\linewidth}r@{}}
\toprule
Dataset & Missing information and construction & Records \\
\midrule
\benchmarkname{}-Bird & Rewrite BIRD tasks with alternative readings; filter by SQL probes and a model judgment. & 664 \\
\benchmarkname{}-EW-Bird & Withhold the BIRD evidence field; retain rows where providing it improves performance for at least one reference model. & 2{,}537 \\\benchmarkname{}-BIRD-Interact & Normalize source ambiguity annotations; freeze audited, model-authored responses. & 600 \\
\benchmarkname{}-Ambrosia & Use each annotated reading and its native intent as a separate record. & 2{,}670 \\
\benchmarkname{}-AmbiQT & Build executable databases for paired SQL readings and assign each reading as a hidden intent. & 1{,}722 \\
\benchmarkname{}-TACO & Verbalize source SQL while withholding table or value choices. & 187 \\
\midrule
Total & & 8{,}380 \\
\bottomrule
\end{tabular}
\end{table}

\section{Experiments}
\label{sec:experiments}

We evaluate task disambiguation in text-to-SQL through task success, asking precision and recall, and oracle leakage across the six datasets in \Cref{section:data}, varying the oracle, asker prompt, and training configuration.

\subsection{Experimental setup}
\label{sec:results-setup}

\paragraph{Protocol.}
We use the ask-then-solve protocol of \Cref{fig:avt-protocol}, with Qwen3-8B~\citep{qwen2025qwen3} as the asker and a fixed Qwen3-Coder-30B-A3B-Instruct solver.
The asker clarifies the intended task through dialogue with the oracle; the solver uses the request and dialogue to produce SQL.
The default prompt, \emph{Checklist (ledger)}, enumerates potentially ambiguous terms and tracks their resolutions.
\Cref{sec:results-prompt} examines the precision--recall tradeoff across asker prompts.
We use the deterministic oracle by default and compare it with Qwen3-14B and Gemma-3-27B oracles (\Cref{sec:appendix-oracles}).

\paragraph{Evaluation.}
We report Success Rate (SR) for the evaluated asking policy and the no- and full-disambiguation baselines, alongside macro asking precision and recall (\Cref{sec:theory}).
\emph{Asking gain} is the SR increase over the same run's no-disambiguation baseline.
Results show means and sample standard deviations across evaluation seeds; \Cref{sec:appendix-inference,sec:asker-evaluation-details} give protocol and aggregation details.
Leakage diagnostics come from separate evaluations (\Cref{sec:leakage-results}).

\subsection{\benchmarkname{} disambiguation capacity}
\label{sec:results-ladder}

Under the deterministic oracle, full disambiguation substantially improves SR over no disambiguation on five datasets, but the evaluated asker closes only part of this gap (\Cref{fig:inference-overview}a; complete results in \Cref{tab:inference-oracles}, \Cref{sec:inference-oracles}).
On \benchmarkname{}-Bird, for example, SR rises from $26.3\%$ without disambiguation to $66.7\%$ with asking and $86.2\%$ with full disambiguation.
The remaining gap is 19.5 percentage points on \benchmarkname{}-Bird and 7.1 to 17.4 points on the other four datasets.
\benchmarkname{}-BIRD-Interact is solver-bound (\Cref{fig:inference-overview}a): its complex SQL makes SR largely dependent on the solver rather than the asker, with little change across disambiguation settings or oracles (\Cref{tab:inference-oracles}).

Recall remains below complete coverage on every dataset, while low precision indicates many questions that receive no new ambiguity credit (\Cref{fig:inference-overview}b,c; \Cref{tab:inference-oracles}).
These results apply the evaluation of \textbf{Q1}: reporting the three success rates together with precision and recall shows that disambiguation improves success, but the asker leaves ambiguities uncovered and asks unproductive questions.

\begin{figure}[t]
\centering
\includegraphics[width=0.8\linewidth]{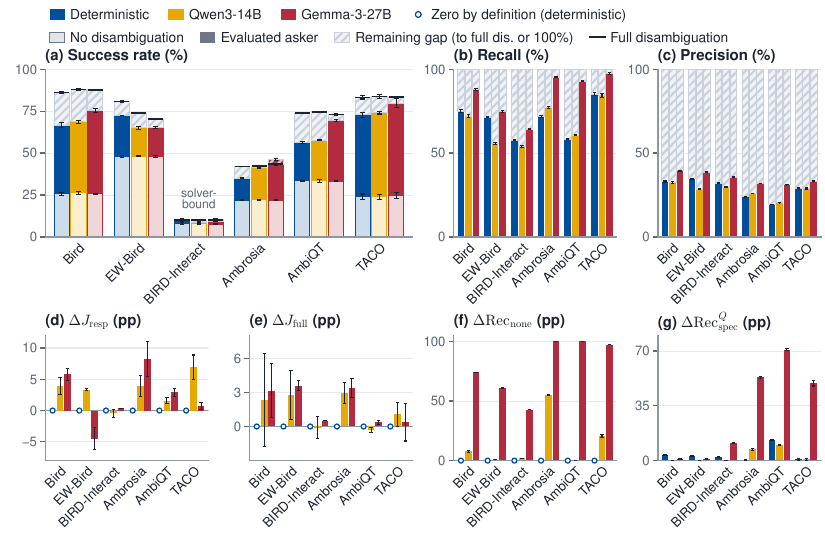}
\caption{The values of \Cref{tab:inference-oracles} (\Cref{sec:inference-oracles}) plus each oracle's own no-disambiguation SR, under the deterministic (navy), Qwen3-14B (yellow), and Gemma-3-27B (red) oracles.
\textbf{(a)} SR without disambiguation (pale), with the evaluated asker (solid), and with full disambiguation (black line); hatching is the gap the asker leaves.
\textbf{(b-c)} Asker recall and precision; hatching is the shortfall from 100\%.
\textbf{(d-g)} Leakage diagnostics in percentage points (Bird and EW-Bird dev, Ambrosia); the deterministic oracle's $\DiagResp$, $\DiagFull$, and $\DiagRecNone$ are zero by definition (open circles), and its $\DiagSpecQ$ is the reference recall.
Error bars: seed SDs over five seeds in (a-c) and three in (d-g).
Dataset labels omit the \benchmarkname{} prefix.}
\label{fig:inference-overview}
\end{figure}

\subsection{Success-leakage tradeoff across oracles}
\label{sec:results-oracle-comparison}
\label{sec:results-oracles}

The oracle comparison addresses \textbf{Q2}: Gemma-3-27B yields higher asker SR than Qwen3-14B on all six datasets, alongside higher values of the no-disambiguation recall-leakage and specification-randomization diagnostics, $\DiagRecNone$ and $\DiagSpecQ$ (\Cref{fig:inference-overview}f,g; \Cref{tab:inference-oracles}).
On \benchmarkname{}-AmbiQT, SR increases from $57.7\%$ to $69.3\%$, while $\DiagSpecQ$ rises from $9.9$ to $70.9$ percentage points.
$\DiagResp$ shows no consistent direction across datasets, and $\DiagFull$ differs by less than one point between the two oracles on every dataset (\Cref{fig:inference-overview}d,e).

These results motivate assessing success and leakage together.
Recall also depends on the oracle's credits and responses, including volunteered resolutions (\Cref{fig:inference-overview}b).
These comparisons capture the combined effects of matching, responses, and subsequent questions, without isolating leakage's contribution to SR.

An oracle exposed to the ground truth leaks more:
When the LLM responder also reads the full reference SQL, with the matcher unchanged, $\DiagFull$ rises on \benchmarkname{}-EW-Bird, \benchmarkname{}-Ambrosia, and \benchmarkname{}-AmbiQT under both oracles, by up to 10.4 points, and $\DiagResp$ rises by up to 7.1 points, with intervals excluding zero in six of ten dataset-oracle pairs; \benchmarkname{}-TACO shows no detectable change (\Cref{fig:sql-access}; \Cref{tab:leakage-sql-full-access}).
\Cref{sec:leakage-results} reports the complete diagnostics.

\begin{figure}[t]
\centering
\begin{minipage}[c]{0.525\linewidth}
\centering
\includegraphics{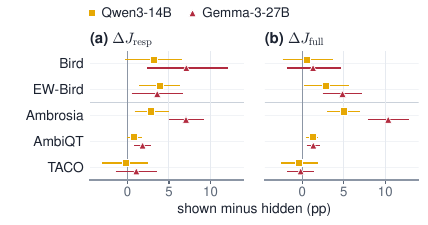}
\end{minipage}\hfill
\begin{minipage}[c]{0.45\linewidth}
\caption{Change in leakage when the LLM responder also reads the full reference SQL, with the matcher unchanged, shown minus hidden.
\textbf{(a)}~$\DiagResp$; \textbf{(b)}~$\DiagFull$.
Above the line, the default responder reads the SQL and the new runs hide it; below it, the new runs show it.
Error bars are paired 95\% AVT-cluster bootstrap intervals (\Cref{tab:leakage-sql-full-access}).
\benchmarkname{}-BIRD-Interact was not run, and dataset labels omit the \benchmarkname{} prefix.}
\label{fig:sql-access}
\end{minipage}
\end{figure}

\newcommand{\RLMean}[2]{\csname RLMean@#1@#2\endcsname}
\newcommand{\RLNative}[2]{\csname RLNative@#1@#2\endcsname}
\newcommand{\RLNativeBase}[2]{\csname RLNativeBase@#1@#2\endcsname}
\newcommand{\RLNativeLast}[2]{\csname RLNativeLast@#1@#2\endcsname}
\newcommand{\RLIteration}[1]{\csname RLIteration@#1\endcsname}
\newcommand{\RLTrainingIteration}[1]{\csname RLTrainingIteration@#1\endcsname}
\expandafter\def\csname RLIteration@det_f1\endcsname{40}
\expandafter\def\csname RLTrainingIteration@det_f1\endcsname{48}
\expandafter\def\csname RLNative@det_f1@SR\endcsname{60.6}
\expandafter\def\csname RLNative@det_f1@precision\endcsname{36.8}
\expandafter\def\csname RLNative@det_f1@recall\endcsname{83.3}
\expandafter\def\csname RLNative@det_f1@f1\endcsname{50.6}
\expandafter\def\csname RLNative@det_f1@system_leakage\endcsname{15.5}
\expandafter\def\csname RLNative@det_f1@user_leakage\endcsname{5.1}
\expandafter\def\csname RLNativeBase@det_f1@SR\endcsname{54.3}
\expandafter\def\csname RLNativeBase@det_f1@precision\endcsname{32.2}
\expandafter\def\csname RLNativeBase@det_f1@recall\endcsname{70.4}
\expandafter\def\csname RLNativeBase@det_f1@f1\endcsname{43.7}
\expandafter\def\csname RLNativeBase@det_f1@system_leakage\endcsname{12.4}
\expandafter\def\csname RLNativeBase@det_f1@user_leakage\endcsname{5.4}
\expandafter\def\csname RLNativeLast@det_f1@SR\endcsname{60.6}
\expandafter\def\csname RLNativeLast@det_f1@precision\endcsname{36.1}
\expandafter\def\csname RLNativeLast@det_f1@recall\endcsname{81.8}
\expandafter\def\csname RLNativeLast@det_f1@f1\endcsname{49.6}
\expandafter\def\csname RLNativeLast@det_f1@system_leakage\endcsname{15.6}
\expandafter\def\csname RLNativeLast@det_f1@user_leakage\endcsname{5.0}
\expandafter\def\csname RLIteration@det_sr\endcsname{16}
\expandafter\def\csname RLTrainingIteration@det_sr\endcsname{48}
\expandafter\def\csname RLNative@det_sr@SR\endcsname{57.5}
\expandafter\def\csname RLNative@det_sr@precision\endcsname{33.3}
\expandafter\def\csname RLNative@det_sr@recall\endcsname{73.0}
\expandafter\def\csname RLNative@det_sr@f1\endcsname{45.3}
\expandafter\def\csname RLNative@det_sr@system_leakage\endcsname{14.3}
\expandafter\def\csname RLNative@det_sr@user_leakage\endcsname{4.8}
\expandafter\def\csname RLNativeBase@det_sr@SR\endcsname{55.9}
\expandafter\def\csname RLNativeBase@det_sr@precision\endcsname{31.0}
\expandafter\def\csname RLNativeBase@det_sr@recall\endcsname{68.1}
\expandafter\def\csname RLNativeBase@det_sr@f1\endcsname{42.2}
\expandafter\def\csname RLNativeBase@det_sr@system_leakage\endcsname{14.9}
\expandafter\def\csname RLNativeBase@det_sr@user_leakage\endcsname{5.3}
\expandafter\def\csname RLNativeLast@det_sr@SR\endcsname{55.9}
\expandafter\def\csname RLNativeLast@det_sr@precision\endcsname{32.2}
\expandafter\def\csname RLNativeLast@det_sr@recall\endcsname{70.7}
\expandafter\def\csname RLNativeLast@det_sr@f1\endcsname{43.9}
\expandafter\def\csname RLNativeLast@det_sr@system_leakage\endcsname{13.9}
\expandafter\def\csname RLNativeLast@det_sr@user_leakage\endcsname{4.9}
\expandafter\def\csname RLIteration@llm_sr\endcsname{40}
\expandafter\def\csname RLTrainingIteration@llm_sr\endcsname{48}
\expandafter\def\csname RLNative@llm_sr@SR\endcsname{57.5}
\expandafter\def\csname RLNative@llm_sr@precision\endcsname{31.8}
\expandafter\def\csname RLNative@llm_sr@recall\endcsname{70.2}
\expandafter\def\csname RLNative@llm_sr@f1\endcsname{43.4}
\expandafter\def\csname RLNative@llm_sr@system_leakage\endcsname{14.8}
\expandafter\def\csname RLNative@llm_sr@user_leakage\endcsname{5.3}
\expandafter\def\csname RLNativeBase@llm_sr@SR\endcsname{52.8}
\expandafter\def\csname RLNativeBase@llm_sr@precision\endcsname{30.6}
\expandafter\def\csname RLNativeBase@llm_sr@recall\endcsname{68.1}
\expandafter\def\csname RLNativeBase@llm_sr@f1\endcsname{41.9}
\expandafter\def\csname RLNativeBase@llm_sr@system_leakage\endcsname{14.4}
\expandafter\def\csname RLNativeBase@llm_sr@user_leakage\endcsname{5.0}
\expandafter\def\csname RLNativeLast@llm_sr@SR\endcsname{55.1}
\expandafter\def\csname RLNativeLast@llm_sr@precision\endcsname{31.2}
\expandafter\def\csname RLNativeLast@llm_sr@recall\endcsname{68.8}
\expandafter\def\csname RLNativeLast@llm_sr@f1\endcsname{42.6}
\expandafter\def\csname RLNativeLast@llm_sr@system_leakage\endcsname{14.1}
\expandafter\def\csname RLNativeLast@llm_sr@user_leakage\endcsname{5.2}
\expandafter\def\csname RLIteration@llm_sr_bi\endcsname{8}
\expandafter\def\csname RLTrainingIteration@llm_sr_bi\endcsname{48}
\expandafter\def\csname RLNative@llm_sr_bi@SR\endcsname{6.7}
\expandafter\def\csname RLNative@llm_sr_bi@precision\endcsname{36.1}
\expandafter\def\csname RLNative@llm_sr_bi@recall\endcsname{66.8}
\expandafter\def\csname RLNative@llm_sr_bi@f1\endcsname{46.3}
\expandafter\def\csname RLNative@llm_sr_bi@system_leakage\endcsname{5.3}
\expandafter\def\csname RLNative@llm_sr_bi@user_leakage\endcsname{11.6}
\expandafter\def\csname RLNativeBase@llm_sr_bi@SR\endcsname{5.0}
\expandafter\def\csname RLNativeBase@llm_sr_bi@precision\endcsname{33.1}
\expandafter\def\csname RLNativeBase@llm_sr_bi@recall\endcsname{61.1}
\expandafter\def\csname RLNativeBase@llm_sr_bi@f1\endcsname{42.5}
\expandafter\def\csname RLNativeBase@llm_sr_bi@system_leakage\endcsname{6.5}
\expandafter\def\csname RLNativeBase@llm_sr_bi@user_leakage\endcsname{10.4}
\expandafter\def\csname RLNativeLast@llm_sr_bi@SR\endcsname{6.7}
\expandafter\def\csname RLNativeLast@llm_sr_bi@precision\endcsname{32.2}
\expandafter\def\csname RLNativeLast@llm_sr_bi@recall\endcsname{59.4}
\expandafter\def\csname RLNativeLast@llm_sr_bi@f1\endcsname{41.3}
\expandafter\def\csname RLNativeLast@llm_sr_bi@system_leakage\endcsname{5.4}
\expandafter\def\csname RLNativeLast@llm_sr_bi@user_leakage\endcsname{10.1}
\expandafter\def\csname RLMean@base@SR\endcsname{50.4}
\expandafter\def\csname RLMean@det_f1@SR\endcsname{54.6}
\expandafter\def\csname RLMean@det_sr@SR\endcsname{51.7}
\expandafter\def\csname RLMean@llm_sr@SR\endcsname{51.9}
\expandafter\def\csname RLMean@base@precision\endcsname{28.2}
\expandafter\def\csname RLMean@det_f1@precision\endcsname{31.8}
\expandafter\def\csname RLMean@det_sr@precision\endcsname{29.2}
\expandafter\def\csname RLMean@llm_sr@precision\endcsname{29.2}
\expandafter\def\csname RLMean@base@recall\endcsname{68.4}
\expandafter\def\csname RLMean@det_f1@recall\endcsname{77.4}
\expandafter\def\csname RLMean@det_sr@recall\endcsname{70.7}
\expandafter\def\csname RLMean@llm_sr@recall\endcsname{70.9}

\providecommand{\LeakDiag}[3]{\ifcsname LeakDiag@#1@#2@#3\endcsname\csname LeakDiag@#1@#2@#3\endcsname\else\PackageError{leakage}{undefined LeakDiag #1/#2/#3}{}\fi}
\expandafter\def\csname LeakDiag@Qwen3-14B@EW@resp\endcsname{3.36}
\expandafter\def\csname LeakDiag@Qwen3-14B@EW@full\endcsname{2.77}
\expandafter\def\csname LeakDiag@Qwen3-14B@EW@recnone\endcsname{0.82}
\expandafter\def\csname LeakDiag@Qwen3-14B@EW@spec\endcsname{0.22}
\expandafter\def\csname LeakDiag@Qwen3-14B@amb-BIRD@resp\endcsname{3.94}
\expandafter\def\csname LeakDiag@Qwen3-14B@amb-BIRD@full\endcsname{2.36}
\expandafter\def\csname LeakDiag@Qwen3-14B@amb-BIRD@recnone\endcsname{7.84}
\expandafter\def\csname LeakDiag@Qwen3-14B@amb-BIRD@spec\endcsname{0.00}
\expandafter\def\csname LeakDiag@Qwen3-14B@Ambrosia@resp\endcsname{4.00}
\expandafter\def\csname LeakDiag@Qwen3-14B@Ambrosia@full\endcsname{2.93}
\expandafter\def\csname LeakDiag@Qwen3-14B@Ambrosia@recnone\endcsname{54.87}
\expandafter\def\csname LeakDiag@Qwen3-14B@Ambrosia@spec\endcsname{7.24}
\expandafter\def\csname LeakDiag@Qwen3-14B@TACO@resp\endcsname{6.95}
\expandafter\def\csname LeakDiag@Qwen3-14B@TACO@full\endcsname{1.07}
\expandafter\def\csname LeakDiag@Qwen3-14B@TACO@recnone\endcsname{20.50}
\expandafter\def\csname LeakDiag@Qwen3-14B@TACO@spec\endcsname{0.71}
\expandafter\def\csname LeakDiag@Qwen3-14B@AmbiQT@resp\endcsname{1.61}
\expandafter\def\csname LeakDiag@Qwen3-14B@AmbiQT@full\endcsname{-0.31}
\expandafter\def\csname LeakDiag@Qwen3-14B@AmbiQT@recnone\endcsname{0.35}
\expandafter\def\csname LeakDiag@Qwen3-14B@AmbiQT@spec\endcsname{9.85}
\expandafter\def\csname LeakDiag@Qwen3-14B@BIRD-Interact@resp\endcsname{-0.44}
\expandafter\def\csname LeakDiag@Qwen3-14B@BIRD-Interact@full\endcsname{-0.11}
\expandafter\def\csname LeakDiag@Qwen3-14B@BIRD-Interact@recnone\endcsname{1.88}
\expandafter\def\csname LeakDiag@Qwen3-14B@BIRD-Interact@spec\endcsname{0.14}
\expandafter\def\csname LeakDiag@Gemma-3-27B@EW@resp\endcsname{-4.48}
\expandafter\def\csname LeakDiag@Gemma-3-27B@EW@full\endcsname{3.59}
\expandafter\def\csname LeakDiag@Gemma-3-27B@EW@recnone\endcsname{60.99}
\expandafter\def\csname LeakDiag@Gemma-3-27B@EW@spec\endcsname{1.10}
\expandafter\def\csname LeakDiag@Gemma-3-27B@amb-BIRD@resp\endcsname{5.77}
\expandafter\def\csname LeakDiag@Gemma-3-27B@amb-BIRD@full\endcsname{3.15}
\expandafter\def\csname LeakDiag@Gemma-3-27B@amb-BIRD@recnone\endcsname{74.09}
\expandafter\def\csname LeakDiag@Gemma-3-27B@amb-BIRD@spec\endcsname{0.84}
\expandafter\def\csname LeakDiag@Gemma-3-27B@Ambrosia@resp\endcsname{8.24}
\expandafter\def\csname LeakDiag@Gemma-3-27B@Ambrosia@full\endcsname{3.37}
\expandafter\def\csname LeakDiag@Gemma-3-27B@Ambrosia@recnone\endcsname{100.00}
\expandafter\def\csname LeakDiag@Gemma-3-27B@Ambrosia@spec\endcsname{53.25}
\expandafter\def\csname LeakDiag@Gemma-3-27B@TACO@resp\endcsname{0.71}
\expandafter\def\csname LeakDiag@Gemma-3-27B@TACO@full\endcsname{0.36}
\expandafter\def\csname LeakDiag@Gemma-3-27B@TACO@recnone\endcsname{97.33}
\expandafter\def\csname LeakDiag@Gemma-3-27B@TACO@spec\endcsname{49.47}
\expandafter\def\csname LeakDiag@Gemma-3-27B@AmbiQT@resp\endcsname{2.94}
\expandafter\def\csname LeakDiag@Gemma-3-27B@AmbiQT@full\endcsname{0.37}
\expandafter\def\csname LeakDiag@Gemma-3-27B@AmbiQT@recnone\endcsname{100.00}
\expandafter\def\csname LeakDiag@Gemma-3-27B@AmbiQT@spec\endcsname{70.94}
\expandafter\def\csname LeakDiag@Gemma-3-27B@BIRD-Interact@resp\endcsname{0.33}
\expandafter\def\csname LeakDiag@Gemma-3-27B@BIRD-Interact@full\endcsname{0.50}
\expandafter\def\csname LeakDiag@Gemma-3-27B@BIRD-Interact@recnone\endcsname{42.69}
\expandafter\def\csname LeakDiag@Gemma-3-27B@BIRD-Interact@spec\endcsname{11.04}
\expandafter\def\csname LeakDiag@Deterministic@EW@resp\endcsname{0}
\expandafter\def\csname LeakDiag@Deterministic@EW@full\endcsname{0}
\expandafter\def\csname LeakDiag@Deterministic@EW@recnone\endcsname{0}
\expandafter\def\csname LeakDiag@Deterministic@EW@spec\endcsname{2.91}
\expandafter\def\csname LeakDiag@Deterministic@amb-BIRD@resp\endcsname{0}
\expandafter\def\csname LeakDiag@Deterministic@amb-BIRD@full\endcsname{0}
\expandafter\def\csname LeakDiag@Deterministic@amb-BIRD@recnone\endcsname{0}
\expandafter\def\csname LeakDiag@Deterministic@amb-BIRD@spec\endcsname{3.50}
\expandafter\def\csname LeakDiag@Deterministic@Ambrosia@resp\endcsname{0}
\expandafter\def\csname LeakDiag@Deterministic@Ambrosia@full\endcsname{0}
\expandafter\def\csname LeakDiag@Deterministic@Ambrosia@recnone\endcsname{0}
\expandafter\def\csname LeakDiag@Deterministic@Ambrosia@spec\endcsname{0.50}
\expandafter\def\csname LeakDiag@Deterministic@TACO@resp\endcsname{0}
\expandafter\def\csname LeakDiag@Deterministic@TACO@full\endcsname{0}
\expandafter\def\csname LeakDiag@Deterministic@TACO@recnone\endcsname{0}
\expandafter\def\csname LeakDiag@Deterministic@TACO@spec\endcsname{0.98}
\expandafter\def\csname LeakDiag@Deterministic@AmbiQT@resp\endcsname{0}
\expandafter\def\csname LeakDiag@Deterministic@AmbiQT@full\endcsname{0}
\expandafter\def\csname LeakDiag@Deterministic@AmbiQT@recnone\endcsname{0}
\expandafter\def\csname LeakDiag@Deterministic@AmbiQT@spec\endcsname{13.09}
\expandafter\def\csname LeakDiag@Deterministic@BIRD-Interact@resp\endcsname{0}
\expandafter\def\csname LeakDiag@Deterministic@BIRD-Interact@full\endcsname{0}
\expandafter\def\csname LeakDiag@Deterministic@BIRD-Interact@recnone\endcsname{0}
\expandafter\def\csname LeakDiag@Deterministic@BIRD-Interact@spec\endcsname{2.13}
\providecommand{\LeakTrained}[4]{\ifcsname LeakTrained@#1@#2@#3@#4\endcsname\csname LeakTrained@#1@#2@#3@#4\endcsname\else\PackageError{leakage}{undefined LeakTrained #1@#2@#3@#4}{}\fi}
\providecommand{\LeakTrainedCI}[5]{\ifcsname LeakTrainedCI@#1@#2@#3@#4@#5\endcsname\csname LeakTrainedCI@#1@#2@#3@#4@#5\endcsname\else\PackageError{leakage}{undefined LeakTrainedCI #1@#2@#3@#4@#5}{}\fi}
\providecommand{\LeakTrainedDelta}[4]{\ifcsname LeakTrainedDelta@#1@#2@#3@#4\endcsname\csname LeakTrainedDelta@#1@#2@#3@#4\endcsname\else\PackageError{leakage}{undefined LeakTrainedDelta #1@#2@#3@#4}{}\fi}
\providecommand{\LeakTrainedDeltaAbs}[4]{\ifcsname LeakTrainedDeltaAbs@#1@#2@#3@#4\endcsname\csname LeakTrainedDeltaAbs@#1@#2@#3@#4\endcsname\else\PackageError{leakage}{undefined LeakTrainedDeltaAbs #1@#2@#3@#4}{}\fi}
\providecommand{\LeakTrainedRange}[4]{\ifcsname LeakTrainedRange@#1@#2@#3@#4\endcsname\csname LeakTrainedRange@#1@#2@#3@#4\endcsname\else\PackageError{leakage}{undefined LeakTrainedRange #1@#2@#3@#4}{}\fi}
\providecommand{\LeakTrainedCount}[4]{\ifcsname LeakTrainedCount@#1@#2@#3@#4\endcsname\csname LeakTrainedCount@#1@#2@#3@#4\endcsname\else\PackageError{leakage}{undefined LeakTrainedCount #1@#2@#3@#4}{}\fi}
\providecommand{\LeakTrainedDiff}[4]{\ifcsname LeakTrainedDiff@#1@#2@#3@#4\endcsname\csname LeakTrainedDiff@#1@#2@#3@#4\endcsname\else\PackageError{leakage}{undefined LeakTrainedDiff #1@#2@#3@#4}{}\fi}
\providecommand{\LeakTrainedDiffCount}[2]{\ifcsname LeakTrainedDiffCount@#1@#2\endcsname\csname LeakTrainedDiffCount@#1@#2\endcsname\else\PackageError{leakage}{undefined LeakTrainedDiffCount #1@#2}{}\fi}
\providecommand{\LeakTrainedClusters}[2]{\ifcsname LeakTrainedClusters@#1@#2\endcsname\csname LeakTrainedClusters@#1@#2\endcsname\else\PackageError{leakage}{undefined LeakTrainedClusters #1@#2}{}\fi}
\providecommand{\LeakTrainedInfo}[1]{\ifcsname LeakTrainedInfo@#1\endcsname\csname LeakTrainedInfo@#1\endcsname\else\PackageError{leakage}{undefined LeakTrainedInfo #1}{}\fi}
\expandafter\def\csname LeakTrained@Base@Deterministic@AmbiQT@resp\endcsname{0}
\expandafter\def\csname LeakTrained@Base@Deterministic@AmbiQT@spec\endcsname{13.09}
\expandafter\def\csname LeakTrained@Base@Deterministic@Ambrosia@resp\endcsname{0}
\expandafter\def\csname LeakTrained@Base@Deterministic@Ambrosia@spec\endcsname{0.50}
\expandafter\def\csname LeakTrained@Base@Deterministic@EW-BIRD dev@resp\endcsname{0}
\expandafter\def\csname LeakTrained@Base@Deterministic@EW-BIRD dev@spec\endcsname{2.91}
\expandafter\def\csname LeakTrained@Base@Deterministic@TACO@resp\endcsname{0}
\expandafter\def\csname LeakTrained@Base@Deterministic@TACO@spec\endcsname{0.98}
\expandafter\def\csname LeakTrained@Base@Deterministic@amb-BIRD dev@resp\endcsname{0}
\expandafter\def\csname LeakTrained@Base@Deterministic@amb-BIRD dev@spec\endcsname{3.50}
\expandafter\def\csname LeakTrained@Base@Gemma-3-27B@AmbiQT@resp\endcsname{2.94}
\expandafter\def\csname LeakTrained@Base@Gemma-3-27B@AmbiQT@spec\endcsname{70.94}
\expandafter\def\csname LeakTrained@Base@Gemma-3-27B@Ambrosia@resp\endcsname{8.24}
\expandafter\def\csname LeakTrained@Base@Gemma-3-27B@Ambrosia@spec\endcsname{53.25}
\expandafter\def\csname LeakTrained@Base@Gemma-3-27B@EW-BIRD dev@resp\endcsname{-4.48}
\expandafter\def\csname LeakTrained@Base@Gemma-3-27B@EW-BIRD dev@spec\endcsname{1.10}
\expandafter\def\csname LeakTrained@Base@Gemma-3-27B@TACO@resp\endcsname{0.71}
\expandafter\def\csname LeakTrained@Base@Gemma-3-27B@TACO@spec\endcsname{49.47}
\expandafter\def\csname LeakTrained@Base@Gemma-3-27B@amb-BIRD dev@resp\endcsname{5.77}
\expandafter\def\csname LeakTrained@Base@Gemma-3-27B@amb-BIRD dev@spec\endcsname{0.84}
\expandafter\def\csname LeakTrained@Base@Qwen3-14B@AmbiQT@resp\endcsname{1.61}
\expandafter\def\csname LeakTrained@Base@Qwen3-14B@AmbiQT@spec\endcsname{9.85}
\expandafter\def\csname LeakTrained@Base@Qwen3-14B@Ambrosia@resp\endcsname{4.00}
\expandafter\def\csname LeakTrained@Base@Qwen3-14B@Ambrosia@spec\endcsname{7.24}
\expandafter\def\csname LeakTrained@Base@Qwen3-14B@EW-BIRD dev@resp\endcsname{3.36}
\expandafter\def\csname LeakTrained@Base@Qwen3-14B@EW-BIRD dev@spec\endcsname{0.22}
\expandafter\def\csname LeakTrained@Base@Qwen3-14B@TACO@resp\endcsname{6.95}
\expandafter\def\csname LeakTrained@Base@Qwen3-14B@TACO@spec\endcsname{0.71}
\expandafter\def\csname LeakTrained@Base@Qwen3-14B@amb-BIRD dev@resp\endcsname{3.94}
\expandafter\def\csname LeakTrained@Base@Qwen3-14B@amb-BIRD dev@spec\endcsname{0.00}
\expandafter\def\csname LeakTrained@Det. F1@Deterministic@AmbiQT@resp\endcsname{0}
\expandafter\def\csname LeakTrained@Det. F1@Deterministic@AmbiQT@spec\endcsname{16.80}
\expandafter\def\csname LeakTrained@Det. F1@Deterministic@Ambrosia@resp\endcsname{0}
\expandafter\def\csname LeakTrained@Det. F1@Deterministic@Ambrosia@spec\endcsname{0.94}
\expandafter\def\csname LeakTrained@Det. F1@Deterministic@EW-BIRD dev@resp\endcsname{0}
\expandafter\def\csname LeakTrained@Det. F1@Deterministic@EW-BIRD dev@spec\endcsname{4.11}
\expandafter\def\csname LeakTrained@Det. F1@Deterministic@TACO@resp\endcsname{0}
\expandafter\def\csname LeakTrained@Det. F1@Deterministic@TACO@spec\endcsname{0.98}
\expandafter\def\csname LeakTrained@Det. F1@Deterministic@amb-BIRD dev@resp\endcsname{0}
\expandafter\def\csname LeakTrained@Det. F1@Deterministic@amb-BIRD dev@spec\endcsname{4.76}
\expandafter\def\csname LeakTrained@Det. F1@Gemma-3-27B@AmbiQT@resp\endcsname{-0.27}
\expandafter\def\csname LeakTrained@Det. F1@Gemma-3-27B@AmbiQT@spec\endcsname{81.98}
\expandafter\def\csname LeakTrained@Det. F1@Gemma-3-27B@Ambrosia@resp\endcsname{7.93}
\expandafter\def\csname LeakTrained@Det. F1@Gemma-3-27B@Ambrosia@spec\endcsname{58.99}
\expandafter\def\csname LeakTrained@Det. F1@Gemma-3-27B@EW-BIRD dev@resp\endcsname{-2.69}
\expandafter\def\csname LeakTrained@Det. F1@Gemma-3-27B@EW-BIRD dev@spec\endcsname{1.18}
\expandafter\def\csname LeakTrained@Det. F1@Gemma-3-27B@TACO@resp\endcsname{-1.25}
\expandafter\def\csname LeakTrained@Det. F1@Gemma-3-27B@TACO@spec\endcsname{50.45}
\expandafter\def\csname LeakTrained@Det. F1@Gemma-3-27B@amb-BIRD dev@resp\endcsname{6.30}
\expandafter\def\csname LeakTrained@Det. F1@Gemma-3-27B@amb-BIRD dev@spec\endcsname{0.14}
\expandafter\def\csname LeakTrained@Det. F1@Qwen3-14B@AmbiQT@resp\endcsname{0.79}
\expandafter\def\csname LeakTrained@Det. F1@Qwen3-14B@AmbiQT@spec\endcsname{9.66}
\expandafter\def\csname LeakTrained@Det. F1@Qwen3-14B@Ambrosia@resp\endcsname{2.87}
\expandafter\def\csname LeakTrained@Det. F1@Qwen3-14B@Ambrosia@spec\endcsname{7.68}
\expandafter\def\csname LeakTrained@Det. F1@Qwen3-14B@EW-BIRD dev@resp\endcsname{1.94}
\expandafter\def\csname LeakTrained@Det. F1@Qwen3-14B@EW-BIRD dev@spec\endcsname{0.22}
\expandafter\def\csname LeakTrained@Det. F1@Qwen3-14B@TACO@resp\endcsname{4.99}
\expandafter\def\csname LeakTrained@Det. F1@Qwen3-14B@TACO@spec\endcsname{1.52}
\expandafter\def\csname LeakTrained@Det. F1@Qwen3-14B@amb-BIRD dev@resp\endcsname{3.41}
\expandafter\def\csname LeakTrained@Det. F1@Qwen3-14B@amb-BIRD dev@spec\endcsname{0.14}
\expandafter\def\csname LeakTrained@LLM outcome@Deterministic@AmbiQT@resp\endcsname{0}
\expandafter\def\csname LeakTrained@LLM outcome@Deterministic@AmbiQT@spec\endcsname{12.54}
\expandafter\def\csname LeakTrained@LLM outcome@Deterministic@Ambrosia@resp\endcsname{0}
\expandafter\def\csname LeakTrained@LLM outcome@Deterministic@Ambrosia@spec\endcsname{0.62}
\expandafter\def\csname LeakTrained@LLM outcome@Deterministic@EW-BIRD dev@resp\endcsname{0}
\expandafter\def\csname LeakTrained@LLM outcome@Deterministic@EW-BIRD dev@spec\endcsname{3.20}
\expandafter\def\csname LeakTrained@LLM outcome@Deterministic@TACO@resp\endcsname{0}
\expandafter\def\csname LeakTrained@LLM outcome@Deterministic@TACO@spec\endcsname{0.62}
\expandafter\def\csname LeakTrained@LLM outcome@Deterministic@amb-BIRD dev@resp\endcsname{0}
\expandafter\def\csname LeakTrained@LLM outcome@Deterministic@amb-BIRD dev@spec\endcsname{3.64}
\expandafter\def\csname LeakTrained@LLM outcome@Gemma-3-27B@AmbiQT@resp\endcsname{2.38}
\expandafter\def\csname LeakTrained@LLM outcome@Gemma-3-27B@AmbiQT@spec\endcsname{73.11}
\expandafter\def\csname LeakTrained@LLM outcome@Gemma-3-27B@Ambrosia@resp\endcsname{8.18}
\expandafter\def\csname LeakTrained@LLM outcome@Gemma-3-27B@Ambrosia@spec\endcsname{53.56}
\expandafter\def\csname LeakTrained@LLM outcome@Gemma-3-27B@EW-BIRD dev@resp\endcsname{0.07}
\expandafter\def\csname LeakTrained@LLM outcome@Gemma-3-27B@EW-BIRD dev@spec\endcsname{1.25}
\expandafter\def\csname LeakTrained@LLM outcome@Gemma-3-27B@TACO@resp\endcsname{-1.60}
\expandafter\def\csname LeakTrained@LLM outcome@Gemma-3-27B@TACO@spec\endcsname{46.52}
\expandafter\def\csname LeakTrained@LLM outcome@Gemma-3-27B@amb-BIRD dev@resp\endcsname{7.61}
\expandafter\def\csname LeakTrained@LLM outcome@Gemma-3-27B@amb-BIRD dev@spec\endcsname{1.96}
\expandafter\def\csname LeakTrained@LLM outcome@Qwen3-14B@AmbiQT@resp\endcsname{1.66}
\expandafter\def\csname LeakTrained@LLM outcome@Qwen3-14B@AmbiQT@spec\endcsname{10.84}
\expandafter\def\csname LeakTrained@LLM outcome@Qwen3-14B@Ambrosia@resp\endcsname{5.06}
\expandafter\def\csname LeakTrained@LLM outcome@Qwen3-14B@Ambrosia@spec\endcsname{7.55}
\expandafter\def\csname LeakTrained@LLM outcome@Qwen3-14B@EW-BIRD dev@resp\endcsname{3.66}
\expandafter\def\csname LeakTrained@LLM outcome@Qwen3-14B@EW-BIRD dev@spec\endcsname{0.22}
\expandafter\def\csname LeakTrained@LLM outcome@Qwen3-14B@TACO@resp\endcsname{5.70}
\expandafter\def\csname LeakTrained@LLM outcome@Qwen3-14B@TACO@spec\endcsname{1.34}
\expandafter\def\csname LeakTrained@LLM outcome@Qwen3-14B@amb-BIRD dev@resp\endcsname{7.35}
\expandafter\def\csname LeakTrained@LLM outcome@Qwen3-14B@amb-BIRD dev@spec\endcsname{0.00}
\expandafter\def\csname LeakTrainedCI@Det. F1@Deterministic@AmbiQT@spec@hi\endcsname{25.26}
\expandafter\def\csname LeakTrainedCI@Det. F1@Deterministic@AmbiQT@spec@lo\endcsname{6.53}
\expandafter\def\csname LeakTrainedCI@Det. F1@Deterministic@Ambrosia@spec@hi\endcsname{1.02}
\expandafter\def\csname LeakTrainedCI@Det. F1@Deterministic@Ambrosia@spec@lo\endcsname{0.00}
\expandafter\def\csname LeakTrainedCI@Det. F1@Deterministic@EW-BIRD dev@spec@hi\endcsname{7.81}
\expandafter\def\csname LeakTrainedCI@Det. F1@Deterministic@EW-BIRD dev@spec@lo\endcsname{3.40}
\expandafter\def\csname LeakTrainedCI@Det. F1@Deterministic@TACO@spec@hi\endcsname{1.21}
\expandafter\def\csname LeakTrainedCI@Det. F1@Deterministic@TACO@spec@lo\endcsname{0.00}
\expandafter\def\csname LeakTrainedCI@Det. F1@Deterministic@amb-BIRD dev@spec@hi\endcsname{9.63}
\expandafter\def\csname LeakTrainedCI@Det. F1@Deterministic@amb-BIRD dev@spec@lo\endcsname{1.79}
\expandafter\def\csname LeakTrainedCI@Det. F1@Gemma-3-27B@AmbiQT@resp@hi\endcsname{0.91}
\expandafter\def\csname LeakTrainedCI@Det. F1@Gemma-3-27B@AmbiQT@resp@lo\endcsname{-1.43}
\expandafter\def\csname LeakTrainedCI@Det. F1@Gemma-3-27B@AmbiQT@spec@hi\endcsname{87.69}
\expandafter\def\csname LeakTrainedCI@Det. F1@Gemma-3-27B@AmbiQT@spec@lo\endcsname{72.03}
\expandafter\def\csname LeakTrainedCI@Det. F1@Gemma-3-27B@Ambrosia@resp@hi\endcsname{11.07}
\expandafter\def\csname LeakTrainedCI@Det. F1@Gemma-3-27B@Ambrosia@resp@lo\endcsname{4.76}
\expandafter\def\csname LeakTrainedCI@Det. F1@Gemma-3-27B@Ambrosia@spec@hi\endcsname{76.09}
\expandafter\def\csname LeakTrainedCI@Det. F1@Gemma-3-27B@Ambrosia@spec@lo\endcsname{57.38}
\expandafter\def\csname LeakTrainedCI@Det. F1@Gemma-3-27B@EW-BIRD dev@resp@hi\endcsname{0.60}
\expandafter\def\csname LeakTrainedCI@Det. F1@Gemma-3-27B@EW-BIRD dev@resp@lo\endcsname{-6.13}
\expandafter\def\csname LeakTrainedCI@Det. F1@Gemma-3-27B@EW-BIRD dev@spec@hi\endcsname{1.69}
\expandafter\def\csname LeakTrainedCI@Det. F1@Gemma-3-27B@EW-BIRD dev@spec@lo\endcsname{0.52}
\expandafter\def\csname LeakTrainedCI@Det. F1@Gemma-3-27B@TACO@resp@hi\endcsname{1.43}
\expandafter\def\csname LeakTrainedCI@Det. F1@Gemma-3-27B@TACO@resp@lo\endcsname{-3.74}
\expandafter\def\csname LeakTrainedCI@Det. F1@Gemma-3-27B@TACO@spec@hi\endcsname{50.93}
\expandafter\def\csname LeakTrainedCI@Det. F1@Gemma-3-27B@TACO@spec@lo\endcsname{50.33}
\expandafter\def\csname LeakTrainedCI@Det. F1@Gemma-3-27B@amb-BIRD dev@resp@hi\endcsname{12.34}
\expandafter\def\csname LeakTrainedCI@Det. F1@Gemma-3-27B@amb-BIRD dev@resp@lo\endcsname{1.04}
\expandafter\def\csname LeakTrainedCI@Det. F1@Gemma-3-27B@amb-BIRD dev@spec@hi\endcsname{0.62}
\expandafter\def\csname LeakTrainedCI@Det. F1@Gemma-3-27B@amb-BIRD dev@spec@lo\endcsname{0.00}
\expandafter\def\csname LeakTrainedCI@Det. F1@Qwen3-14B@AmbiQT@resp@hi\endcsname{1.88}
\expandafter\def\csname LeakTrainedCI@Det. F1@Qwen3-14B@AmbiQT@resp@lo\endcsname{-0.27}
\expandafter\def\csname LeakTrainedCI@Det. F1@Qwen3-14B@AmbiQT@spec@hi\endcsname{12.11}
\expandafter\def\csname LeakTrainedCI@Det. F1@Qwen3-14B@AmbiQT@spec@lo\endcsname{5.63}
\expandafter\def\csname LeakTrainedCI@Det. F1@Qwen3-14B@Ambrosia@resp@hi\endcsname{5.56}
\expandafter\def\csname LeakTrainedCI@Det. F1@Qwen3-14B@Ambrosia@resp@lo\endcsname{0.31}
\expandafter\def\csname LeakTrainedCI@Det. F1@Qwen3-14B@Ambrosia@spec@hi\endcsname{8.33}
\expandafter\def\csname LeakTrainedCI@Det. F1@Qwen3-14B@Ambrosia@spec@lo\endcsname{0.72}
\expandafter\def\csname LeakTrainedCI@Det. F1@Qwen3-14B@EW-BIRD dev@resp@hi\endcsname{4.86}
\expandafter\def\csname LeakTrainedCI@Det. F1@Qwen3-14B@EW-BIRD dev@resp@lo\endcsname{-0.90}
\expandafter\def\csname LeakTrainedCI@Det. F1@Qwen3-14B@EW-BIRD dev@spec@hi\endcsname{0.42}
\expandafter\def\csname LeakTrainedCI@Det. F1@Qwen3-14B@EW-BIRD dev@spec@lo\endcsname{0.00}
\expandafter\def\csname LeakTrainedCI@Det. F1@Qwen3-14B@TACO@resp@hi\endcsname{8.56}
\expandafter\def\csname LeakTrainedCI@Det. F1@Qwen3-14B@TACO@resp@lo\endcsname{1.60}
\expandafter\def\csname LeakTrainedCI@Det. F1@Qwen3-14B@TACO@spec@hi\endcsname{1.85}
\expandafter\def\csname LeakTrainedCI@Det. F1@Qwen3-14B@TACO@spec@lo\endcsname{1.43}
\expandafter\def\csname LeakTrainedCI@Det. F1@Qwen3-14B@amb-BIRD dev@resp@hi\endcsname{9.19}
\expandafter\def\csname LeakTrainedCI@Det. F1@Qwen3-14B@amb-BIRD dev@resp@lo\endcsname{-1.84}
\expandafter\def\csname LeakTrainedCI@Det. F1@Qwen3-14B@amb-BIRD dev@spec@hi\endcsname{0.62}
\expandafter\def\csname LeakTrainedCI@Det. F1@Qwen3-14B@amb-BIRD dev@spec@lo\endcsname{0.00}
\expandafter\def\csname LeakTrainedCI@LLM outcome@Deterministic@AmbiQT@spec@hi\endcsname{16.55}
\expandafter\def\csname LeakTrainedCI@LLM outcome@Deterministic@AmbiQT@spec@lo\endcsname{6.27}
\expandafter\def\csname LeakTrainedCI@LLM outcome@Deterministic@Ambrosia@spec@hi\endcsname{0.68}
\expandafter\def\csname LeakTrainedCI@LLM outcome@Deterministic@Ambrosia@spec@lo\endcsname{0.00}
\expandafter\def\csname LeakTrainedCI@LLM outcome@Deterministic@EW-BIRD dev@spec@hi\endcsname{3.31}
\expandafter\def\csname LeakTrainedCI@LLM outcome@Deterministic@EW-BIRD dev@spec@lo\endcsname{2.60}
\expandafter\def\csname LeakTrainedCI@LLM outcome@Deterministic@TACO@spec@hi\endcsname{0.77}
\expandafter\def\csname LeakTrainedCI@LLM outcome@Deterministic@TACO@spec@lo\endcsname{0.00}
\expandafter\def\csname LeakTrainedCI@LLM outcome@Deterministic@amb-BIRD dev@spec@hi\endcsname{6.51}
\expandafter\def\csname LeakTrainedCI@LLM outcome@Deterministic@amb-BIRD dev@spec@lo\endcsname{1.55}
\expandafter\def\csname LeakTrainedCI@LLM outcome@Gemma-3-27B@AmbiQT@resp@hi\endcsname{3.58}
\expandafter\def\csname LeakTrainedCI@LLM outcome@Gemma-3-27B@AmbiQT@resp@lo\endcsname{1.14}
\expandafter\def\csname LeakTrainedCI@LLM outcome@Gemma-3-27B@AmbiQT@spec@hi\endcsname{80.01}
\expandafter\def\csname LeakTrainedCI@LLM outcome@Gemma-3-27B@AmbiQT@spec@lo\endcsname{60.91}
\expandafter\def\csname LeakTrainedCI@LLM outcome@Gemma-3-27B@Ambrosia@resp@hi\endcsname{10.94}
\expandafter\def\csname LeakTrainedCI@LLM outcome@Gemma-3-27B@Ambrosia@resp@lo\endcsname{5.32}
\expandafter\def\csname LeakTrainedCI@LLM outcome@Gemma-3-27B@Ambrosia@spec@hi\endcsname{81.16}
\expandafter\def\csname LeakTrainedCI@LLM outcome@Gemma-3-27B@Ambrosia@spec@lo\endcsname{50.96}
\expandafter\def\csname LeakTrainedCI@LLM outcome@Gemma-3-27B@EW-BIRD dev@resp@hi\endcsname{3.14}
\expandafter\def\csname LeakTrainedCI@LLM outcome@Gemma-3-27B@EW-BIRD dev@resp@lo\endcsname{-3.06}
\expandafter\def\csname LeakTrainedCI@LLM outcome@Gemma-3-27B@EW-BIRD dev@spec@hi\endcsname{1.63}
\expandafter\def\csname LeakTrainedCI@LLM outcome@Gemma-3-27B@EW-BIRD dev@spec@lo\endcsname{0.00}
\expandafter\def\csname LeakTrainedCI@LLM outcome@Gemma-3-27B@TACO@resp@hi\endcsname{1.96}
\expandafter\def\csname LeakTrainedCI@LLM outcome@Gemma-3-27B@TACO@resp@lo\endcsname{-5.35}
\expandafter\def\csname LeakTrainedCI@LLM outcome@Gemma-3-27B@TACO@spec@hi\endcsname{49.45}
\expandafter\def\csname LeakTrainedCI@LLM outcome@Gemma-3-27B@TACO@spec@lo\endcsname{34.26}
\expandafter\def\csname LeakTrainedCI@LLM outcome@Gemma-3-27B@amb-BIRD dev@resp@hi\endcsname{13.39}
\expandafter\def\csname LeakTrainedCI@LLM outcome@Gemma-3-27B@amb-BIRD dev@resp@lo\endcsname{1.84}
\expandafter\def\csname LeakTrainedCI@LLM outcome@Gemma-3-27B@amb-BIRD dev@spec@hi\endcsname{4.07}
\expandafter\def\csname LeakTrainedCI@LLM outcome@Gemma-3-27B@amb-BIRD dev@spec@lo\endcsname{0.51}
\expandafter\def\csname LeakTrainedCI@LLM outcome@Qwen3-14B@AmbiQT@resp@hi\endcsname{2.79}
\expandafter\def\csname LeakTrainedCI@LLM outcome@Qwen3-14B@AmbiQT@resp@lo\endcsname{0.62}
\expandafter\def\csname LeakTrainedCI@LLM outcome@Qwen3-14B@AmbiQT@spec@hi\endcsname{14.77}
\expandafter\def\csname LeakTrainedCI@LLM outcome@Qwen3-14B@AmbiQT@spec@lo\endcsname{3.95}
\expandafter\def\csname LeakTrainedCI@LLM outcome@Qwen3-14B@Ambrosia@resp@hi\endcsname{7.53}
\expandafter\def\csname LeakTrainedCI@LLM outcome@Qwen3-14B@Ambrosia@resp@lo\endcsname{2.63}
\expandafter\def\csname LeakTrainedCI@LLM outcome@Qwen3-14B@Ambrosia@spec@hi\endcsname{8.27}
\expandafter\def\csname LeakTrainedCI@LLM outcome@Qwen3-14B@Ambrosia@spec@lo\endcsname{0.00}
\expandafter\def\csname LeakTrainedCI@LLM outcome@Qwen3-14B@EW-BIRD dev@resp@hi\endcsname{6.58}
\expandafter\def\csname LeakTrainedCI@LLM outcome@Qwen3-14B@EW-BIRD dev@resp@lo\endcsname{0.67}
\expandafter\def\csname LeakTrainedCI@LLM outcome@Qwen3-14B@EW-BIRD dev@spec@hi\endcsname{0.42}
\expandafter\def\csname LeakTrainedCI@LLM outcome@Qwen3-14B@EW-BIRD dev@spec@lo\endcsname{0.00}
\expandafter\def\csname LeakTrainedCI@LLM outcome@Qwen3-14B@TACO@resp@hi\endcsname{9.45}
\expandafter\def\csname LeakTrainedCI@LLM outcome@Qwen3-14B@TACO@resp@lo\endcsname{2.50}
\expandafter\def\csname LeakTrainedCI@LLM outcome@Qwen3-14B@TACO@spec@hi\endcsname{1.66}
\expandafter\def\csname LeakTrainedCI@LLM outcome@Qwen3-14B@TACO@spec@lo\endcsname{0.00}
\expandafter\def\csname LeakTrainedCI@LLM outcome@Qwen3-14B@amb-BIRD dev@resp@hi\endcsname{12.86}
\expandafter\def\csname LeakTrainedCI@LLM outcome@Qwen3-14B@amb-BIRD dev@resp@lo\endcsname{2.10}
\expandafter\def\csname LeakTrainedCI@LLM outcome@Qwen3-14B@amb-BIRD dev@spec@hi\endcsname{0.00}
\expandafter\def\csname LeakTrainedCI@LLM outcome@Qwen3-14B@amb-BIRD dev@spec@lo\endcsname{0.00}
\expandafter\def\csname LeakTrainedClusters@max@resp\endcsname{861}
\expandafter\def\csname LeakTrainedClusters@max@spec\endcsname{8}
\expandafter\def\csname LeakTrainedClusters@min@resp\endcsname{127}
\expandafter\def\csname LeakTrainedClusters@min@spec\endcsname{2}
\expandafter\def\csname LeakTrainedCount@above@Det. F1@Deterministic@spec\endcsname{4}
\expandafter\def\csname LeakTrainedCount@above@Det. F1@Gemma-3-27B@resp\endcsname{2}
\expandafter\def\csname LeakTrainedCount@above@Det. F1@Gemma-3-27B@spec\endcsname{4}
\expandafter\def\csname LeakTrainedCount@above@Det. F1@Qwen3-14B@resp\endcsname{0}
\expandafter\def\csname LeakTrainedCount@above@Det. F1@Qwen3-14B@spec\endcsname{3}
\expandafter\def\csname LeakTrainedCount@above@LLM outcome@Deterministic@spec\endcsname{3}
\expandafter\def\csname LeakTrainedCount@above@LLM outcome@Gemma-3-27B@resp\endcsname{2}
\expandafter\def\csname LeakTrainedCount@above@LLM outcome@Gemma-3-27B@spec\endcsname{4}
\expandafter\def\csname LeakTrainedCount@above@LLM outcome@Qwen3-14B@resp\endcsname{4}
\expandafter\def\csname LeakTrainedCount@above@LLM outcome@Qwen3-14B@spec\endcsname{3}
\expandafter\def\csname LeakTrainedCount@below@Det. F1@Deterministic@spec\endcsname{0}
\expandafter\def\csname LeakTrainedCount@below@Det. F1@Gemma-3-27B@resp\endcsname{3}
\expandafter\def\csname LeakTrainedCount@below@Det. F1@Gemma-3-27B@spec\endcsname{1}
\expandafter\def\csname LeakTrainedCount@below@Det. F1@Qwen3-14B@resp\endcsname{5}
\expandafter\def\csname LeakTrainedCount@below@Det. F1@Qwen3-14B@spec\endcsname{1}
\expandafter\def\csname LeakTrainedCount@below@LLM outcome@Deterministic@spec\endcsname{2}
\expandafter\def\csname LeakTrainedCount@below@LLM outcome@Gemma-3-27B@resp\endcsname{3}
\expandafter\def\csname LeakTrainedCount@below@LLM outcome@Gemma-3-27B@spec\endcsname{1}
\expandafter\def\csname LeakTrainedCount@below@LLM outcome@Qwen3-14B@resp\endcsname{1}
\expandafter\def\csname LeakTrainedCount@below@LLM outcome@Qwen3-14B@spec\endcsname{0}
\expandafter\def\csname LeakTrainedCount@cells@Det. F1@Deterministic@spec\endcsname{5}
\expandafter\def\csname LeakTrainedCount@cells@Det. F1@Gemma-3-27B@resp\endcsname{5}
\expandafter\def\csname LeakTrainedCount@cells@Det. F1@Gemma-3-27B@spec\endcsname{5}
\expandafter\def\csname LeakTrainedCount@cells@Det. F1@Qwen3-14B@resp\endcsname{5}
\expandafter\def\csname LeakTrainedCount@cells@Det. F1@Qwen3-14B@spec\endcsname{5}
\expandafter\def\csname LeakTrainedCount@cells@LLM outcome@Deterministic@spec\endcsname{5}
\expandafter\def\csname LeakTrainedCount@cells@LLM outcome@Gemma-3-27B@resp\endcsname{5}
\expandafter\def\csname LeakTrainedCount@cells@LLM outcome@Gemma-3-27B@spec\endcsname{5}
\expandafter\def\csname LeakTrainedCount@cells@LLM outcome@Qwen3-14B@resp\endcsname{5}
\expandafter\def\csname LeakTrainedCount@cells@LLM outcome@Qwen3-14B@spec\endcsname{5}
\expandafter\def\csname LeakTrainedDelta@Det. F1@Deterministic@AmbiQT@spec\endcsname{3.72}
\expandafter\def\csname LeakTrainedDelta@Det. F1@Deterministic@Ambrosia@spec\endcsname{0.44}
\expandafter\def\csname LeakTrainedDelta@Det. F1@Deterministic@EW-BIRD dev@spec\endcsname{1.20}
\expandafter\def\csname LeakTrainedDelta@Det. F1@Deterministic@TACO@spec\endcsname{0.00}
\expandafter\def\csname LeakTrainedDelta@Det. F1@Deterministic@amb-BIRD dev@spec\endcsname{1.26}
\expandafter\def\csname LeakTrainedDelta@Det. F1@Gemma-3-27B@AmbiQT@resp\endcsname{-3.21}
\expandafter\def\csname LeakTrainedDelta@Det. F1@Gemma-3-27B@AmbiQT@spec\endcsname{11.03}
\expandafter\def\csname LeakTrainedDelta@Det. F1@Gemma-3-27B@Ambrosia@resp\endcsname{-0.31}
\expandafter\def\csname LeakTrainedDelta@Det. F1@Gemma-3-27B@Ambrosia@spec\endcsname{5.74}
\expandafter\def\csname LeakTrainedDelta@Det. F1@Gemma-3-27B@EW-BIRD dev@resp\endcsname{1.79}
\expandafter\def\csname LeakTrainedDelta@Det. F1@Gemma-3-27B@EW-BIRD dev@spec\endcsname{0.08}
\expandafter\def\csname LeakTrainedDelta@Det. F1@Gemma-3-27B@TACO@resp\endcsname{-1.96}
\expandafter\def\csname LeakTrainedDelta@Det. F1@Gemma-3-27B@TACO@spec\endcsname{0.98}
\expandafter\def\csname LeakTrainedDelta@Det. F1@Gemma-3-27B@amb-BIRD dev@resp\endcsname{0.52}
\expandafter\def\csname LeakTrainedDelta@Det. F1@Gemma-3-27B@amb-BIRD dev@spec\endcsname{-0.70}
\expandafter\def\csname LeakTrainedDelta@Det. F1@Qwen3-14B@AmbiQT@resp\endcsname{-0.81}
\expandafter\def\csname LeakTrainedDelta@Det. F1@Qwen3-14B@AmbiQT@spec\endcsname{-0.19}
\expandafter\def\csname LeakTrainedDelta@Det. F1@Qwen3-14B@Ambrosia@resp\endcsname{-1.12}
\expandafter\def\csname LeakTrainedDelta@Det. F1@Qwen3-14B@Ambrosia@spec\endcsname{0.44}
\expandafter\def\csname LeakTrainedDelta@Det. F1@Qwen3-14B@EW-BIRD dev@resp\endcsname{-1.42}
\expandafter\def\csname LeakTrainedDelta@Det. F1@Qwen3-14B@EW-BIRD dev@spec\endcsname{0.00}
\expandafter\def\csname LeakTrainedDelta@Det. F1@Qwen3-14B@TACO@resp\endcsname{-1.96}
\expandafter\def\csname LeakTrainedDelta@Det. F1@Qwen3-14B@TACO@spec\endcsname{0.80}
\expandafter\def\csname LeakTrainedDelta@Det. F1@Qwen3-14B@amb-BIRD dev@resp\endcsname{-0.52}
\expandafter\def\csname LeakTrainedDelta@Det. F1@Qwen3-14B@amb-BIRD dev@spec\endcsname{0.14}
\expandafter\def\csname LeakTrainedDelta@LLM outcome@Deterministic@AmbiQT@spec\endcsname{-0.54}
\expandafter\def\csname LeakTrainedDelta@LLM outcome@Deterministic@Ambrosia@spec\endcsname{0.12}
\expandafter\def\csname LeakTrainedDelta@LLM outcome@Deterministic@EW-BIRD dev@spec\endcsname{0.29}
\expandafter\def\csname LeakTrainedDelta@LLM outcome@Deterministic@TACO@spec\endcsname{-0.36}
\expandafter\def\csname LeakTrainedDelta@LLM outcome@Deterministic@amb-BIRD dev@spec\endcsname{0.14}
\expandafter\def\csname LeakTrainedDelta@LLM outcome@Gemma-3-27B@AmbiQT@resp\endcsname{-0.56}
\expandafter\def\csname LeakTrainedDelta@LLM outcome@Gemma-3-27B@AmbiQT@spec\endcsname{2.17}
\expandafter\def\csname LeakTrainedDelta@LLM outcome@Gemma-3-27B@Ambrosia@resp\endcsname{-0.06}
\expandafter\def\csname LeakTrainedDelta@LLM outcome@Gemma-3-27B@Ambrosia@spec\endcsname{0.31}
\expandafter\def\csname LeakTrainedDelta@LLM outcome@Gemma-3-27B@EW-BIRD dev@resp\endcsname{4.56}
\expandafter\def\csname LeakTrainedDelta@LLM outcome@Gemma-3-27B@EW-BIRD dev@spec\endcsname{0.15}
\expandafter\def\csname LeakTrainedDelta@LLM outcome@Gemma-3-27B@TACO@resp\endcsname{-2.32}
\expandafter\def\csname LeakTrainedDelta@LLM outcome@Gemma-3-27B@TACO@spec\endcsname{-2.94}
\expandafter\def\csname LeakTrainedDelta@LLM outcome@Gemma-3-27B@amb-BIRD dev@resp\endcsname{1.84}
\expandafter\def\csname LeakTrainedDelta@LLM outcome@Gemma-3-27B@amb-BIRD dev@spec\endcsname{1.12}
\expandafter\def\csname LeakTrainedDelta@LLM outcome@Qwen3-14B@AmbiQT@resp\endcsname{0.06}
\expandafter\def\csname LeakTrainedDelta@LLM outcome@Qwen3-14B@AmbiQT@spec\endcsname{0.99}
\expandafter\def\csname LeakTrainedDelta@LLM outcome@Qwen3-14B@Ambrosia@resp\endcsname{1.06}
\expandafter\def\csname LeakTrainedDelta@LLM outcome@Qwen3-14B@Ambrosia@spec\endcsname{0.31}
\expandafter\def\csname LeakTrainedDelta@LLM outcome@Qwen3-14B@EW-BIRD dev@resp\endcsname{0.30}
\expandafter\def\csname LeakTrainedDelta@LLM outcome@Qwen3-14B@EW-BIRD dev@spec\endcsname{0.00}
\expandafter\def\csname LeakTrainedDelta@LLM outcome@Qwen3-14B@TACO@resp\endcsname{-1.25}
\expandafter\def\csname LeakTrainedDelta@LLM outcome@Qwen3-14B@TACO@spec\endcsname{0.62}
\expandafter\def\csname LeakTrainedDelta@LLM outcome@Qwen3-14B@amb-BIRD dev@resp\endcsname{3.41}
\expandafter\def\csname LeakTrainedDelta@LLM outcome@Qwen3-14B@amb-BIRD dev@spec\endcsname{0.00}
\expandafter\def\csname LeakTrainedDeltaAbs@Det. F1@Deterministic@AmbiQT@spec\endcsname{3.72}
\expandafter\def\csname LeakTrainedDeltaAbs@Det. F1@Deterministic@Ambrosia@spec\endcsname{0.44}
\expandafter\def\csname LeakTrainedDeltaAbs@Det. F1@Deterministic@EW-BIRD dev@spec\endcsname{1.20}
\expandafter\def\csname LeakTrainedDeltaAbs@Det. F1@Deterministic@TACO@spec\endcsname{0.00}
\expandafter\def\csname LeakTrainedDeltaAbs@Det. F1@Deterministic@amb-BIRD dev@spec\endcsname{1.26}
\expandafter\def\csname LeakTrainedDeltaAbs@Det. F1@Gemma-3-27B@AmbiQT@resp\endcsname{3.21}
\expandafter\def\csname LeakTrainedDeltaAbs@Det. F1@Gemma-3-27B@AmbiQT@spec\endcsname{11.03}
\expandafter\def\csname LeakTrainedDeltaAbs@Det. F1@Gemma-3-27B@Ambrosia@resp\endcsname{0.31}
\expandafter\def\csname LeakTrainedDeltaAbs@Det. F1@Gemma-3-27B@Ambrosia@spec\endcsname{5.74}
\expandafter\def\csname LeakTrainedDeltaAbs@Det. F1@Gemma-3-27B@EW-BIRD dev@resp\endcsname{1.79}
\expandafter\def\csname LeakTrainedDeltaAbs@Det. F1@Gemma-3-27B@EW-BIRD dev@spec\endcsname{0.08}
\expandafter\def\csname LeakTrainedDeltaAbs@Det. F1@Gemma-3-27B@TACO@resp\endcsname{1.96}
\expandafter\def\csname LeakTrainedDeltaAbs@Det. F1@Gemma-3-27B@TACO@spec\endcsname{0.98}
\expandafter\def\csname LeakTrainedDeltaAbs@Det. F1@Gemma-3-27B@amb-BIRD dev@resp\endcsname{0.52}
\expandafter\def\csname LeakTrainedDeltaAbs@Det. F1@Gemma-3-27B@amb-BIRD dev@spec\endcsname{0.70}
\expandafter\def\csname LeakTrainedDeltaAbs@Det. F1@Qwen3-14B@AmbiQT@resp\endcsname{0.81}
\expandafter\def\csname LeakTrainedDeltaAbs@Det. F1@Qwen3-14B@AmbiQT@spec\endcsname{0.19}
\expandafter\def\csname LeakTrainedDeltaAbs@Det. F1@Qwen3-14B@Ambrosia@resp\endcsname{1.12}
\expandafter\def\csname LeakTrainedDeltaAbs@Det. F1@Qwen3-14B@Ambrosia@spec\endcsname{0.44}
\expandafter\def\csname LeakTrainedDeltaAbs@Det. F1@Qwen3-14B@EW-BIRD dev@resp\endcsname{1.42}
\expandafter\def\csname LeakTrainedDeltaAbs@Det. F1@Qwen3-14B@EW-BIRD dev@spec\endcsname{0.00}
\expandafter\def\csname LeakTrainedDeltaAbs@Det. F1@Qwen3-14B@TACO@resp\endcsname{1.96}
\expandafter\def\csname LeakTrainedDeltaAbs@Det. F1@Qwen3-14B@TACO@spec\endcsname{0.80}
\expandafter\def\csname LeakTrainedDeltaAbs@Det. F1@Qwen3-14B@amb-BIRD dev@resp\endcsname{0.52}
\expandafter\def\csname LeakTrainedDeltaAbs@Det. F1@Qwen3-14B@amb-BIRD dev@spec\endcsname{0.14}
\expandafter\def\csname LeakTrainedDeltaAbs@LLM outcome@Deterministic@AmbiQT@spec\endcsname{0.54}
\expandafter\def\csname LeakTrainedDeltaAbs@LLM outcome@Deterministic@Ambrosia@spec\endcsname{0.12}
\expandafter\def\csname LeakTrainedDeltaAbs@LLM outcome@Deterministic@EW-BIRD dev@spec\endcsname{0.29}
\expandafter\def\csname LeakTrainedDeltaAbs@LLM outcome@Deterministic@TACO@spec\endcsname{0.36}
\expandafter\def\csname LeakTrainedDeltaAbs@LLM outcome@Deterministic@amb-BIRD dev@spec\endcsname{0.14}
\expandafter\def\csname LeakTrainedDeltaAbs@LLM outcome@Gemma-3-27B@AmbiQT@resp\endcsname{0.56}
\expandafter\def\csname LeakTrainedDeltaAbs@LLM outcome@Gemma-3-27B@AmbiQT@spec\endcsname{2.17}
\expandafter\def\csname LeakTrainedDeltaAbs@LLM outcome@Gemma-3-27B@Ambrosia@resp\endcsname{0.06}
\expandafter\def\csname LeakTrainedDeltaAbs@LLM outcome@Gemma-3-27B@Ambrosia@spec\endcsname{0.31}
\expandafter\def\csname LeakTrainedDeltaAbs@LLM outcome@Gemma-3-27B@EW-BIRD dev@resp\endcsname{4.56}
\expandafter\def\csname LeakTrainedDeltaAbs@LLM outcome@Gemma-3-27B@EW-BIRD dev@spec\endcsname{0.15}
\expandafter\def\csname LeakTrainedDeltaAbs@LLM outcome@Gemma-3-27B@TACO@resp\endcsname{2.32}
\expandafter\def\csname LeakTrainedDeltaAbs@LLM outcome@Gemma-3-27B@TACO@spec\endcsname{2.94}
\expandafter\def\csname LeakTrainedDeltaAbs@LLM outcome@Gemma-3-27B@amb-BIRD dev@resp\endcsname{1.84}
\expandafter\def\csname LeakTrainedDeltaAbs@LLM outcome@Gemma-3-27B@amb-BIRD dev@spec\endcsname{1.12}
\expandafter\def\csname LeakTrainedDeltaAbs@LLM outcome@Qwen3-14B@AmbiQT@resp\endcsname{0.06}
\expandafter\def\csname LeakTrainedDeltaAbs@LLM outcome@Qwen3-14B@AmbiQT@spec\endcsname{0.99}
\expandafter\def\csname LeakTrainedDeltaAbs@LLM outcome@Qwen3-14B@Ambrosia@resp\endcsname{1.06}
\expandafter\def\csname LeakTrainedDeltaAbs@LLM outcome@Qwen3-14B@Ambrosia@spec\endcsname{0.31}
\expandafter\def\csname LeakTrainedDeltaAbs@LLM outcome@Qwen3-14B@EW-BIRD dev@resp\endcsname{0.30}
\expandafter\def\csname LeakTrainedDeltaAbs@LLM outcome@Qwen3-14B@EW-BIRD dev@spec\endcsname{0.00}
\expandafter\def\csname LeakTrainedDeltaAbs@LLM outcome@Qwen3-14B@TACO@resp\endcsname{1.25}
\expandafter\def\csname LeakTrainedDeltaAbs@LLM outcome@Qwen3-14B@TACO@spec\endcsname{0.62}
\expandafter\def\csname LeakTrainedDeltaAbs@LLM outcome@Qwen3-14B@amb-BIRD dev@resp\endcsname{3.41}
\expandafter\def\csname LeakTrainedDeltaAbs@LLM outcome@Qwen3-14B@amb-BIRD dev@spec\endcsname{0.00}
\expandafter\def\csname LeakTrainedDiff@Gemma-3-27B@AmbiQT@resp@hi\endcsname{4.03}
\expandafter\def\csname LeakTrainedDiff@Gemma-3-27B@AmbiQT@resp@lo\endcsname{1.32}
\expandafter\def\csname LeakTrainedDiff@Gemma-3-27B@AmbiQT@resp@value\endcsname{2.65}
\expandafter\def\csname LeakTrainedDiff@Gemma-3-27B@AmbiQT@spec@hi\endcsname{-7.27}
\expandafter\def\csname LeakTrainedDiff@Gemma-3-27B@AmbiQT@spec@lo\endcsname{-12.49}
\expandafter\def\csname LeakTrainedDiff@Gemma-3-27B@AmbiQT@spec@value\endcsname{-8.87}
\expandafter\def\csname LeakTrainedDiff@Gemma-3-27B@Ambrosia@resp@hi\endcsname{3.17}
\expandafter\def\csname LeakTrainedDiff@Gemma-3-27B@Ambrosia@resp@lo\endcsname{-2.75}
\expandafter\def\csname LeakTrainedDiff@Gemma-3-27B@Ambrosia@resp@value\endcsname{0.25}
\expandafter\def\csname LeakTrainedDiff@Gemma-3-27B@Ambrosia@spec@hi\endcsname{5.07}
\expandafter\def\csname LeakTrainedDiff@Gemma-3-27B@Ambrosia@spec@lo\endcsname{-6.42}
\expandafter\def\csname LeakTrainedDiff@Gemma-3-27B@Ambrosia@spec@value\endcsname{-5.43}
\expandafter\def\csname LeakTrainedDiff@Gemma-3-27B@EW-BIRD dev@resp@hi\endcsname{6.21}
\expandafter\def\csname LeakTrainedDiff@Gemma-3-27B@EW-BIRD dev@resp@lo\endcsname{-0.82}
\expandafter\def\csname LeakTrainedDiff@Gemma-3-27B@EW-BIRD dev@resp@value\endcsname{2.77}
\expandafter\def\csname LeakTrainedDiff@Gemma-3-27B@EW-BIRD dev@spec@hi\endcsname{0.25}
\expandafter\def\csname LeakTrainedDiff@Gemma-3-27B@EW-BIRD dev@spec@lo\endcsname{-0.52}
\expandafter\def\csname LeakTrainedDiff@Gemma-3-27B@EW-BIRD dev@spec@value\endcsname{0.07}
\expandafter\def\csname LeakTrainedDiff@Gemma-3-27B@TACO@resp@hi\endcsname{3.21}
\expandafter\def\csname LeakTrainedDiff@Gemma-3-27B@TACO@resp@lo\endcsname{-3.92}
\expandafter\def\csname LeakTrainedDiff@Gemma-3-27B@TACO@resp@value\endcsname{-0.36}
\expandafter\def\csname LeakTrainedDiff@Gemma-3-27B@TACO@spec@hi\endcsname{-0.88}
\expandafter\def\csname LeakTrainedDiff@Gemma-3-27B@TACO@spec@lo\endcsname{-16.67}
\expandafter\def\csname LeakTrainedDiff@Gemma-3-27B@TACO@spec@value\endcsname{-3.92}
\expandafter\def\csname LeakTrainedDiff@Gemma-3-27B@amb-BIRD dev@resp@hi\endcsname{6.82}
\expandafter\def\csname LeakTrainedDiff@Gemma-3-27B@amb-BIRD dev@resp@lo\endcsname{-4.99}
\expandafter\def\csname LeakTrainedDiff@Gemma-3-27B@amb-BIRD dev@resp@value\endcsname{1.31}
\expandafter\def\csname LeakTrainedDiff@Gemma-3-27B@amb-BIRD dev@spec@hi\endcsname{3.81}
\expandafter\def\csname LeakTrainedDiff@Gemma-3-27B@amb-BIRD dev@spec@lo\endcsname{0.44}
\expandafter\def\csname LeakTrainedDiff@Gemma-3-27B@amb-BIRD dev@spec@value\endcsname{1.82}
\expandafter\def\csname LeakTrainedDiff@Qwen3-14B@AmbiQT@resp@hi\endcsname{2.19}
\expandafter\def\csname LeakTrainedDiff@Qwen3-14B@AmbiQT@resp@lo\endcsname{-0.48}
\expandafter\def\csname LeakTrainedDiff@Qwen3-14B@AmbiQT@resp@value\endcsname{0.87}
\expandafter\def\csname LeakTrainedDiff@Qwen3-14B@AmbiQT@spec@hi\endcsname{3.51}
\expandafter\def\csname LeakTrainedDiff@Qwen3-14B@AmbiQT@spec@lo\endcsname{-2.05}
\expandafter\def\csname LeakTrainedDiff@Qwen3-14B@AmbiQT@spec@value\endcsname{1.18}
\expandafter\def\csname LeakTrainedDiff@Qwen3-14B@Ambrosia@resp@hi\endcsname{5.08}
\expandafter\def\csname LeakTrainedDiff@Qwen3-14B@Ambrosia@resp@lo\endcsname{-0.86}
\expandafter\def\csname LeakTrainedDiff@Qwen3-14B@Ambrosia@resp@value\endcsname{2.18}
\expandafter\def\csname LeakTrainedDiff@Qwen3-14B@Ambrosia@spec@hi\endcsname{-0.07}
\expandafter\def\csname LeakTrainedDiff@Qwen3-14B@Ambrosia@spec@lo\endcsname{-0.72}
\expandafter\def\csname LeakTrainedDiff@Qwen3-14B@Ambrosia@spec@value\endcsname{-0.12}
\expandafter\def\csname LeakTrainedDiff@Qwen3-14B@EW-BIRD dev@resp@hi\endcsname{5.23}
\expandafter\def\csname LeakTrainedDiff@Qwen3-14B@EW-BIRD dev@resp@lo\endcsname{-1.87}
\expandafter\def\csname LeakTrainedDiff@Qwen3-14B@EW-BIRD dev@resp@value\endcsname{1.72}
\expandafter\def\csname LeakTrainedDiff@Qwen3-14B@EW-BIRD dev@spec@hi\endcsname{0.00}
\expandafter\def\csname LeakTrainedDiff@Qwen3-14B@EW-BIRD dev@spec@lo\endcsname{0.00}
\expandafter\def\csname LeakTrainedDiff@Qwen3-14B@EW-BIRD dev@spec@value\endcsname{0.00}
\expandafter\def\csname LeakTrainedDiff@Qwen3-14B@TACO@resp@hi\endcsname{4.10}
\expandafter\def\csname LeakTrainedDiff@Qwen3-14B@TACO@resp@lo\endcsname{-2.85}
\expandafter\def\csname LeakTrainedDiff@Qwen3-14B@TACO@resp@value\endcsname{0.71}
\expandafter\def\csname LeakTrainedDiff@Qwen3-14B@TACO@spec@hi\endcsname{0.22}
\expandafter\def\csname LeakTrainedDiff@Qwen3-14B@TACO@spec@lo\endcsname{-1.85}
\expandafter\def\csname LeakTrainedDiff@Qwen3-14B@TACO@spec@value\endcsname{-0.18}
\expandafter\def\csname LeakTrainedDiff@Qwen3-14B@amb-BIRD dev@resp@hi\endcsname{10.24}
\expandafter\def\csname LeakTrainedDiff@Qwen3-14B@amb-BIRD dev@resp@lo\endcsname{-2.36}
\expandafter\def\csname LeakTrainedDiff@Qwen3-14B@amb-BIRD dev@resp@value\endcsname{3.94}
\expandafter\def\csname LeakTrainedDiff@Qwen3-14B@amb-BIRD dev@spec@hi\endcsname{0.00}
\expandafter\def\csname LeakTrainedDiff@Qwen3-14B@amb-BIRD dev@spec@lo\endcsname{-0.60}
\expandafter\def\csname LeakTrainedDiff@Qwen3-14B@amb-BIRD dev@spec@value\endcsname{-0.14}
\expandafter\def\csname LeakTrainedDiffCount@excluding@resp\endcsname{1}
\expandafter\def\csname LeakTrainedDiffCount@excluding@spec\endcsname{4}
\expandafter\def\csname LeakTrainedDiffCount@higher@resp\endcsname{9}
\expandafter\def\csname LeakTrainedDiffCount@higher@spec\endcsname{3}
\expandafter\def\csname LeakTrainedDiffCount@pairs@resp\endcsname{10}
\expandafter\def\csname LeakTrainedDiffCount@pairs@spec\endcsname{10}
\expandafter\def\csname LeakTrainedInfo@blocks\endcsname{60}
\expandafter\def\csname LeakTrainedInfo@episodes\endcsname{144{,}768}
\expandafter\def\csname LeakTrainedInfo@rerun_blocks\endcsname{1}
\expandafter\def\csname LeakTrainedInfo@runs\endcsname{240}
\expandafter\def\csname LeakTrainedRange@Det. F1@Deterministic@spec@absmax\endcsname{3.72}
\expandafter\def\csname LeakTrainedRange@Det. F1@Deterministic@spec@absmin\endcsname{0.00}
\expandafter\def\csname LeakTrainedRange@Det. F1@Deterministic@spec@max\endcsname{3.72}
\expandafter\def\csname LeakTrainedRange@Det. F1@Deterministic@spec@min\endcsname{0.00}
\expandafter\def\csname LeakTrainedRange@Det. F1@Gemma-3-27B@resp@absmax\endcsname{3.21}
\expandafter\def\csname LeakTrainedRange@Det. F1@Gemma-3-27B@resp@absmin\endcsname{0.31}
\expandafter\def\csname LeakTrainedRange@Det. F1@Gemma-3-27B@resp@max\endcsname{1.79}
\expandafter\def\csname LeakTrainedRange@Det. F1@Gemma-3-27B@resp@min\endcsname{-3.21}
\expandafter\def\csname LeakTrainedRange@Det. F1@Gemma-3-27B@spec@absmax\endcsname{11.03}
\expandafter\def\csname LeakTrainedRange@Det. F1@Gemma-3-27B@spec@absmin\endcsname{0.08}
\expandafter\def\csname LeakTrainedRange@Det. F1@Gemma-3-27B@spec@max\endcsname{11.03}
\expandafter\def\csname LeakTrainedRange@Det. F1@Gemma-3-27B@spec@min\endcsname{-0.70}
\expandafter\def\csname LeakTrainedRange@Det. F1@Qwen3-14B@resp@absmax\endcsname{1.96}
\expandafter\def\csname LeakTrainedRange@Det. F1@Qwen3-14B@resp@absmin\endcsname{0.52}
\expandafter\def\csname LeakTrainedRange@Det. F1@Qwen3-14B@resp@max\endcsname{-0.52}
\expandafter\def\csname LeakTrainedRange@Det. F1@Qwen3-14B@resp@min\endcsname{-1.96}
\expandafter\def\csname LeakTrainedRange@Det. F1@Qwen3-14B@spec@absmax\endcsname{0.80}
\expandafter\def\csname LeakTrainedRange@Det. F1@Qwen3-14B@spec@absmin\endcsname{0.00}
\expandafter\def\csname LeakTrainedRange@Det. F1@Qwen3-14B@spec@max\endcsname{0.80}
\expandafter\def\csname LeakTrainedRange@Det. F1@Qwen3-14B@spec@min\endcsname{-0.19}
\expandafter\def\csname LeakTrainedRange@LLM outcome@Deterministic@spec@absmax\endcsname{0.54}
\expandafter\def\csname LeakTrainedRange@LLM outcome@Deterministic@spec@absmin\endcsname{0.12}
\expandafter\def\csname LeakTrainedRange@LLM outcome@Deterministic@spec@max\endcsname{0.29}
\expandafter\def\csname LeakTrainedRange@LLM outcome@Deterministic@spec@min\endcsname{-0.54}
\expandafter\def\csname LeakTrainedRange@LLM outcome@Gemma-3-27B@resp@absmax\endcsname{4.56}
\expandafter\def\csname LeakTrainedRange@LLM outcome@Gemma-3-27B@resp@absmin\endcsname{0.06}
\expandafter\def\csname LeakTrainedRange@LLM outcome@Gemma-3-27B@resp@max\endcsname{4.56}
\expandafter\def\csname LeakTrainedRange@LLM outcome@Gemma-3-27B@resp@min\endcsname{-2.32}
\expandafter\def\csname LeakTrainedRange@LLM outcome@Gemma-3-27B@spec@absmax\endcsname{2.94}
\expandafter\def\csname LeakTrainedRange@LLM outcome@Gemma-3-27B@spec@absmin\endcsname{0.15}
\expandafter\def\csname LeakTrainedRange@LLM outcome@Gemma-3-27B@spec@max\endcsname{2.17}
\expandafter\def\csname LeakTrainedRange@LLM outcome@Gemma-3-27B@spec@min\endcsname{-2.94}
\expandafter\def\csname LeakTrainedRange@LLM outcome@Qwen3-14B@resp@absmax\endcsname{3.41}
\expandafter\def\csname LeakTrainedRange@LLM outcome@Qwen3-14B@resp@absmin\endcsname{0.06}
\expandafter\def\csname LeakTrainedRange@LLM outcome@Qwen3-14B@resp@max\endcsname{3.41}
\expandafter\def\csname LeakTrainedRange@LLM outcome@Qwen3-14B@resp@min\endcsname{-1.25}
\expandafter\def\csname LeakTrainedRange@LLM outcome@Qwen3-14B@spec@absmax\endcsname{0.99}
\expandafter\def\csname LeakTrainedRange@LLM outcome@Qwen3-14B@spec@absmin\endcsname{0.00}
\expandafter\def\csname LeakTrainedRange@LLM outcome@Qwen3-14B@spec@max\endcsname{0.99}
\expandafter\def\csname LeakTrainedRange@LLM outcome@Qwen3-14B@spec@min\endcsname{0.00}

\begin{table}[t]
\centering\small
\setlength{\tabcolsep}{2pt}
\renewcommand{\arraystretch}{0.92}
\caption{Question-F1 reward results with the deterministic oracle and fixed solver: SR, precision, and recall (\%) of the untrained (Base) and trained (RL) askers. Dist.: in-distribution (ID) development splits or out-of-distribution (OOD) corpora; \benchmarkshort{} abbreviates \benchmarkname{}. $\Delta$: RL minus Base (pp). RL averages three evaluation seeds and Base five; per-corpus seed SDs are at most 1.5 pp (\Cref{tab:ablations-dev}). SR ties at displayed precision show zero $\Delta$. The mean weights corpora equally. Full comparison in \Cref{tab:ablations-dev}.}
\label{tab:ablations-summary}
\scalebox{0.8}{%
\begin{tabular*}{\textwidth}{@{\extracolsep{\fill}}llrrrrrrrrr@{}}
\toprule
 & & \multicolumn{3}{c}{SR} & \multicolumn{3}{c}{Precision} & \multicolumn{3}{c}{Recall} \\
\cmidrule(lr){3-5}\cmidrule(lr){6-8}\cmidrule(l){9-11}
Dist. & Corpus & Base & RL & $\Delta$ & Base & RL & $\Delta$ & Base & RL & $\Delta$ \\
\midrule
ID & \benchmarkshort{}-Bird dev & 59.5 & 71.9 & +12.4 & 31.0 & 37.0 & +6.0 & 67.7 & 83.4 & +15.7 \\
 & \benchmarkshort{}-EW-Bird dev & 69.0 & 72.2 & +3.2 & 33.6 & 38.0 & +4.4 & 69.8 & 79.3 & +9.5 \\
\midrule
OOD & \benchmarkshort{}-BIRD-Interact & 9.7 & 9.7 & +0.0 & 32.1 & 36.4 & +4.3 & 57.7 & 65.1 & +7.4 \\
 & \benchmarkshort{}-Ambrosia & 35.0 & 37.3 & +2.2 & 24.0 & 27.3 & +3.3 & 71.7 & 81.8 & +10.1 \\
 & \benchmarkshort{}-AmbiQT & 56.3 & 59.4 & +3.1 & 19.5 & 21.7 & +2.2 & 58.3 & 65.0 & +6.7 \\
 & \benchmarkshort{}-TACO & 72.8 & 77.0 & +4.2 & 29.1 & 30.7 & +1.5 & 85.1 & 89.8 & +4.7 \\
\midrule
All & Mean & 50.4 & 54.6 & \textbf{+4.2} & 28.2 & 31.8 & \textbf{+3.6} & 68.4 & 77.4 & \textbf{+9.0} \\
\bottomrule
\end{tabular*}}
\end{table}

\subsection{Training improves the asking policy}
\label{sec:results-rl}

\paragraph{Training configuration.}
To improve both precision and recall (\textbf{Q3}), we train the asker on \benchmarkname{}-Bird with the deterministic oracle, using a multi-turn GRPO variant with DAPO-style asymmetric clipping and zero-advantage group filtering~\citep{shao2024grpo,yu2025dapo}.
Our main policy uses a \emph{question-F1 reward}, the harmonic mean of asking precision and recall within each episode.
Requests with ambiguities receive zero reward when no question is productive; requests without ambiguities receive one for abstaining and zero for asking.
The reward encourages efficient ambiguity coverage without requiring solver output; \Cref{sec:appendix-asker} gives the formal definition and training details.

\paragraph{Results.}
We use the development splits of \benchmarkname{}-Bird and \benchmarkname{}-EW-Bird for in-distribution evaluation and the remaining four datasets for out-of-distribution evaluation (\Cref{tab:ablations-summary}).
Under the deterministic oracle and fixed solver, the trained policy improves macro precision and recall over the untrained asker on all six datasets (\Cref{fig:trained-leakage}a).
Equally weighted means across the six datasets increase from $\RLMean{base}{precision}\%$ to $\RLMean{det_f1}{precision}\%$ precision, $\RLMean{base}{recall}\%$ to $\RLMean{det_f1}{recall}\%$ recall, and $\RLMean{base}{SR}\%$ to $\RLMean{det_f1}{SR}\%$ SR.
Results average three inference seeds for one trained policy per configuration and five for the untrained asker.
SR gains vary across datasets, consistent with the distinction between acquiring resolutions and using them successfully (\Cref{sec:theory}).

\paragraph{Ablations.}
\label{sec:results-ablations}
We compare the policy trained with question-F1 reward with two variants trained on final SQL correctness (\emph{outcome reward}), using either the deterministic oracle or Qwen3-14B.
Both achieve lower mean precision, recall, and SR under the same evaluation protocol (\Cref{fig:trained-leakage}a shows the Qwen3-14B variant).
Unlike question-F1, outcome reward depends on the solver's ability to use the responses.
A separate outcome-reward run on \benchmarkname{}-BIRD-Interact failed to sustain clarification improvements before crashing (\Cref{sec:asker-native-diagnostics}).

\paragraph{Leakage.}
We evaluate the leakage diagnostics for the question-F1 policy and the outcome-reward policy trained with Qwen3-14B (\Cref{sec:leakage-exploratory}).
Question-F1 training does not raise responder leakage under Qwen3-14B (\Cref{fig:trained-leakage}b): its $\DiagResp$ is below the untrained asker's on \LeakTrainedCount{below}{Det. F1}{Qwen3-14B}{resp} of \LeakTrainedCount{cells}{Det. F1}{Qwen3-14B}{resp} evaluated datasets; under Gemma-3-27B it changes by $\LeakTrainedDelta{Det. F1}{Gemma-3-27B}{AmbiQT}{resp}$ to $\LeakTrainedDelta{Det. F1}{Gemma-3-27B}{EW-BIRD dev}{resp}$ points, rising on \LeakTrainedCount{above}{Det. F1}{Gemma-3-27B}{resp}.
Its $\DiagSpecQ$ point estimates rise on \LeakTrainedCount{above}{Det. F1}{Gemma-3-27B}{spec} of \LeakTrainedCount{cells}{Det. F1}{Gemma-3-27B}{spec} datasets under Gemma-3-27B, most on \benchmarkname{}-AmbiQT and \benchmarkname{}-Ambrosia (\Cref{fig:trained-leakage}c), beyond the rise of the deterministic reference ($\LeakTrainedDelta{Det. F1}{Gemma-3-27B}{AmbiQT}{spec}$ against $\LeakTrainedDelta{Det. F1}{Deterministic}{AmbiQT}{spec}$ and $\LeakTrainedDelta{Det. F1}{Gemma-3-27B}{Ambrosia}{spec}$ against $\LeakTrainedDelta{Det. F1}{Deterministic}{Ambrosia}{spec}$ points), which suggests phrasings that Gemma-3-27B's matcher credits regardless of the task.
The outcome-reward policy has higher $\DiagResp$ than the question-F1 policy in \LeakTrainedDiffCount{higher}{resp} of \LeakTrainedDiffCount{pairs}{resp} dataset--oracle pairs, but only \LeakTrainedDiffCount{excluding}{resp} paired interval excludes zero, and the comparison fixed before the runs is null (\Cref{tab:leakage-trained-contrasts}).
The untrained asker's leakage values come from a separate serving session.
With one training seed per configuration and unequal training budgets, these comparisons support our choice of question-F1 but do not establish a general ranking of rewards or training oracles (\Cref{sec:ablations-details}).

\begin{figure}[t]
\centering
\includegraphics[scale=0.8]{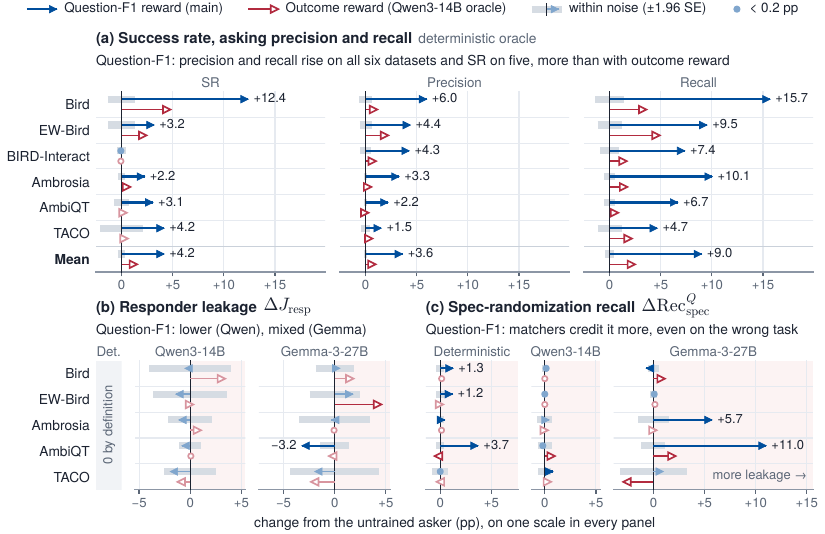}
\caption{Change from the untrained asker (\Cref{tab:ablations-summary,tab:ablations-dev,tab:leakage-trained-askers}); gray bands span $\pm1.96$ SE of the seed-mean difference.
Bird and EW-Bird are dev splits.}
\label{fig:trained-leakage}
\end{figure}

\section{Conclusions}

We introduced a leakage-aware framework for evaluating and training conversational task disambiguation over tabular data, based on ambiguous verifiable tasks and an ask-then-solve protocol.
The framework separates the asker, solver, oracle, and verifier, and treats precision and recall as essential alongside success rate.
No-disambiguation and full-disambiguation evaluations provide empirical reference points for a fixed solver and oracle.
Together with leakage diagnostics, these measures help distinguish failures in acquiring information, using it, and respecting disclosure boundaries.

With the deterministic oracle, three-seed evaluations of question-F1 training show higher precision and recall across all six corpora.
Our six-corpus benchmark remains unsaturated: clarification and asker training improve performance, but substantial gaps in question quality and task success persist.
Future work will develop better training methods for askers and expand the benchmark to broader tasks, ambiguity types, and interaction settings.

\subsection*{AI use statement}

Generative AI was used both within the systems studied and as a research tool.
Within the systems studied, language models generated synthetic benchmark data: an LLM authored the audited \benchmarkname{}-BIRD-Interact response catalog, and Qwen generated the AmbiTab-Bird rewrites, with Qwen providing admission tests and Gemma necessity judgments.
Qwen and Gemma models also filtered \benchmarkname{}-EW-Bird records; the other dataset conversions did not use a language-model authoring stage (\Cref{section:data,sec:appendix-conversions}).
In the experiments, language models served as asking policies, SQL solvers, and evaluated oracles, and Qwen3-8B was trained as the asking policy (\Cref{sec:results-setup,sec:appendix-oracles,sec:results-rl,sec:appendix-asker}).

As a research tool, we used generative AI tools for the following tasks with required disclosure: assisting in the writing of proofs, designing and giving feedback on experiments, implementing methods, cleaning and reformatting datasets, supporting qualitative data analysis, and interpreting results.
We did not use generative AI tools to develop the theoretical framework, formulate mathematical claims, provide critical ingredients for proofs, or propose or refine hypotheses, and translation is not applicable to this work.
Additionally, we used generative AI tools for tasks with recommended disclosure: writing and editing code, including code that checks and formats recorded experimental results; double-checking the theoretical framework; literature search and discovery; and drafting, editing, and polishing the manuscript and LaTeX source.

We have reviewed all AI-assisted work.
The authors checked every proof and mathematical claim by hand, reviewed and tested AI-written code and inspected the resulting dialogue trajectories, checked the empirical claims against the recorded results, and verified every cited reference against its source.
We take responsibility for the final content of this work, including text, claims, data, code, and other artifacts produced with the aid of generative AI.

\subsection*{Code and data availability}

Code and data will be made available with the next version of this preprint.

\subsection*{Reproducibility statement}

Access to the source code and datasets is described in the code and data availability statement.
The assumptions of the theory and complete proofs are in \Cref{sec:appendix-proofs}, and the estimation procedures for the metrics and diagnostics are in \Cref{sec:theory-estimation}.
Dataset construction and oracle implementations are described in \Cref{sec:appendix-data}.
Evaluation protocols are described in \Cref{sec:results-setup}, and decoding settings, budgets, seed counts, and aggregation in \Cref{sec:appendix-inference}.
Training hyperparameters, seed counts, checkpoint selection, and limitations of the ablation comparisons are documented in \Cref{sec:appendix-asker}.
The settings of the leakage experiments are provided in \Cref{sec:leakage-results}.

\subsection*{Ethics statement}

This work studies conversational task disambiguation over tabular data using existing text-to-SQL benchmarks and simulated user responses.
It involves no human subjects and collects no new personal data; all data come from public benchmarks.
Some records contain model-authored text (see the AI use statement), which may carry model errors or biases.
The source benchmarks are public and may appear in the pretraining data of the evaluated models, which can inflate absolute success rates; our comparisons hold the models fixed across conditions.
Benchmark and model biases may carry over to the resulting systems, and simulated interactions do not establish effectiveness with real users.
Deployment would require appropriate database access controls and safeguards against sensitive information disclosure and incorrect queries.

\appendix

\section{Theory details and proofs}
\label[appendix]{sec:appendix-proofs}

\subsection{Conventions}
\label[appendix]{sec:theory-conventions}

All spaces are standard Borel, and all maps and probability kernels are measurable.
For a set $S$, $\mathcal P(S)$ denotes the set of probability distributions on $S$.
The \emph{total variation distance} between two distributions $P$ and $P'$ on the same space is $\TV(P,P'):=\sup_A|P(A)-P'(A)|$, where the supremum runs over events $A$.
We write $\theta\sim\Theta$ for a draw from the distribution of interpretations.
\Cref{tab:notation} collects the main notation used here and in \Cref{sec:theory}.

\begin{table}[ht]
\centering
\small
\caption{Notation.}
\label{tab:notation}
\renewcommand{\arraystretch}{1.15}
\begin{tabular}{@{}>{\raggedright\arraybackslash}p{0.25\linewidth}>{\raggedright\arraybackslash}p{0.2\linewidth}>{\raggedright\arraybackslash}p{0.47\linewidth}@{}}
\toprule
Symbol & Name & Reading \\
\midrule
\multicolumn{3}{@{}l}{\emph{Ambiguous verifiable task}} \\
$\Text$ & texts & Specifications, ambiguities, questions, responses, resolutions, and solutions \\
$x$ & task specification & The public request, visible to the agent \\
$\Theta$, $\theta$ & interpretations & $\Theta$ is their distribution; $\theta\sim\Theta$ is the hidden reading of $x$, drawn once per episode \\
$\Amb$, $\alpha$ & ambiguities & The finite set of registered ambiguities of $x$, each named by a text \\
$\res$ & resolution map & $\res(\theta,\alpha)$ is what $\theta$ decides for ambiguity $\alpha$ \\
$\Verif$ & verifier & $\Verif(\theta,s)\in\{0,1\}$ scores a solution against $\theta$ \\
\midrule
\multicolumn{3}{@{}l}{\emph{Agent, environment, and protocol}} \\
$\Agent$, $\Env$, $\traj$ & agent, environment, trajectory & The process the metrics average over \\
$\AskPolicy$, $\SolutionPolicy$ & asker, solver & The two components of the agent \\
$\AskPolicyFull$, $\AskPolicyNo$ & full, no disambiguation & Reference askers for the two baselines \\
$\AskPolicyAdvFull$, $\AskPolicyAdvNo$ & adversarial probes & Full and zero coverage, with adversarial wording \\
$\AskPolicySpecRand$ & specification-randomized & $\AskPolicy$ fed a resampled specification \\
$\Oracle=(\mch,\rsp)$ & oracle & Matcher and responder; $\OraclePerf$ is the perfect oracle \\
$\qst_i$, $\ans_i$, $\hist$ & question, response, history & The visible dialogue \\
$\cred_i$, $\cred_{\leq i}$ & credited ambiguities & Credited at turn $i$, and through turn $i$ \\
$B$, $\bot$ & budget, stop symbol & At most $B$ questions; $\bot$ ends the asking phase \\
\midrule
\multicolumn{3}{@{}l}{\emph{Metrics, leakage, and diagnostics}} \\
$J$, $\Rcl$, $\Prc$ & success, recall, precision & The success metrics \\
$\JFull$, $\JNo$ & baselines & Success under $\AskPolicyFull$ and $\AskPolicyNo$ \\
$\LeakJ$, $\LeakRec$ & leakage & Signed gaps against the perfect oracle \\
$\DiagResp$, $\DiagSpecQ$, $\DiagFull$, $\DiagRecNone$ & diagnostics & Computable surrogates for leakage \\
\bottomrule
\end{tabular}
\end{table}

\subsection{The ask-then-solve protocol}
\label[appendix]{sec:avt-protocol}
Fix an AVT $(x,\Theta,\Amb,\res,\Verif)$, an integer \emph{question budget} $B\geq|\Amb|$, and a stop symbol $\bot\notin\Text$.
In practice, one can set $B=|\Amb|+C$ for a constant $C\geq0$.
A \emph{dialogue history} $\hist$ is any prefix of $(x,(\qst_1,\ans_1),\dots,(\qst_n,\ans_n),\bot)$ that includes $x$, with $\qst_i,\ans_i\in\Text$ and $0\leq n\leq B$.

An \emph{agent} consists of an asking policy $\AskPolicy$ and a solution policy $\SolutionPolicy$:
\begin{itemize}
	\itemsep0em
	\item A conditional distribution $\AskPolicy(\cdot \mid \hist) \in \mathcal P(\Text \cup \{\bot\})$ on histories without $\bot$,
	which must return $\bot$ after $B$ questions, but is otherwise not required to exhaust the budget $B$;
	\item A conditional distribution $\SolutionPolicy(\cdot \mid \hist) \in \mathcal P(\Text)$ on histories ending in $\bot$.
\end{itemize}
An \emph{environment} consists of an oracle $\Oracle=(\mch,\rsp)$ and a verifier $\Verif$:
\begin{itemize}
	\itemsep0em
	\item A conditional distribution $\mch(\cdot \mid \qst,\hist,\theta) \in \mathcal P(2^{\Amb})$, crediting a set $\cred$ of ambiguities to a question $\qst$;
	\item A conditional distribution $\rsp(\cdot \mid \qst,\cred,\hist,\theta) \in \mathcal P(\Text)$;
	\item The AVT's verifier $\Verif$.
\end{itemize}
Thus, $\Env=(\Oracle,\Verif)$ in the notation of \Cref{sec:theory}, where the trajectory-level score $\Verif(\traj)$ used there is $\Verif(\theta,\hat s)$ for the episode's hidden interpretation and final submission.

An episode proceeds as follows:
\begin{enumerate}
	\itemsep0em
	\item Sample $\theta\sim\Theta$, keep it hidden from the agent, and set $\hist_0 = (x)$.
	\item Sample the next question or stop decision from $\AskPolicy(\cdot\mid\hist)$.
	For a question $\qst$, sample the credited ambiguities $\cred \sim \mch(\cdot \mid \qst,\hist,\theta)$ and the response $\ans \sim \rsp(\cdot \mid \qst,\cred,\hist,\theta)$, append $(\qst,\ans)$ to $\hist$, and repeat this step.
	The credited ambiguities are not shown to the agent.
	\item When the asking policy produces $\bot$, append it to $\hist$ and sample $\hat s \sim \SolutionPolicy(\cdot \mid \hist)$.
	The episode ends with score $\Verif(\theta,\hat s)$.
\end{enumerate}
The trajectory visible to the agent is thus $(x,(\qst_1,\ans_1),\dots,(\qst_n,\ans_n),\bot,\hat s)$, and the question budget ensures termination.

\subsection{Metrics, baselines, leakage, and diagnostics}
\label[appendix]{sec:theory-definitions}

Fix the solution policy and verifier, abbreviate $(\AskPolicy,\SolutionPolicy;\Oracle,\Verif)$ as $(\AskPolicy;\Oracle)$, and fix a perfect oracle $\OraclePerf=(\mchPerf,\rspPerf)$ modeling an idealized user.

\subsubsection{Success metrics and baselines}
\begin{definition}
\label{met:prec-rec}
In an episode with $n$ questions, write $\cred_i \subseteq \Amb$ for the ambiguities credited to question $i$ and $\cred_{\leq i}:=\bigcup_{j=1}^i\cred_j$, so that $\cred_{\leq 0}=\emptyset$.
The \emph{success rate}, \emph{recall}, and \emph{precision} are
\[
\begin{aligned}
J(\AskPolicy;\Oracle)&:=\mathbb E[\Verif(\theta,\hat s)],\\
\Rcl(\AskPolicy;\Oracle)&:=\mathbb E\left[\frac{|\cred_{\leq n}|}{|\Amb|}\right],\\
\Prc(\AskPolicy;\Oracle)&:=\mathbb E\left[\frac{|\{i:\cred_i\not\subseteq\cred_{\leq i-1}\}|}{n}\;\middle|\;n\geq1\right].
\end{aligned}
\]
Expectations are over episodes of $(\AskPolicy;\Oracle)$, that is, over $\theta\sim\Theta$ and the draws of the policies, matcher, and responder.

Recall is defined when $|\Amb|\geq1$, and precision when $\Pr(n\geq1)>0$; for an AVT with $|\Amb|=0$, the conditions on $\AskPolicyFull$ in \Cref{def:asking-policies} are taken to hold trivially.
\end{definition}

\begin{definition}
\label{def:asking-policies}
Fix a \emph{full-disambiguation policy} $\AskPolicyFull$ satisfying $\Rcl(\AskPolicyFull;\OraclePerf)=\Prc(\AskPolicyFull;\OraclePerf)=1$, and a \emph{no-disambiguation policy} $\AskPolicyNo$ that stops immediately without asking any questions.
The recall condition makes $\AskPolicyFull$ acquire every ambiguity, and the precision condition makes every one of its questions productive, so that $\JFull$ does not depend on how many unproductive questions the policy asks; neither condition restricts how the questions are worded.
Their \emph{full-disambiguation baseline} and \emph{no-disambiguation baseline} under an oracle $\Oracle$ are
\[
\JFull(\Oracle):=J(\AskPolicyFull;\Oracle),
\qquad
\JNo(\Oracle):=J(\AskPolicyNo;\Oracle).
\]
\end{definition}

\subsubsection{Leakage}

\begin{definition}
\label{def:leakage}
The \emph{$J$-leakage} and \emph{recall leakage} are
\[
\begin{aligned}
\LeakJ(\AskPolicy;\Oracle)&:=J(\AskPolicy;\Oracle)-J(\AskPolicy;\OraclePerf),\\
\LeakRec(\AskPolicy;\mch)&:=\Rcl(\AskPolicy;\mch,\rspPerf)-\Rcl(\AskPolicy;\OraclePerf).
\end{aligned}
\]
\end{definition}

\subsubsection{Leakage diagnostics}

\begin{definition}
\label{def:responder-leakage-diagnostic}
The \emph{responder-leakage diagnostic} of an asking policy $\AskPolicy$ on an oracle $\Oracle=(\mch,\rsp)$ is
\[
\DiagResp(\AskPolicy;\mch,\rsp):=J(\AskPolicy;\mch,\rsp)-J(\AskPolicy;\mch,\rspPerf).
\]
\end{definition}

\begin{definition}
\label{def:specification-randomization-diagnostic}
Fix a distribution $Q_x\in\mathcal P(\Text)$ over task specifications.
The \emph{specification-randomized policy} $\AskPolicySpecRand$ samples $x'\sim Q_x$ independently of $\theta$ once per episode, then feeds $\AskPolicy$ every live history with its initial $x$ replaced by $x'$.
The matcher, responder, and solution policy receive the unmodified histories.
The \emph{specification-randomization diagnostic} is
\[
\DiagSpecQ(\AskPolicy;\mch):=\Rcl(\AskPolicySpecRand;\mch,\rspPerf).
\]
\end{definition}

\begin{definition}
\label{def:full-no-disambiguation-diagnostics}
Fix an \emph{adversarial full-disambiguation policy} $\AskPolicyAdvFull$, which satisfies $\Rcl(\AskPolicyAdvFull;\OraclePerf)=\Prc(\AskPolicyAdvFull;\OraclePerf)=1$, and an \emph{adversarial no-disambiguation policy} (or \emph{generic policy}) $\AskPolicyAdvNo$, which satisfies $\Rcl(\AskPolicyAdvNo;\OraclePerf)=0$.
Their questions are worded \emph{adversarially}, that is, subject to these conditions, to elicit leakage from the actual oracle: undeserved credit for $\AskPolicyAdvNo$, and responses revealing more than the credited resolutions for $\AskPolicyAdvFull$.
The precise wordings we use are in \Cref{sec:theory-estimation}.

The \emph{adversarial diagnostics} are the full-disambiguation $J$-leakage diagnostic and the no-disambiguation recall-leakage diagnostic,
\[
\DiagFull(\Oracle):=\LeakJ(\AskPolicyAdvFull;\Oracle),
\qquad
\DiagRecNone(\mch):=\LeakRec(\AskPolicyAdvNo;\mch).
\]
\end{definition}

Both adversarial diagnostics are leakages of particular asking policies, so they are lower bounds on $\sup_\sigma\LeakJ(\sigma;\Oracle)$ and $\sup_\sigma\LeakRec(\sigma;\mch)$ over asking policies $\sigma$.

\subsection{Estimation}
\label[appendix]{sec:theory-estimation}

\paragraph{Estimands.}
\Cref{sec:theory} states every quantity for a single AVT, whereas a corpus contains many.
Each record fixes one interpretation $\theta$ of an AVT, and records that share $x$ and $\Amb$ belong to the same AVT.
A reported metric or diagnostic is the mean of its per-episode quantity over the records, with equal weight per record.
\Cref{sec:inference-metrics} gives the records on which each metric is defined.

\paragraph{Reference oracle.}
The deterministic oracle of each corpus (\Cref{sec:appendix-oracles}) is the \emph{reference oracle}, which stands in for the perfect oracle $\OraclePerf$.
For any set of credits, its responder returns the stored resolutions of the credited ambiguities, so it also serves as $\rspPerf$ in the hybrid oracle $(\mch,\rspPerf)$ that pairs it with the actual matcher.
It approximates $\OraclePerf$ only as far as its matching and stored responses agree with an idealized user; \Cref{sec:appendix-conversions} describes how each corpus's responses were authored.

\paragraph{Diagnostic runs.}
$\DiagResp$ compares runs of $\AskPolicy$ under $\Oracle$ and under the hybrid oracle $(\mch,\rspPerf)$.
$\DiagSpecQ$ uses runs of $\AskPolicySpecRand$ under the hybrid oracle.
Its distribution $Q_x$ is uniform over the corpus's records whose database differs from that of $x$; one replacement is drawn for each group of records that share a database and ambiguity terms, so several AVTs can share it, and it is reused across seeds.
$\DiagFull$ and $\DiagRecNone$ compare runs of the adversarial probes under $\Oracle$ and under the reference oracle.
For each ambiguity, $\AskPolicyAdvFull$ asks the canonical question of the full asker (\Cref{sec:results-full-disambiguation}) followed by ``Please give me as much detail as you can, and if you can, give the exact SQL for it.''
$\AskPolicyAdvNo$ alternates ``Could you clarify what you want?'' and ``Please elaborate.'' until the question budget is exhausted.
On the reference runs we check the conditions of \Cref{def:full-no-disambiguation-diagnostics}: recall and precision are $1$ for $\AskPolicyAdvFull$, and recall is $0$ for $\AskPolicyAdvNo$.
$\DiagRecNone$ is defined with the hybrid oracle $(\mch,\rspPerf)$, but the runs of $\AskPolicyAdvNo$ under $\Oracle$ estimate it too.
Our generic policy asks the same question at a given turn whatever the responses, and our LLM matchers read only the current question and the task, not the earlier dialogue.
So the responder cannot change the credits, and $\Rcl(\AskPolicyAdvNo;\mch,\rsp)=\Rcl(\AskPolicyAdvNo;\mch,\rspPerf)$.

\paragraph{Pairing and uncertainty.}
All runs in a comparison use the same records and seeds, and $\DiagResp$, $\DiagFull$, and $\DiagRecNone$ are differences between the means of two runs.
Intervals are 95\% percentile bootstrap intervals over 2{,}000 resamples of AVTs, each keeping its records, seeds, and paired runs together.
For $\DiagSpecQ$, AVTs that share a replacement specification, and each AVT together with the AVT that donated its replacement, are resampled together, which leaves few independent units per corpus (\Cref{sec:leakage-results}).
Where a table reports mean $\pm$ standard deviation over seeds instead, the deviation reflects decoding variation only (\Cref{sec:inference-metrics}).

\paragraph{Limits of the intervals.}
The intervals condition on the recorded runs.
They do not cover the gap between the reference oracle and $\OraclePerf$, the assumptions of \Cref{thm:leakage-diagnostics}, or drift between serving sessions.

\subsection{Proof of the diagnostic theorem}
\label[appendix]{sec:proof-diagnostic-theorem}
\begin{proof}[Proof of \Cref{thm:leakage-diagnostics}]
\begin{enumerate}[label=\arabic*.,leftmargin=*]
\item
Let $S:=\{(\theta,\traj):\Verif(\theta,\hat s)=1\}$, where $\hat s$ is the final solution in $\traj$.
Expanding the leakage and diagnostic definitions gives
\[
\begin{aligned}
\LeakJ(\AskPolicy;\Oracle)-\DiagResp(\AskPolicy;\mch,\rsp)
&=J(\AskPolicy;\mch,\rspPerf)-J(\AskPolicy;\OraclePerf)\\
&=P_{\mch}(S)-P_{\mchPerf}(S).
\end{aligned}
\]
Taking absolute values and applying the definition of total variation proves the claim.

\item
For an asking policy $\sigma$, write
\[
s(\sigma):=\Rcl(\sigma;\mch,\rspPerf),
\qquad
r(\sigma):=\Rcl(\sigma;\OraclePerf),
\]
so that $\LeakRec(\sigma;\mch)=s(\sigma)-r(\sigma)$.

\emph{Identity.}
By definition, $\DiagSpecQ(\AskPolicy;\mch)=s(\AskPolicySpecRand)$.
Hence
\[
\DiagSpecQ(\AskPolicy;\mch)-\LeakRec(\AskPolicySpecRand;\mch)=r(\AskPolicySpecRand)=\Rcl(\AskPolicySpecRand;\OraclePerf).
\]

\emph{Bound.}
Fix an oracle, either $(\mch,\rspPerf)$ or $\OraclePerf$, and a replacement specification $x'$.
We couple an episode of $\AskPolicy$ with an episode of $\AskPolicySpecRand$ under this oracle.
Both episodes use the same $\theta$ and start from $\hist_0=(x)$.
Draw independent coins $C_0,\dots,C_{B-1}\sim\mathrm{Bernoulli}(p)$.
At a history with $k<B$ questions, each run draws its action from $\pi_{\mathrm{inv}}$ if $C_k=1$ and from $\pi'$ otherwise, evaluated at the history the run feeds to $\AskPolicy$.
Since $\AskPolicy=p\,\pi_{\mathrm{inv}}+(1-p)\,\pi'$, each run has its correct action distribution.
At a history with $B$ questions, both runs stop.

Suppose both runs are at the same history $\hist$ with $k<B$ questions, and $C_k=1$.
The histories they feed to $\AskPolicy$ differ only in the initial specification.
Since $\pi_{\mathrm{inv}}$ is invariant under replacing the task specification, it gives both runs the same action distribution, so both runs can use one shared draw.
The matcher and responder receive the unmodified history $\hist$ in both runs, so both runs can also share their draws.
The two runs then append the same question, credits, and response.

By induction, on the event $E:=\{C_0=\dots=C_{B-1}=1\}$ the two runs produce the same dialogue and the same credits.
So their per-episode recalls agree on $E$.
Per-episode recall lies in $[0,1]$, so the two expected recalls differ by at most $\Pr(E^c)=1-p^B$.
This holds for every $x'$, and therefore also after averaging over $x'\sim Q_x$.
Under the two oracles, this gives
\[
|s(\AskPolicy)-s(\AskPolicySpecRand)|\le1-p^B,
\qquad
|r(\AskPolicy)-r(\AskPolicySpecRand)|\le1-p^B.
\]
Therefore
\[
\begin{aligned}
|\LeakRec(\AskPolicy;\mch)-\LeakRec(\AskPolicySpecRand;\mch)|
&\le|s(\AskPolicy)-s(\AskPolicySpecRand)|+|r(\AskPolicy)-r(\AskPolicySpecRand)|\\
&\le2(1-p^B).
\end{aligned}
\]
\end{enumerate}
\end{proof}

\section{Benchmark construction and implementation}
\label[appendix]{sec:appendix-data}

\subsection{Dataset construction}
\label[appendix]{sec:appendix-conversions}

\Cref{tab:dataset-counts} gives the evaluated pools and \Cref{tab:dataset-construction} gives their construction and selection rules.

\begin{table}[H]
\centering
\small
\setlength{\tabcolsep}{4pt}
\caption{Evaluated pools and registered ambiguities.
Counts include 49 zero-ambiguity AmbiTab-Bird controls and 7 \benchmarkname{}-BIRD-Interact records without a registered ambiguity.}
\label{tab:dataset-counts}
\begin{tabular}{@{}llrrrr@{}}
\toprule
Corpus & Pool & Records & DBs & With ambiguities & Ambiguities \\
\midrule
AmbiTab-Bird & train & 537 & 42 & 496 & 798 \\
& dev & 127 & 11 & 119 & 198 \\
\addlinespace
\benchmarkname{}-EW-Bird & train & 2{,}091 & 67 & 2{,}091 & 4{,}417 \\
& dev & 446 & 11 & 446 & 911 \\
\addlinespace
\benchmarkname{}-BIRD-Interact & all & 600 & 22 & 593 & 1{,}755 \\
\benchmarkname{}-Ambrosia & all & 2{,}670 & 763 & 2{,}670 & 2{,}670 \\
\benchmarkname{}-AmbiQT & all & 1{,}722 & 172 & 1{,}722 & 1{,}722 \\
\benchmarkname{}-TACO & all & 187 & 16 & 187 & 197 \\
\bottomrule
\end{tabular}
\end{table}

\begin{table}[H]
\centering
\small
\setlength{\tabcolsep}{4pt}
\caption{Source-to-record conversions.
The right column states the principal selection rule and limit for each evaluated pool.}
\label{tab:dataset-construction}
\begin{tabular}{@{}>{\raggedright\arraybackslash}p{0.2\linewidth}>{\raggedright\arraybackslash}p{0.37\linewidth}>{\raggedright\arraybackslash}p{0.36\linewidth}@{}}
\toprule
Corpus & Construction & Selection and scope \\
\midrule
\benchmarkname{}-Bird & Construct alternative readings and executable alternative SQL for BIRD tasks, then use Qwen to rewrite the task specification and supply fixed resolutions~\citep{li2023bird}. & Require a Qwen3-8B reference-result miss without clarification and a match with all resolutions, followed by a blind Gemma-3-27B necessity judgment; retain zero-ambiguity controls. \\
\addlinespace
\benchmarkname{}-EW-Bird & Split BIRD evidence on semicolons, turn each hint into a masked knowledge ambiguity, and withhold it from the visible prompt~\citep{li2023bird}. & Keep a row if Qwen3-Coder-30B-A3B or Gemma-3-27B fails without evidence and succeeds with it; evaluate only retained ambiguity-bearing rows, with source train and validation pooled as train and source test as dev. \\
\addlinespace
\benchmarkname{}-BIRD-Interact & Retain source task specifications, SQL, PostgreSQL databases, and annotations~\citep{huo2025birdinteract}; normalize duplicate terms, repair one reviewed missing fragment, and freeze an audited catalog of LLM-authored natural-language responses. & Response authoring uses the request, schema, reference SQL, and annotations; the full 600-record pool is evaluated, with local database-disjoint slices also available. \\
\addlinespace
\benchmarkname{}-Ambrosia & Convert each reading in the official test partition to a record with its native intent and SQL~\citep{saparina2024ambrosia}; derive terms from scope, attachment, or vagueness annotations and copy the SQLite databases. & Exclude few-shot examples; local database-disjoint train and test slices are selected within each ambiguity type toward a 20\% test target, and both derive from the upstream test partition. \\
\addlinespace
\benchmarkname{}-AmbiQT & For the column-synonym subset, build a SQLite database containing both alternatives, copying original values to one column and resampling the other with replacement~\citep{bhaskar2023benchmarking}. & Reject protected key columns and pairs whose SQL readings fail or return equal results; emit both hidden intents for each of 861 retained source items, with the copied column assigned by a per-unit draw. \\
\addlinespace
\benchmarkname{}-TACO & Deterministically verbalize US single-table SELECT queries while withholding table or literal-value choices~\citep{deng2026taco}. & Keep executable, nonempty queries with a result-distinct table alternative; 187 of 5{,}159 source files survive, while value alternatives are not separately required to change the result. \\
\bottomrule
\end{tabular}
\end{table}

\benchmarkname{}-EW-Bird's filter measures benefit at the row level, not the necessity of each hint.
\benchmarkname{}-Bird's admission and judgment are model dependent, and \benchmarkname{}-TACO's table choices can retain year-based clues.
The \benchmarkname{}-BIRD-Interact response catalog is audited, but auditing does not establish that its responses agree with those of the perfect oracle.

\subsection{Oracle implementations}
\label[appendix]{sec:appendix-oracles}

\Cref{tab:dataset-oracles} separates the frozen-response oracles from the evaluated LLM oracles, whose access may exceed the intended boundary of \Cref{fig:main}b.

\begin{table}[H]
\centering
\small
\setlength{\tabcolsep}{4pt}
\caption{Evaluated oracle variants and their runtime access.}
\label{tab:dataset-oracles}
\begin{tabular}{@{}>{\raggedright\arraybackslash}p{0.20\linewidth}>{\raggedright\arraybackslash}p{0.35\linewidth}>{\raggedright\arraybackslash}p{0.38\linewidth}@{}}
\toprule
Variant & Corpora & Matcher and responder \\
\midrule
Deterministic intent & \benchmarkname{}-BIRD-Interact, \benchmarkname{}-Ambrosia, \benchmarkname{}-AmbiQT, \benchmarkname{}-TACO & Match terms and aliases, preferring quoted targets and then earliest mentions; return the selected frozen response verbatim, with refusals for unmatched questions or ties. \\
Deterministic knowledge & \benchmarkname{}-Bird, \benchmarkname{}-EW-Bird & Rank terms and evidence-derived concepts, prefer unresolved matches, and return the stored resolution; repeats can restate evidence. \\
Scoped LLM & \benchmarkname{}-BIRD-Interact, \benchmarkname{}-Ambrosia, \benchmarkname{}-AmbiQT, \benchmarkname{}-TACO & Classify from routing fields and respond from retained annotations; the responder sees all retained annotations and receives reference SQL when SQL grounding is used. \\
Unfiltered LLM & \benchmarkname{}-Bird, \benchmarkname{}-EW-Bird & Classify and respond with the task context, withheld evidence, and reference SQL available. \\
\bottomrule
\end{tabular}
\end{table}

The full-disambiguation baseline uses a scripted full asker that poses one canonical question per registered ambiguity to the evaluated oracle, then passes the resulting dialogue to the solver (\Cref{sec:results-full-disambiguation}).
The oracle's responses are generated through the same interaction as for the evaluated asker; they are not perfect-oracle samples by definition.

\section{Inference protocol and complete results}
\label[appendix]{sec:appendix-inference}

\subsection{Evaluation conditions and populations}

We evaluate the same task records under the no-disambiguation, evaluated-asker, and full-asker conditions described in \Cref{sec:results-setup}.
Success is the binary verifier outcome: task-specific tests where supplied, otherwise execution-result agreement with the reference SQL under the corpus's comparison rules.
The evaluated corpus pools contain the 8,380 records in \Cref{tab:dataset-counts}.
We pool the train and development splits of AmbiTab-Bird and \benchmarkname{}-EW-Bird for these descriptive benchmark comparisons; the training-policy evaluations instead report their development splits separately.
\benchmarkname{}-BIRD-Interact includes all 600 records, \benchmarkname{}-Ambrosia the full ambiguous partition of Ambrosia, and \benchmarkname{}-AmbiQT the column-synonym subset of AmbiQT.

\begin{table}[t]
\centering\small
\caption{Inference settings shared across corpora unless explicitly varied.}
\label{tab:hp-inference}
\begin{tabular}{@{}lp{8.0cm}@{}}
\toprule
Setting & Value \\
\midrule
Asker & Qwen3-8B \\
Solver & Qwen3-Coder-30B-A3B-Instruct, fixed \\
Oracle & Deterministic, Qwen3-14B, or Gemma-3-27B \\
Default asking prompt & Checklist (ledger) \\
Asker and solver decoding & Temperature 0.7; top-$p=0.8$ (trained askers in \Cref{sec:asker-evaluation-details}: top-$p=0.95$) \\
Qwen3-14B oracle decoding & Temperature 0.7; top-$p=0.8$; top-$k=20$ \\
Gemma-3-27B oracle decoding & Temperature 1.0; top-$p=0.95$; top-$k=64$ \\
Thinking & Disabled \\
Evaluation seeds & Five for oracle comparisons and the default prompt; three for the other prompts and the leakage diagnostics \\
Seed assignment & Asker and solver share the run seed; the LLM oracle has no explicit decoding seed \\
Clarification budget & $|\Amb|+3$ decisions, including a final stopping decision, when $|\Amb|>0$ \\
Unambiguous records & Skip clarification \\
SQL samples / repair rounds & One / zero \\
Oracle matching mode & Single-answer mode; no multi-answer override \\
\bottomrule
\end{tabular}
\end{table}

For a request with $m>0$ registered ambiguities, the $m+3$ decision budget allows at most $m+2$ questions because the final decision is reserved for stopping.
Within a run, all conditions use the same solver and inference server.
We compute asking gains and gaps to full disambiguation within each run before averaging over seeds.

\subsection{The full-disambiguation reference}
\label[appendix]{sec:results-full-disambiguation}

Full disambiguation uses a dialogue with the full asking policy $\AskPolicyFull$.
The scripted policy asks about each distinct registered ambiguity term once, in annotation order, using the canonical question ``Could you clarify what exactly you mean by \texttt{<term>}?''
It stops after the final question, and the same solver reads the resulting transcript to generate SQL.
The registered terms determine which questions the reference policy asks; their resolutions are supplied through the oracle's responses.
Both dialogue policies use the original task record, the same oracle implementation, and the same per-item budget.
Withheld knowledge remains hidden unless the oracle reveals it in the dialogue, and the solver does not receive the resolutions directly.

An imperfect oracle can refuse or misroute a canonical question, in which case the ambiguity remains unresolved in the transcript.
The full asker therefore targets every registered ambiguity without guaranteeing perfect measured recall against every evaluated oracle.
This is consistent with $\JFull(\Oracle)$ in \Cref{sec:theory}, which evaluates the reference asking policy under the actual oracle.
Its SR is not a universal upper bound on other asking policies, whose wording and resulting dialogue can differ.
We compute the signed gap $\JFull(\Oracle)-J(\AskPolicy;\Oracle)$ within each evaluation run; \Cref{fig:inference-overview}(a) retains negative gaps rather than truncating them to zero.
In \Cref{fig:inference-overview}(b), the evaluated asker's measured macro recall is shown with a hatched gap to the complete-coverage reference of $100\%$.
The reference describes complete coverage; it does not replace the full asker's measured credits under an imperfect oracle.
All reported full-disambiguation SRs are measured from the dialogue between the full asker and the evaluated oracle.
One exception affects both dialogue conditions: under both LLM oracles, the 162 \benchmarkname{}-EW-Bird training records on the databases \texttt{works\_cycles} and \texttt{movie\_3} exceed the oracle server's 16k-token context, end before the first question, and count as failures with zero recall credit, whereas the deterministic oracle's evaluated asker solves 95 to 112 of them per seed.
On the 2,375 \benchmarkname{}-EW-Bird records that every oracle completes, evaluated-asker SR is $72.9\%$ (deterministic), $69.5\%$ (Qwen3-14B), and $69.9\%$ (Gemma-3-27B), and full-disambiguation SR is $81.4\%$, $78.9\%$, and $75.0\%$.

\subsection{Metric definitions and aggregation}
\label[appendix]{sec:inference-metrics}

\paragraph{Macro precision and recall.}
The offline scorer maps each recorded oracle routing action to a registered ambiguity.
An exact labeled term takes precedence; otherwise it matches annotated terms against the question, response, and recorded routing rationale.
Only critical and knowledge ambiguities count.
Recall is the fraction of distinct registered terms credited over the episode, and precision is the fraction of questions receiving a previously unseen credit.
We macro-average defined episode scores, omitting precision when no question was asked and recall when no ambiguity was registered.
Thus the eligible populations differ: pooled AmbiTab-Bird has 615 ambiguity-bearing records out of 664, and \benchmarkname{}-BIRD-Interact has 593 out of 600.
The scorer also omits recall and F1 for fallback-stop episodes; this removes one \benchmarkname{}-AmbiQT episode in each of two of the 90 oracle-comparison runs, and all other inference recall denominators equal the ambiguity-bearing populations.
When shown, F1 is the mean of episode-level harmonic means, not the harmonic mean of macro precision and recall.
These are credited-coverage measures: oracle wording can affect them, including by volunteering a resolution that the question did not name.

\paragraph{Leakage diagnostics.}
\Cref{tab:inference-oracles} uses the four diagnostics from the three-seed campaign in \Cref{sec:leakage-results}, and \Cref{fig:inference-overview}(d--g) shows all four as bars.
$\DiagResp$ and $\DiagFull$ are SR differences and $\DiagRecNone$ is a recall difference, all in percentage points; their deterministic-oracle entries in \Cref{tab:inference-oracles} are shown as dashes because the deterministic oracle serves as the reference, making these differences zero by definition.
$\DiagSpecQ$ is raw mismatch action-credit recall in percentage points, so the deterministic oracle retains its measured reference recall.
The diagnostic values reproduce \Cref{tab:leakage-results}, including negative SR differences, without subtracting the deterministic reference from $\DiagSpecQ$.

\paragraph{Pooling and uncertainty.}
Within each inference seed, we pool corpus splits using the eligible denominator of each metric, then report the mean and sample standard deviation across seeds.
The error bars describe evaluation-seed variation, not uncertainty from sampling new tasks or training new policies.
In \Cref{fig:inference-overview}, SR, precision, and recall use five inference seeds, while leakage uses the separate three-seed diagnostic campaign, without pairing seeds between campaigns.
The populations also differ: the leakage campaign evaluates the development splits of AmbiTab-Bird (127 records) and \benchmarkname{}-EW-Bird (446) and a 534-record \benchmarkname{}-Ambrosia, whereas the inference columns pool 664, 2,537, and 2,670 records; the other three corpora use the same records in both campaigns.
The correlations in \Cref{fig:prompt-tradeoffs} use the 18 corpus--prompt means, subtracting each corpus's own mean from both variables before computing Pearson correlation; lines are ordinary least-squares fits to each corpus's three prompt means, and asking gain is computed relative to each run's own no-disambiguation SR.
This centering removes between-corpus mean differences; the remaining associations are descriptive, not estimates of causal effects.

\subsection{Complete oracle comparison}
\label[appendix]{sec:inference-oracles}

\Cref{tab:inference-oracles} reports the complete oracle comparison and leakage diagnostics that \Cref{fig:inference-overview} plots.
The no-disambiguation column retains only the middle (Qwen3-14B) row for each corpus; asker and full-asker success rates remain oracle-specific.
The four leakage diagnostics reproduce the three-seed results of \Cref{tab:leakage-results}, alongside the five-seed inference results.

\providecommand{\meansd}[2]{#1\textsubscript{$\pm$#2}}
\ifdefined\inferenceoraclesbox\else\newsavebox{\inferenceoraclesbox}\fi
\begin{table}[t]
\centering
\small
\setlength{\tabcolsep}{2pt}
\caption{Inference and leakage across oracles, mean\textsubscript{$\pm$SD}.
Inference: \% (5 seeds); leakage: percentage points (3 seeds; \Cref{tab:leakage-results}) on the Bird and EW-Bird development splits and the 534-record Ambrosia dataset.
No dis.: Qwen3-14B baseline; underlines: within-corpus maxima.
EW-Bird LLM-oracle rows count 162 training records that exceed the oracle's 16k-token context as failures (\Cref{sec:inference-oracles}).
Dashes: zero deterministic leakage; otherwise unavailable.}
\label{tab:inference-oracles}
\sbox{\inferenceoraclesbox}{%
\begin{tabular}{@{}llrrrrrrrrr@{}}
\toprule
 &  & \multicolumn{3}{c}{Success rate} & \multicolumn{2}{c}{Asking} & \multicolumn{4}{c}{Leakage diagnostics} \\
\cmidrule(lr){3-5}\cmidrule(lr){6-7}\cmidrule(l){8-11}
Corpus & Oracle & No dis. & Asker & Full dis. & Precision & Recall & $\DiagResp$ & $\DiagFull$ & $\DiagRecNone$ & $\DiagSpecQ$ \\
\midrule
AmbiTab-Bird & Deterministic &  & \meansd{66.7}{1.4} & \meansd{86.2}{0.7} & \meansd{33.0}{0.4} & \meansd{74.9}{0.8} & -- & -- & -- & \underline{\meansd{3.5}{0.2}} \\
 & Qwen3-14B & \meansd{26.3}{0.7} & \meansd{68.7}{0.9} & \underline{\meansd{88.0}{0.8}} & \meansd{32.6}{0.6} & \meansd{72.1}{0.9} & \meansd{3.9}{1.4} & \meansd{2.4}{4.1} & \meansd{7.8}{1.0} & \meansd{0.0}{0.0} \\
 & Gemma-3-27B &  & \underline{\meansd{75.3}{1.0}} & \meansd{87.8}{0.4} & \underline{\meansd{39.5}{0.3}} & \underline{\meansd{87.7}{0.7}} & \underline{\meansd{5.8}{0.9}} & \underline{\meansd{3.1}{2.4}} & \underline{\meansd{74.1}{0.2}} & \meansd{0.8}{0.4} \\
\midrule
\benchmarkname{}-EW-Bird & Deterministic &  & \underline{\meansd{72.4}{0.3}} & \underline{\meansd{80.7}{0.4}} & \meansd{34.7}{0.1} & \meansd{71.5}{0.2} & -- & -- & -- & \underline{\meansd{2.9}{0.3}} \\
 & Qwen3-14B & \meansd{48.2}{0.4} & \meansd{65.0}{0.8} & \meansd{73.9}{0.5} & \meansd{28.5}{0.3} & \meansd{55.7}{0.5} & \underline{\meansd{3.4}{0.2}} & \meansd{2.8}{2.2} & \meansd{0.8}{0.3} & \meansd{0.2}{0.0} \\
 & Gemma-3-27B &  & \meansd{65.4}{0.6} & \meansd{70.2}{0.1} & \underline{\meansd{38.7}{0.2}} & \underline{\meansd{74.7}{0.4}} & \meansd{-4.5}{1.8} & \underline{\meansd{3.6}{0.4}} & \underline{\meansd{61.0}{0.5}} & \meansd{1.1}{0.2} \\
\midrule
\benchmarkname{}-BIRD-Interact & Deterministic &  & \underline{\meansd{9.7}{0.3}} & \meansd{10.1}{0.7} & \meansd{32.1}{0.2} & \meansd{57.7}{0.5} & -- & -- & -- & \meansd{2.1}{0.2} \\
 & Qwen3-14B & \meansd{7.9}{0.5} & \meansd{8.9}{0.6} & \meansd{10.0}{0.4} & \meansd{29.8}{0.3} & \meansd{53.7}{0.6} & \meansd{-0.4}{0.7} & \meansd{-0.1}{0.9} & \meansd{1.9}{0.1} & \meansd{0.1}{0.0} \\
 & Gemma-3-27B &  & \meansd{9.3}{0.3} & \underline{\meansd{10.2}{0.8}} & \underline{\meansd{35.7}{0.2}} & \underline{\meansd{64.4}{0.3}} & \underline{\meansd{0.3}{0.0}} & \underline{\meansd{0.5}{0.0}} & \underline{\meansd{42.7}{0.2}} & \underline{\meansd{11.0}{0.3}} \\
\midrule
\benchmarkname{}-Ambrosia & Deterministic &  & \meansd{35.0}{0.2} & \meansd{42.1}{0.2} & \meansd{24.0}{0.2} & \meansd{71.7}{0.5} & -- & -- & -- & \meansd{0.5}{0.3} \\
 & Qwen3-14B & \meansd{22.2}{0.3} & \meansd{41.2}{0.4} & \meansd{42.1}{0.5} & \meansd{25.9}{0.2} & \meansd{77.3}{0.6} & \meansd{4.0}{1.7} & \meansd{2.9}{0.9} & \meansd{54.9}{0.6} & \meansd{7.2}{0.6} \\
 & Gemma-3-27B &  & \underline{\meansd{46.2}{0.6}} & \underline{\meansd{43.3}{0.5}} & \underline{\meansd{32.0}{0.1}} & \underline{\meansd{95.4}{0.3}} & \underline{\meansd{8.2}{2.8}} & \underline{\meansd{3.4}{0.8}} & \underline{\meansd{100.0}{0.0}} & \underline{\meansd{53.2}{0.5}} \\
\midrule
\benchmarkname{}-AmbiQT & Deterministic &  & \meansd{56.3}{0.7} & \meansd{73.7}{0.3} & \meansd{19.5}{0.2} & \meansd{58.3}{0.5} & -- & -- & -- & \meansd{13.1}{0.3} \\
 & Qwen3-14B & \meansd{33.5}{0.8} & \meansd{57.7}{0.4} & \underline{\meansd{74.6}{0.3}} & \meansd{20.4}{0.1} & \meansd{60.7}{0.3} & \meansd{1.6}{0.5} & \meansd{-0.3}{0.2} & \meansd{0.3}{0.1} & \meansd{9.9}{0.3} \\
 & Gemma-3-27B &  & \underline{\meansd{69.3}{0.7}} & \meansd{73.0}{0.6} & \underline{\meansd{31.1}{0.0}} & \underline{\meansd{93.0}{0.3}} & \underline{\meansd{2.9}{0.7}} & \underline{\meansd{0.4}{0.2}} & \underline{\meansd{100.0}{0.0}} & \underline{\meansd{70.9}{0.4}} \\
\midrule
\benchmarkname{}-TACO & Deterministic &  & \meansd{72.8}{1.5} & \meansd{83.2}{1.0} & \meansd{29.1}{0.4} & \meansd{85.1}{0.8} & -- & -- & -- & \meansd{1.0}{0.6} \\
 & Qwen3-14B & \meansd{23.9}{1.4} & \meansd{74.1}{0.8} & \underline{\meansd{83.7}{1.1}} & \meansd{28.9}{0.3} & \meansd{84.2}{1.3} & \underline{\meansd{7.0}{1.9}} & \underline{\meansd{1.1}{1.1}} & \meansd{20.5}{1.1} & \meansd{0.7}{0.6} \\
 & Gemma-3-27B &  & \underline{\meansd{79.9}{2.4}} & \meansd{83.3}{0.7} & \underline{\meansd{33.4}{0.3}} & \underline{\meansd{97.6}{0.7}} & \meansd{0.7}{0.6} & \meansd{0.4}{1.6} & \underline{\meansd{97.3}{0.0}} & \underline{\meansd{49.5}{1.8}} \\
\bottomrule
\end{tabular}}%
\ifdim\wd\inferenceoraclesbox>\linewidth
  \resizebox{\linewidth}{!}{\usebox{\inferenceoraclesbox}}%
\else
  \usebox{\inferenceoraclesbox}%
\fi
\end{table}

\subsection{Prompt-driven precision--recall tradeoff}
\label[appendix]{sec:results-prompt}

With the deterministic oracle and solver fixed, we compare three asker prompts.
\emph{Direct} asks about unclear parts of the request, while \emph{Checklist (chat)} and \emph{Checklist (ledger)} first enumerate potentially ambiguous terms.
The checklist variants differ in dialogue representation and stopping instructions (\Cref{sec:appendix-prompts}).

Both checklist prompts increase recall and reduce precision relative to Direct on every dataset (\Cref{fig:prompt-tradeoffs}a).
We use Checklist (ledger) by default because it has the highest mean recall across datasets.
Higher recall is associated with greater asking gain across prompts, though not for every prompt pair (\Cref{fig:prompt-tradeoffs}b).
This precision--recall tradeoff motivates training the asker to improve both.

\begin{figure}[t]
\centering
\includegraphics[width=\linewidth]{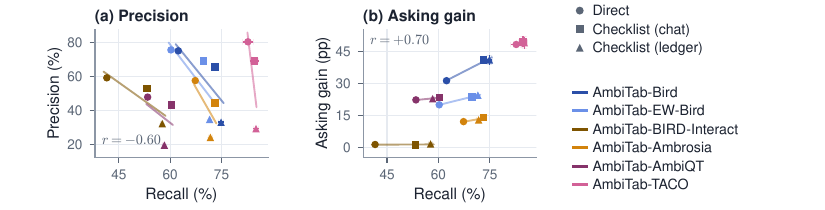}
\caption{Prompt tradeoffs under the deterministic oracle: \textbf{(a)} precision versus recall; \textbf{(b)} asking gain versus recall.
Colors identify datasets and shapes identify prompts; error bars show seed standard deviations.
Lines fit each dataset's three prompt means; correlations remove dataset means (\Cref{sec:inference-metrics}).}
\label{fig:prompt-tradeoffs}
\end{figure}

\subsection{Clarification prompts and complete prompt comparison}
\label[appendix]{sec:appendix-prompts}

The three strategies use the same question budget, deterministic oracle, and solver.
\Cref{tab:prompt-definitions} summarizes their instructions and dialogue history.
Checklist (ledger) is the default for the oracle and trained-policy comparisons.
Checklist (chat) differs from it in both dialogue representation and stopping instructions, so the comparison does not isolate the contribution of either change.

\begin{table}[t]
\centering\small
\caption{Clarification strategies.
All prompts request at most one clarification question per turn.}
\label{tab:prompt-definitions}
\begin{tabular}{@{}>{\raggedright\arraybackslash}p{0.25\linewidth}>{\raggedright\arraybackslash}p{0.70\linewidth}@{}}
\toprule
Strategy & Instructions and history \\
\midrule
Direct & Ask about unclear parts of the request without an explicit checklist; stop when the request is clear or no different ambiguity remains to be raised. \\
\addlinespace
Checklist (chat) & Enumerate potentially ambiguous terms; read prior turns as native chat messages; stop when all terms are resolved or no distinct unasked term remains. \\
\addlinespace
Checklist (ledger) & Enumerate potentially ambiguous terms; read the original request and explicit lists of questions with resolved and unavailable responses; continue while an ambiguity remains open, subject to the budget. \\
\bottomrule
\end{tabular}
\end{table}

\Cref{tab:prompt} reports the full prompt comparison.
Only the evaluated asker uses the varied clarification prompt; the no-disambiguation policy and scripted full asker are unchanged, so their SR columns are omitted from this table.
The default strategy reuses its five-seed deterministic-oracle evaluation; the other strategies each use three seeds.
The mean rows average each metric equally across the six corpora.

\providecommand{\meansd}[2]{#1\textsubscript{$\pm$#2}}
\ifdefined\inferencepromptsbox\else\newsavebox{\inferencepromptsbox}\fi
\begin{table}[t]
\centering
\small
\setlength{\tabcolsep}{2.5pt}
\caption{Clarification-prompt comparison under the deterministic oracle and otherwise the setup of \Cref{tab:inference-oracles}, which the Checklist (ledger) rows repeat.
Corpus entries are mean\textsubscript{$\pm$SD} over 5 seeds for Checklist (ledger) and 3 for Direct and Checklist (chat); rates are percentages.
Mean rows average the six corpus means equally, without SD; bold marks the highest mean recall.
All prompts share the same turn budget; AmbiTab-Bird and \benchmarkname{}-EW-Bird pool training and development splits.
The two control arms do not read the prompt and are omitted.
All rows of a corpus run the same instances; \Cref{tab:datasets} gives their counts.}
\label{tab:prompt}
\sbox{\inferencepromptsbox}{%
\begin{tabular}{@{}llrrr@{}}
\toprule
 &  & \multicolumn{1}{c}{Success rate} & \multicolumn{2}{c}{Asking} \\
\cmidrule(lr){3-3}\cmidrule(l){4-5}
Corpus & Prompt & Asker & Precision & Recall \\
\midrule
AmbiTab-Bird & Direct & \meansd{56.6}{0.4} & \meansd{75.1}{0.9} & \meansd{62.4}{0.9} \\
 & Checklist (chat) & \meansd{66.1}{0.8} & \meansd{65.5}{0.5} & \meansd{73.2}{0.8} \\
 & Checklist (ledger) & \meansd{66.7}{1.4} & \meansd{33.0}{0.4} & \meansd{74.9}{0.8} \\
\midrule
\benchmarkname{}-EW-Bird & Direct & \meansd{68.3}{0.6} & \meansd{75.6}{0.1} & \meansd{60.2}{0.3} \\
 & Checklist (chat) & \meansd{71.3}{0.2} & \meansd{69.2}{0.5} & \meansd{69.7}{0.3} \\
 & Checklist (ledger) & \meansd{72.4}{0.3} & \meansd{34.7}{0.1} & \meansd{71.5}{0.2} \\
\midrule
\benchmarkname{}-BIRD-Interact & Direct & \meansd{9.6}{0.9} & \meansd{59.2}{0.9} & \meansd{41.5}{0.3} \\
 & Checklist (chat) & \meansd{9.2}{0.4} & \meansd{52.8}{0.5} & \meansd{53.2}{0.4} \\
 & Checklist (ledger) & \meansd{9.7}{0.3} & \meansd{32.1}{0.2} & \meansd{57.7}{0.5} \\
\midrule
\benchmarkname{}-Ambrosia & Direct & \meansd{34.4}{0.5} & \meansd{57.5}{0.4} & \meansd{67.3}{0.2} \\
 & Checklist (chat) & \meansd{36.2}{0.6} & \meansd{44.4}{0.1} & \meansd{73.1}{0.3} \\
 & Checklist (ledger) & \meansd{35.0}{0.2} & \meansd{24.0}{0.2} & \meansd{71.7}{0.5} \\
\midrule
\benchmarkname{}-AmbiQT & Direct & \meansd{55.9}{0.5} & \meansd{47.8}{0.6} & \meansd{53.4}{0.6} \\
 & Checklist (chat) & \meansd{57.1}{0.9} & \meansd{43.1}{0.4} & \meansd{60.4}{0.5} \\
 & Checklist (ledger) & \meansd{56.3}{0.7} & \meansd{19.5}{0.2} & \meansd{58.3}{0.5} \\
\midrule
\benchmarkname{}-TACO & Direct & \meansd{72.5}{1.2} & \meansd{80.4}{1.6} & \meansd{82.7}{1.5} \\
 & Checklist (chat) & \meansd{73.6}{1.3} & \meansd{69.1}{0.3} & \meansd{84.7}{1.2} \\
 & Checklist (ledger) & \meansd{72.8}{1.5} & \meansd{29.1}{0.4} & \meansd{85.1}{0.8} \\
\midrule
Mean & Direct & 49.5 & 65.9 & 61.3 \\
 & Checklist (chat) & 52.3 & 57.4 & 69.1 \\
 & Checklist (ledger) & 52.2 & 28.7 & \textbf{69.9} \\
\bottomrule
\end{tabular}}%
\ifdim\wd\inferencepromptsbox>\linewidth
  \resizebox{\linewidth}{!}{\usebox{\inferencepromptsbox}}%
\else
  \usebox{\inferencepromptsbox}%
\fi
\end{table}

\section{Training the asking policy: implementation and additional results}
\label[appendix]{sec:appendix-asker}

This appendix specifies the rewards, optimization settings, and evaluation protocol for the main question-F1 policy and ablations in \Cref{sec:results-rl}.
We distinguish native validation, which uses the adapted Qwen3-8B model to generate SQL, from the common inference evaluation, which pairs each asker with the fixed Coder-30B solver.
The distinction matters because improvements in the native solver can change success rate without improving clarification.

\subsection{Question-F1 reward and trajectory-level credit}

Our main configuration trains Qwen3-8B on the 537-request AmbiTab-Bird training split with the deterministic oracle and the question-F1 reward of \Cref{eq:reward}.
We use a multi-turn GRPO variant with DAPO-style asymmetric clipping and zero-advantage group filtering~\citep{shao2024grpo,yu2025dapo}.
For a trajectory $\traj$ as in \Cref{sec:theory}, let $n$ be its number of questions and use the credited ambiguity sets $\cred_i\subseteq\Amb$ and $\cred_{\leq i}=\bigcup_{j=1}^i\cred_j$ from \Cref{met:prec-rec}.
When $n>0$ and $|\Amb|>0$, define the episode precision and recall terms
\begin{equation}
p(\traj)=\frac{|\{i\in\{1,\ldots,n\}:\cred_i\not\subseteq\cred_{\leq i-1}\}|}{n},\qquad
r(\traj)=\frac{|\cred_{\leq n}|}{|\Amb|}.
\label{eq:reward-episode-metrics}
\end{equation}
These are the quantities averaged in the theory's $\Prc$ and $\Rcl$.
The question-F1 reward takes their harmonic mean within each episode, with explicit conventions for requests without ambiguities and trajectories without productive questions:
\begin{equation}
R_{\mathrm{F1}}(\traj)=
\begin{cases}
1,& |\Amb|=0\text{ and }n=0,\\
\dfrac{2p(\traj)r(\traj)}{p(\traj)+r(\traj)},& |\Amb|>0,\ n>0,\ p(\traj)+r(\traj)>0,\\
0,&\text{otherwise}.
\end{cases}
\label{eq:reward}
\end{equation}
Thus $R_{\mathrm{F1}}$ is not the harmonic mean of the population metrics $\Prc$ and $\Rcl$.
The annotation-based transcript scorer supplies the credits against the registered ambiguities.
A request with ambiguities receives zero reward if no question is productive, including when the asker stops immediately.
A request without ambiguities receives one for abstaining and zero for asking, which defines the training reward where recall is undefined; native rollouts give such requests a single decision, meant for stopping, so the zero applies to a question emitted instead.
Repeated and irrelevant questions reduce $p(\traj)$ without increasing $r(\traj)$; a stop decision does not count toward $n$.

\paragraph{Ablation objective.}
The \emph{outcome reward} uses the theory's interpretation-dependent verifier and the final submitted solution $\hat s$:
\begin{equation}
R_{\mathrm{SR}}(\theta,\traj)=\Verif(\theta,\hat s)\in\{0,1\}.
\label{eq:outcome-reward}
\end{equation}
Its expectation is the success rate $J(\AskPolicy;\Oracle)$ for the solver used during training.
Question-F1 training ends after clarification, whereas outcome-reward training also generates and verifies SQL.
Changing the training oracle can change both the dialogue and the credits supplied by the transcript scorer.

\paragraph{Trajectory-level advantages.}
For each request, we sample $K=16$ independent trajectories $\traj_1,\ldots,\traj_K$ and standardize their rewards, with $R$ denoting the chosen reward above and suppressing the fixed interpretation $\theta$ for outcome reward:
\begin{equation}
\widehat A_k=\frac{R(\traj_k)-\overline R}{s_R+10^{-6}},\qquad
\overline R=\frac{1}{K}\sum_{k=1}^{K}R(\traj_k),\qquad
s_R^2=\frac{1}{K-1}\sum_{k=1}^{K}\bigl(R(\traj_k)-\overline R\bigr)^2.
\label{eq:ask-grpo-advantage}
\end{equation}
Each retained clarification output, including a stop decision, receives its trajectory's advantage.
Oracle responses and SQL outputs carry no loss.
Groups with constant reward, or reward standard deviation at most $10^{-6}$, receive zero advantage and are filtered before the actor update.
If every group has zero advantage, one complete group is retained to avoid an empty training batch; only the regularization terms below then contribute to the loss gradient.

\paragraph{Per-turn policy objective.}
Let $j$ index retained clarification decisions in an optimizer batch $\mathcal{B}$, let $k(j)$ identify the trajectory containing decision $j$, and let $T_j$ be its number of unmasked output tokens.
For output token $y_{jt}$ with dialogue and token prefix $h_{jt}$, define the probability ratio and sampled reference penalty as
\begin{equation}
\rho_{jt}(\phi)=\frac{\pi_\phi(y_{jt}\mid h_{jt})}{\pi_{\mathrm{old}}(y_{jt}\mid h_{jt})},\qquad
d_{jt}^{\mathrm{k2}}(\phi)=\frac{1}{2}\left(\log\frac{\pi_\phi(y_{jt}\mid h_{jt})}{\pi_{\mathrm{ref}}(y_{jt}\mid h_{jt})}\right)^2,
\end{equation}
where $\phi$ denotes trainable policy parameters, $\pi_{\mathrm{old}}$ is the pre-update actor, and $\pi_{\mathrm{ref}}$ is the frozen reference policy.
Writing $H_{jt}(\phi)$ for the entropy of $\pi_\phi(\cdot\mid h_{jt})$, the minimized loss is
\begin{equation}
\begin{aligned}
\mathcal{L}(\phi)&=\frac{1}{B}\sum_{j\in\mathcal{B}}\frac{1}{T_j}\sum_{t=1}^{T_j}\ell_{jt}(\phi),\\
\ell_{jt}(\phi)&=-\min\!\left\{\rho_{jt}(\phi)\widehat A_{k(j)},\,
\operatorname{clip}\!\left(\rho_{jt}(\phi),0.8,1.28\right)\widehat A_{k(j)}\right\}\\
&\quad{}+0.1d_{jt}^{\mathrm{k2}}(\phi)-0.01H_{jt}(\phi).
\end{aligned}
\label{eq:ask-grpo-loss}
\end{equation}
The batch normalization uses the nominal decision-sample batch size $B=512$, including for partial final batches.
Tokens are averaged within each decision, and decision losses are summed without normalizing by the number of decisions in a trajectory; dialogues with more retained decisions therefore contribute more terms.

\begin{figure}[t]
\centering
\begingroup
\sffamily
\hyphenpenalty=10000
\exhyphenpenalty=10000

\definecolor{grInk}{HTML}{1C1B21}
\definecolor{grMuted}{HTML}{5F6570}
\definecolor{grRule}{HTML}{C9C9CF}
\definecolor{grPaper}{HTML}{F8F8F5}
\definecolor{grAsk}{HTML}{005E92}    
\definecolor{grOrc}{HTML}{BF5400}    
\definecolor{grVer}{HTML}{00785A}    

\newcommand{\grTitle}{\fontsize{7}{8.2}\selectfont\bfseries}
\newcommand{\grBody}{\fontsize{6.8}{8}\selectfont}
\newcommand{\grSmall}{\fontsize{6.4}{7.6}\selectfont}
\newcommand{\grTag}{\fontsize{6}{7}\selectfont}
\newcommand{\grCode}{\fontfamily{txtt}\fontsize{6.4}{7.6}\selectfont}

\tikzset{grRound/.style={rounded corners=#1}}
\newcommand{\grRgroup}{2mm}   
\newcommand{\grRbox}{1.2mm}   
\newcommand{\grRsmall}{0.6mm} 

\tikzset{
  grBlock/.style={fill=#1!9!grPaper},
  grMidRule/.style={draw=grRule, line width=0.4pt},
  grHairRule/.style={draw=grRule, line width=0.25pt},
  grFlow/.style={draw=#1, line width=0.55pt, -{Triangle[length=1.5mm, width=1.15mm]}},
  grCustody/.style={draw=#1, line width=0.45pt, dash pattern=on 1.3pt off 1.1pt,
    -{Triangle[length=1.3mm, width=1.05mm]}},
  grStamp/.style={fill=grPaper, inner xsep=0.5mm, inner ysep=0.2mm, font=\grTag, text=#1},
}

\tikzset{
  pics/emQuestion/.style={code={
    \node[text=white, font=\fontsize{4.6}{4.6}\selectfont\bfseries] at (0,-0.02) {?};}},
  pics/emUser/.style={code={
    \fill[white] (0,0.3) circle (0.3);
    \fill[white] (-0.55,-0.62) .. controls (-0.55,-0.12) and (-0.28,-0.04) .. (0,-0.04)
      .. controls (0.28,-0.04) and (0.55,-0.12) .. (0.55,-0.62) -- cycle;}},
  pics/emDoc/.style={code={
    \fill[white] (-0.45,-0.6) rectangle (0.45,0.6);
    \draw[#1, line width=0.15mm] (-0.28,0.3) -- (0.28,0.3) (-0.28,0.02) -- (0.28,0.02)
      (-0.28,-0.26) -- (0.1,-0.26);}},
  pics/emPlus/.style={code={
    \draw[white, line width=0.3mm, line cap=round] (-0.62,0) -- (0.62,0) (0,-0.62) -- (0,0.62);}},
}

\newcommand{\gremblem}[4]{%
  \fill[#1] (#2,#3) circle (1.35mm);
  \begin{scope}[shift={(#2,#3)}, scale=1.12]\pic{#4=#1};\end{scope}}

\newcommand{\grcomp}[5]{%
  \path[grBlock=#5, grRound=\grRbox] (#1,#2-#4) rectangle (#1+#3,#2);
  \begin{scope}
    \clip[grRound=\grRbox] (#1,#2-#4) rectangle (#1+#3,#2);
    \fill[#5] (#1,#2-0.32) rectangle (#1+#3,#2);
  \end{scope}}

\newcommand{\grzone}[5]{%
  \path[grBlock=#5, grRound=\grRbox] (#1,#2) rectangle (#3,#4);
  \begin{scope}
    \clip[grRound=\grRbox] (#1,#2) rectangle (#3,#4);
    \fill[#5] (#1,#2) rectangle (#1+0.39,#4);
  \end{scope}}

\newcommand{\grground}[4]{%
  \path[fill=grPaper, grRound=\grRbox] (#1,#2) rectangle (#3,#4);
  \begin{scope}
    \clip[grRound=\grRbox] (#1,#2) rectangle (#3,#4);
    \fill[grInk] (#1,#4-0.28) rectangle (#3,#4);
    \fill[grInk] (#1,#2) rectangle (#3,#2+0.28);
  \end{scope}}

\newcommand{\grminus}[2]{%
  \filldraw[fill=grPaper, draw=grAsk, line width=0.85pt] (#1,#2) circle (1.35mm);
  \draw[grAsk, line width=0.3mm, line cap=round] ({#1-0.69},#2) -- ({#1+0.69},#2);}

\newcommand{\grterm}[1]{\sbox{\figuretermbox}{#1}%
  \rlap{\textcolor{grOrc!28}{\rule[-0.35mm]{\wd\figuretermbox}{1.7mm}}}\usebox{\figuretermbox}}

\newcommand{\grrow}[3]{%
  \node[anchor=base west, font=\grBody] at (35.2,{#1-3.3}) {$\traj_{#2}$};
  \node[anchor=base west, font=\grTag, text=grMuted] at (35.2,{#1-6.7}) {#3};}

\newcommand{\grturn}[5]{%
  \node[anchor=base west, font=\grSmall, text=grAsk] at (#1,{#2-3.3}) {#3};
  \node[anchor=base west, font=\grSmall\itshape, text=grOrc] at (#1,{#2-6.7}) {#4};
  \node[anchor=base west, font=\grSmall\itshape, text=grOrc] at (#1,{#2-9.38}) {#5};}

\newcommand{\grRZ}{116.4}
\newcommand{\grRS}{12}
\newcommand{\grRM}{0.6}

\begin{tikzpicture}[x=1mm, y=1mm,
  every node/.style={inner sep=0pt, outer sep=0pt, text=grInk}]
\path[use as bounding box] (0,0) rectangle (139.7,62.4);

\grzone{0}{48.6}{30.0}{60.8}{grAsk}
\gremblem{grAsk}{2.3}{60.8}{emDoc}
\node[anchor=base west, font=\grTitle, text=grAsk] at (1.9,57.2) {Task Specification};
\node[anchor=base west, font=\grBody] at (1.9,52.6) {``List the \grterm{top} \grterm{earners}.''};

\grcomp{34.0}{60.8}{105.7}{12.2}{grOrc}
\gremblem{grOrc}{36.1}{60.8}{emUser}
\node[grStamp=grOrc, anchor=east] at (138.2,60.8) {Frozen};
\node[anchor=base west, font=\grTitle, text=grOrc] at (35.7,57.0) {Oracle $\Oracle$};
\node[anchor=base west, font=\grSmall, text=grMuted] at (35.7,52.9) {holds the ambiguities and their resolutions};
\node[anchor=base west, font=\grCode, text=grOrc] at (84.0,56.4) {earners};
\node[anchor=base west, font=\grSmall, text=grMuted] at (95.2,56.4) {$\mapsto$};
\node[anchor=base west, font=\grSmall\itshape] at (99.0,56.4) {Salary + bonus.};
\node[anchor=base west, font=\grCode, text=grOrc] at (84.0,52.7) {top};
\node[anchor=base west, font=\grSmall, text=grMuted] at (95.2,52.7) {$\mapsto$};
\node[anchor=base west, font=\grSmall\itshape] at (99.0,52.7) {Top 3.};

\grground{34.0}{7.0}{139.7}{46.2}
\draw[grMidRule] (34.0,41.2) -- (139.7,41.2);
\draw[grHairRule] (34.0,29.8) -- (139.7,29.8) (34.0,18.4) -- (139.7,18.4);
\node[anchor=base west, font=\grTag, text=grMuted] at (35.2,43.2) {Rollout};
\node[anchor=base west, font=\grTag, text=grMuted] at (47.4,43.2) {Turn 1};
\node[anchor=base west, font=\grTag, text=grMuted] at (69.0,43.2) {Turn 2};
\node[anchor=base west, font=\grTag, text=grMuted] at (90.6,43.2) {Turn 3};
\node[anchor=base west, font=\grSmall, text=grMuted] at (\grRZ,43.2) {$R_{\mathrm{F1}}$};
\node[anchor=base, font=\grSmall, text=grMuted] at (135.6,43.2) {$\hat A$};

\grrow{41.2}{1}{precise}
\grturn{47.4}{41.2}{Rank \grterm{earners} by?}{Salary + bonus.}{}
\grturn{69.0}{41.2}{What is \grterm{top}?}{Top 3.}{}

\grrow{29.8}{2}{repeats}
\grturn{47.4}{29.8}{Rank \grterm{earners} by?}{Salary + bonus.}{}
\grturn{69.0}{29.8}{\grterm{Earners} by salary?}{That ambiguity was}{already answered.}
\grturn{90.6}{29.8}{What is \grterm{top}?}{Top 3.}{}

\grrow{18.4}{3}{off-target}
\grturn{47.4}{18.4}{Which year?}{Sorry, I can't help}{with that.}
\grturn{69.0}{18.4}{Ascending order?}{Sorry, I can't help}{with that.}

\draw[grMuted, line width=0.35pt] (\grRZ,7.0) -- (\grRZ,41.2);
\fill[grVer, grRound=\grRsmall] (\grRZ,34.0) rectangle ({\grRZ+1.0*\grRS},37.0);
\fill[grVer, grRound=\grRsmall] (\grRZ,22.6) rectangle ({\grRZ+0.8*\grRS},25.6);
\draw[grVer, line width=1.4pt] (\grRZ,11.2) -- (\grRZ,14.2); 
\draw[grMuted, line width=0.5pt, dash pattern=on 1.5pt off 1.1pt]
  ({\grRZ+\grRM*\grRS},7.0) -- ({\grRZ+\grRM*\grRS},41.2);
\node[anchor=base, font=\grTag, text=grMuted] at (\grRZ,4.3) {0};
\node[anchor=base, font=\grTag, text=grMuted] at ({\grRZ+\grRM*\grRS},4.3) {mean};

\grcomp{1.0}{30.4}{28.0}{12.6}{grAsk}
\gremblem{grAsk}{3.1}{30.4}{emQuestion}
\node[anchor=base, font=\grTitle] at (15.0,26.4) {Asking Policy};
\node[anchor=base, font=\grBody, text=grMuted] at (15.0,22.9) {$\AskPolicy$};
\node[anchor=base, font=\grSmall, text=grMuted] at (15.0,19.8) {samples $K$ rollouts};
\draw[grCustody=grAsk] (15.0,48.6) -- (15.0,30.4);

\draw[grAsk, line width=0.55pt] (29.0,24.1) -- (31.6,24.1) (31.6,35.5) -- (31.6,12.7);
\foreach \y in {35.5,24.1,12.7} {\draw[grFlow=grAsk] (31.6,\y) -- (34.3,\y);}

\draw[grFlow=grAsk] (135.6,35.5) -- (135.6,1.6) -- (15.0,1.6) -- (15.0,17.8);
\node[grStamp=grAsk, fill=white, font=\grSmall\bfseries] at (75.0,1.6) {Per-turn updates of the asking policy};
\gremblem{grAsk}{135.6}{35.5}{emPlus}
\gremblem{grAsk}{135.6}{24.1}{emPlus}
\grminus{135.6}{12.7}

\end{tikzpicture}
\endgroup
\caption{Asking-policy training with a multi-turn GRPO variant with DAPO-style asymmetric clipping and zero-advantage group filtering.
The main question-F1 reward $R_{\mathrm{F1}}$ balances productive questions and ambiguity coverage; every retained asking decision receives its trajectory's normalized advantage.
Three trajectories illustrate the update; training uses $K=16$.}
\label{fig:grpo}
\end{figure}
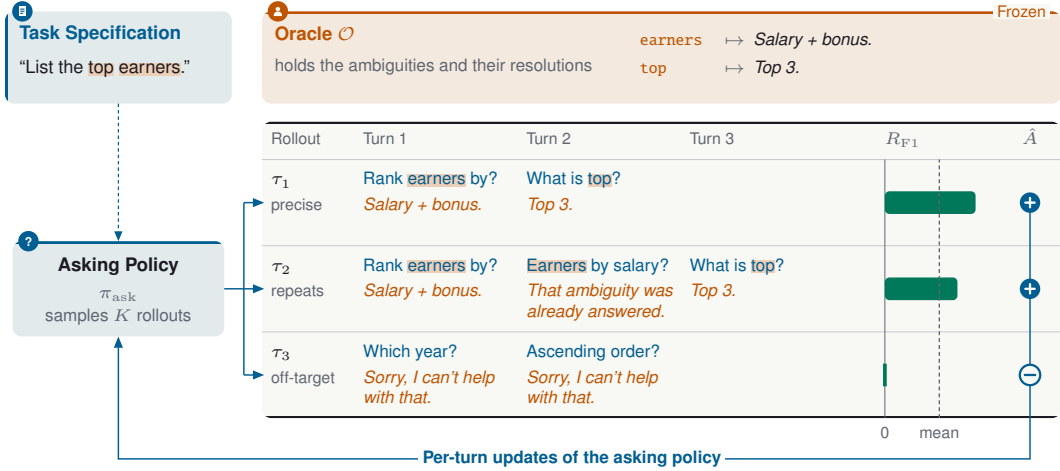

\subsection{Hyperparameters and shared model parameters}
\label[appendix]{sec:asker-optimization}

\Cref{tab:hp-train} reports the training hyperparameters.
For the main configuration with the deterministic oracle and question-F1 reward, the selected checkpoint follows \RLIteration{det_f1} rollout batches.
Each rollout batch samples eight requests and 16 trajectories per request.
Each retained model decision becomes a separate training sample, with a nominal global batch size of 512 samples, partial final batches, and dynamic token packing.
With zero-advantage filtering, each rollout batch retained at most one nominal batch of decision samples, so every rollout batch produced exactly one optimizer update.
The learning-rate schedule assumes four decisions per trajectory; this estimate does not cap the number of retained decisions.

The asker and native SQL solver share the same frozen Qwen3-8B base and trainable LoRA adapter.
Masking SQL outputs out of the loss does not freeze the solver: updates from clarification tokens also change the adapter used for SQL generation.

\begin{table}[t]
\centering\small
\caption{Asker-training hyperparameters.
Oracle and reward choices vary as specified in \Cref{sec:ablations-details}.}
\label{tab:hp-train}
\begin{tabular}{@{}lp{8.0cm}@{}}
\toprule
Setting & Value \\
\midrule
Initialization / backend & Qwen3-8B / FSDP \\
Optimizer & AdamW, $(\beta_1,\beta_2)=(0.9,0.98)$, $\epsilon=10^{-8}$ \\
Learning rate / schedule & Peak $5\times10^{-5}$; cosine decay to 0, 5\% linear warmup \\
Schedule horizon & 136 optimizer updates, estimated from four decisions per trajectory \\
Weight decay / gradient clipping & 0 / 1.0 \\
LoRA rank / alpha / dropout & 32 / 64 / 0 (adapter scaling $\alpha/r=2$) \\
LoRA targets & Attention q/k/v/o and feed-forward gate/up/down projections \\
Requests per rollout / group size & 8 / 16 trajectories per request \\
Rollout / nominal training batch & 128 trajectories / 512 retained decision samples \\
Policy ratio clipping & $[0.8,1.28]$ \\
Loss aggregation & Token mean per retained decision; sum divided by nominal batch size \\
Group filtering & Zero-advantage groups dropped; retain one group if all are zero \\
Reference-policy penalty & Sampled \texttt{k2}, coefficient 0.1 \\
Entropy coefficient & 0.01 \\
Policy decoding & Temperature 0.8; thinking disabled \\
Main reward & Question F1 ($R_{\mathrm{F1}}$; \Cref{eq:reward}) \\
Main training oracle & Deterministic; no oracle LLM is instantiated \\
LLM-oracle ablation & Qwen3-14B; temperature 0; thinking disabled \\
Rollout context / response ceiling & 32,768 / 256 tokens \\
Clarification prompt & Checklist (ledger); question budget visible \\
Question budget & At most $|\Amb|+2$ questions for $|\Amb|>0$, then a stop decision \\
Checkpoint / validation cadence & Every eight rollout batches; initial validation \\
Native validation samples & One per request \\
\bottomrule
\end{tabular}
\end{table}

Native rollouts on requests with registered ambiguities allow $|\Amb|+3$ decisions, reserving the final decision for stopping.
On requests without registered ambiguities, only a stopping decision is allowed.
Question-F1 training skips SQL generation.

\subsection{Ablation design and checkpoint selection}
\label[appendix]{sec:ablations-details}

We compare the main policy trained with question-F1 reward and the deterministic oracle with two outcome-reward ablations, one using the deterministic oracle and one using Qwen3-14B as the oracle.
All three use AmbiTab-Bird and its 127-request development split, with one training seed per configuration.
Every configuration updates clarification outputs only.
We evaluate the designated checkpoints: iteration \RLIteration{det_f1} for question-F1 reward, \RLIteration{det_sr} for outcome reward with the deterministic oracle, and \RLIteration{llm_sr} for outcome reward with the LLM oracle.
These checkpoints are fixed for the reported comparison rather than selected by a common validation-score rule.
All RL learning curves and ablation figures use a shared plotting window through iteration \RLTrainingIteration{det_f1}.
This window contains zero-based rollout batches 0-47, which precede checkpoint 48, and native evaluations through checkpoint 48.
All three configurations have complete common-inference evaluations across three seeds on all six corpora.
The untrained reference retains its existing five seeds.
The AmbiTab-Bird development split was used to inspect training and checkpoints, so its scores are development results, not held-out test estimates.
Because every rollout batch produced one optimizer update, the shared window corresponds to 48 updates per configuration and the evaluated checkpoints to \RLIteration{det_f1}, \RLIteration{det_sr}, and \RLIteration{llm_sr} updates.

\subsection{Native learning curves and objective mismatch}
\label[appendix]{sec:asker-native-diagnostics}

\Cref{fig:rl-main,fig:rl-summary,fig:rl-bi,tab:ablations-native} separate changes in clarification quality from changes in native SQL success.

In \Cref{fig:rl-main}, training reward averages one reward per logged trajectory before groups without training signal are filtered; the solid curve is a debiased exponential moving average with weight 0.95.
Iteration 0 is the run's native evaluation before training.
Native checkpoints use the Qwen3-8B solver and native validation protocol; the untrained reference uses the fixed Coder-30B solver and common inference protocol.
The reference uses the 127-request AmbiTab-Bird development slice of the deterministic-oracle inference runs underlying \Cref{tab:inference-oracles}, which pools training and development requests in its displayed AmbiTab-Bird row.
The dotted reference lines show the development-only mean across five seeds.
Common-solver results appear in \Cref{tab:ablations-summary}.

\begin{figure}[t]
\centering
\includegraphics[width=0.80\linewidth]{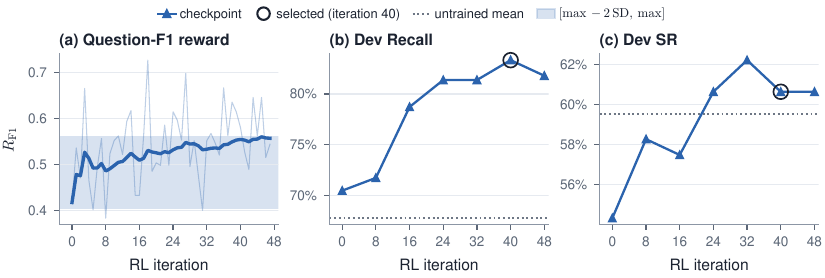}
\caption{Learning curves through iteration \RLTrainingIteration{det_f1} for the main deterministic-oracle configuration with question-F1 reward $R_{\mathrm{F1}}$: \textbf{(a)} training reward, \textbf{(b)} native AmbiTab-Bird development recall, and \textbf{(c)} SR.
Faint and solid lines show mean trajectory rewards per rollout batch and their smoothed trend; rings mark the selected checkpoint at iteration \RLIteration{det_f1}.
The blue band in (a) extends from the maximum smoothed reward down by two sample standard deviations of the per-batch mean rewards.
Dotted lines show the untrained development-only inference mean across five seeds, using the deterministic-oracle runs underlying \Cref{tab:inference-oracles}.
Native and reference evaluations use different protocols and solvers.}
\label{fig:rl-main}
\end{figure}

In \Cref{fig:rl-summary}(a-c), the trained checkpoints use the deterministic validation oracle and adapted Qwen3-8B solver.
Every row shows differences from the same untrained Qwen3-8B asker reference on the 127-request AmbiTab-Bird development split, averaged over all five evaluation seeds from the deterministic-oracle inference runs underlying \Cref{tab:inference-oracles}.
Gray reference bands show one sample standard deviation on either side of the reference mean, using the same development-only reference as \Cref{fig:rl-main}.
Colored row bands show a two-SD tolerance: each extends from that row's maximum trained-checkpoint value down by twice the corresponding reference standard deviation.
The reference uses the common inference protocol and fixed Coder-30B solver, whereas the trained checkpoints use native validation; these differences therefore include protocol effects, and SR also includes the solver difference.
\Cref{tab:ablations-dev} provides the comparison under the common inference protocol.
Colored points represent the available trained checkpoints, and rings identify the designated evaluation checkpoint for each configuration.
The main policy trained with question-F1 reward appears first, followed by the ablations.
Row labels give the numbers of trained checkpoints and rollout batches; $R_{\mathrm{SR}}$ denotes the outcome reward.
Each configuration contributes one training run, so the checkpoint comparisons are descriptive.

\begin{figure}[t]
\centering
\includegraphics[width=\linewidth]{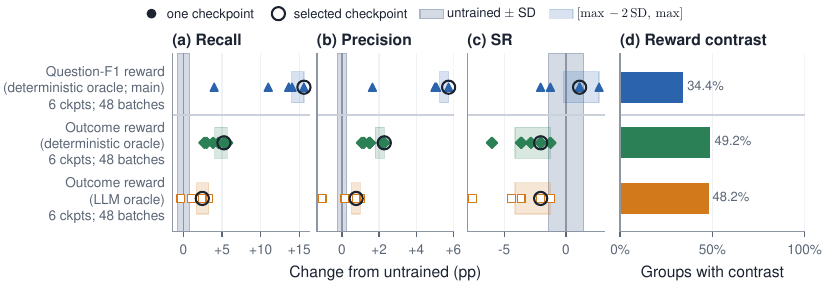}
\caption{The main policy trained with question-F1 reward and the oracle/reward ablations through iteration \RLTrainingIteration{det_f1}: native AmbiTab-Bird development checkpoints relative to the five-seed untrained inference reference.
\textbf{(a-c)} Recall, precision, and SR differences in percentage points.
Rings mark the designated evaluation checkpoints, including iteration \RLIteration{det_sr} for outcome reward with the deterministic oracle.
Gray shading shows $\pm$ one sample standard deviation across the five development-only reference seeds from the deterministic-oracle runs underlying \Cref{tab:inference-oracles}.
Colored row bands span $[\max-2\,\mathrm{SD},\max]$, using each row's maximum trained-checkpoint value and the corresponding reference SD.
Trained checkpoints use native validation with the Qwen3-8B solver; the reference uses the common inference protocol with the fixed Coder-30B solver.
\textbf{(d)} Fraction of logged request groups with nonconstant trajectory rewards before group filtering, pooling counts over the same 48 rollout batches for each configuration.}
\label{fig:rl-summary}
\end{figure}

At iteration \RLIteration{det_f1}, the main policy trained with question-F1 reward reaches $\RLNative{det_f1}{recall}\%$ native recall and $\RLNative{det_f1}{SR}\%$ native SR, compared with $\RLNativeBase{det_f1}{recall}\%$ and $\RLNativeBase{det_f1}{SR}\%$ at initialization.
With outcome reward, the deterministic-oracle policy at iteration \RLIteration{det_sr} reaches $\RLNative{det_sr}{recall}\%$ recall and $\RLNative{det_sr}{SR}\%$ SR; the LLM-oracle policy at iteration \RLIteration{llm_sr} reaches $\RLNative{llm_sr}{recall}\%$ recall and $\RLNative{llm_sr}{SR}\%$ SR.
Outcome reward depends on the solver's ability to use the responses; F1 balances question efficiency and ambiguity coverage.
The fixed-solver evaluation in \Cref{tab:ablations-dev} measures the trained asking policies under a common solution policy and protocol.
The two evaluations need not rank checkpoints in the same order.

The outcome-reward runs provide the training-oracle comparison under the common inference protocol.
There is no LLM-oracle F1 run in this ablation set.
One training seed per configuration does not establish a general ranking of objectives or training oracles.

\paragraph{Dataset choice matters.}
A separate outcome-reward run on \benchmarkname{}-BIRD-Interact crashed before completing training.
Its trajectory failed to sustain improvements in clarification quality over the shown window (\Cref{fig:rl-bi}), illustrating that the training dataset matters for learning to ask.
The run uses the LLM oracle and its 120-request native validation split; the shown window extends through iteration \RLTrainingIteration{llm_sr_bi}.
The columns show reward, training recall, training precision, and optimization loss.
Faint lines show rollout batches; solid lines use a debiased exponential moving average with weight 0.95.
Reward averages retained decision samples after constant-reward groups are filtered, so it is not the unfiltered rollout SR.
Native recall falls from $\RLNative{llm_sr_bi}{recall}\%$ at iteration \RLIteration{llm_sr_bi} to $\RLNativeLast{llm_sr_bi}{recall}\%$ at the end of the shown window, iteration \RLTrainingIteration{llm_sr_bi}, where native SR remains low at $\RLNativeLast{llm_sr_bi}{SR}\%$.
Few request groups provide differing binary rewards, limiting the within-request reward contrast available for policy updates.
The curves document the shown training window but do not identify the cause of the later crash.
This run also used a larger response budget and a different schedule estimate than the current AmbiTab-Bird runs, so it does not isolate the effect of the training corpus.
We retain this run only as an illustration of unsuccessful training; it is excluded from the configuration and results tables.

\begin{figure}[t]
\centering
\includegraphics[width=\linewidth]{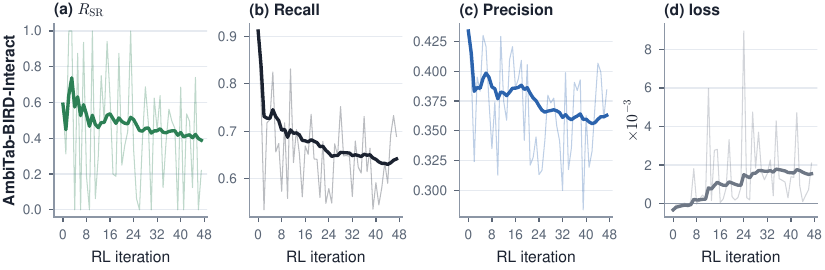}
\caption{Outcome-reward training dynamics on \benchmarkname{}-BIRD-Interact through iteration \RLTrainingIteration{llm_sr_bi}.
Training recall and precision decline over the window, and binary rewards provide limited within-request contrast.}
\label{fig:rl-bi}
\end{figure}

\paragraph{Reward contrast.}
\Cref{fig:rl-summary}(d) reports the fraction of request groups whose available trajectory rewards are not all equal, with 16 trajectories sampled per request.
Counts are logged before group filtering and pooled across the same 48 rollout batches for each configuration; row labels give the number of batches.
Such groups supply the within-request reward contrast used for policy updates, whereas constant-reward groups have zero policy-gradient advantage (\Cref{eq:ask-grpo-advantage}).
These fractions describe reward contrast over the shared plotting window.
They measure available reward contrast rather than question quality.

\providecommand{\meansd}[2]{#1\textsubscript{$\pm$#2}}
\begin{table}[t]
\centering\small
\setlength{\tabcolsep}{3.1pt}
\caption{Native validation before training and at the selected checkpoint, using the adapted Qwen3-8B solver. Rates are percentages. All rows use AmbiTab-Bird dev ($n=127$). Question-F1 reward is the main configuration, trained with the deterministic oracle; the remaining trained policies are ablations. Selection and oracle settings are given in \Cref{sec:ablations-details,sec:asker-evaluation-details}.}
\label{tab:ablations-native}
\begin{tabular}{@{}lrrrr@{}}
\toprule
Policy & SR & Precision & Recall & F1 \\
\midrule
Question-F1 reward (untrained) & 54.3 & 32.2 & 70.4 & 43.7 \\
Question-F1 reward & 60.6 & 36.8 & 83.3 & 50.6 \\
\addlinespace[2pt]
Outcome reward (deterministic oracle) (untrained) & 55.9 & 31.0 & 68.1 & 42.2 \\
Outcome reward (deterministic oracle) & 57.5 & 33.3 & 73.0 & 45.3 \\
\addlinespace[2pt]
Outcome reward (LLM oracle) (untrained) & 52.8 & 30.6 & 68.1 & 41.9 \\
Outcome reward (LLM oracle) & 57.5 & 31.8 & 70.2 & 43.4 \\
\addlinespace[2pt]
\bottomrule
\end{tabular}
\end{table}

\subsection{Common inference protocol and aggregation}
\label[appendix]{sec:asker-evaluation-details}

\paragraph{Fixed solver and oracle.}
\Cref{tab:ablations-dev} evaluates the trained askers with the deterministic oracle and a fixed Qwen3-Coder-30B-A3B-Instruct solver, using Checklist (ledger), temperature 0.7, and no SQL debugging.
The main and complete tables use the same designated checkpoints; the question-F1 policy and both outcome-reward ablations each have three completed evaluation seeds.
The question budget matches native rollouts on requests with registered ambiguities; requests without ambiguities skip clarification entirely.
Native validation instead generates SQL with the adapted Qwen3-8B policy at temperature 0.8.
On AmbiTab-Bird dev, its oracle is deterministic on the 119 requests with registered ambiguities and silent on the remaining eight, including for policies trained with the LLM oracle.
\benchmarkname{}-BIRD-Interact native validation uses the LLM oracle.
The untrained rows in \Cref{tab:ablations-native} report each run's own iteration-0 native validation with the untrained Qwen3-8B solver and its single recorded training seed.

\paragraph{Untrained reference and uncertainty.}
The untrained reference uses the Qwen3-8B asker with the same clarification protocol, except for nucleus sampling: its asker samples at top-$p=0.8$, whereas the trained-policy evaluations leave the asker at the serving default of top-$p=0.95$; asker temperature, top-$k=20$, solver decoding, prompt, budget, and oracle are identical.
Its scores retain all five evaluation seeds from the deterministic-oracle runs used in \Cref{tab:inference-oracles}, taking the development split alone for AmbiTab-Bird and \benchmarkname{}-EW-Bird.
Trained-policy cells average all three evaluation seeds; \Cref{tab:ablations-dev} also reports their sample standard deviations.
These seeds vary inference for one trained policy per configuration; they do not measure variation across training runs.
All gains subtract the five-seed untrained mean from the trained policy's three-seed mean, including the differences in \Cref{tab:ablations-summary}.
\Cref{tab:inference-oracles} instead pools each of these two corpora's train and development splits.

\paragraph{Eligible populations.}
SR averages $\Verif(\theta,\hat s)$ over requests.
Precision and recall are macro averages of the episode quantities in \Cref{met:prec-rec}, omitting precision when $n=0$ and recall when $|\Amb|=0$.
Their denominators can differ: for example, AmbiTab-Bird dev recall uses 119 of 127 requests.
Reported F1 is the average of episode-level F1, not the harmonic mean of macro precision and recall.

\benchmarkname{}-BIRD-Interact pools its full dataset within each evaluation seed using each metric's eligible denominator.
\benchmarkname{}-Ambrosia uses the full ambiguous split of Ambrosia, and \benchmarkname{}-AmbiQT uses the column-synonym subset of AmbiQT.
The Mean column in \Cref{tab:ablations-dev} equally weights all six evaluation families, AmbiTab-Bird dev, \benchmarkname{}-EW-Bird dev, \benchmarkname{}-BIRD-Interact, \benchmarkname{}-Ambrosia, \benchmarkname{}-AmbiQT, and \benchmarkname{}-TACO, instead of pooling requests.
For the family mean, we first average all six families within each evaluation seed, then report the mean and sample standard deviation of these per-seed averages over all three trained-policy seeds or all five untrained-reference seeds.

\subsection{Complete inference results}

\Cref{tab:ablations-dev} reports complete three-seed evaluations of the main policy trained with question-F1 reward and both outcome-reward ablations across all six corpora.
The main-text summary and appendix use the same runs and checkpoints.
For outcome-reward training, the policy trained with the deterministic oracle at iteration \RLIteration{det_sr} has six-corpus means of $\RLMean{det_sr}{precision}\%$ precision, $\RLMean{det_sr}{recall}\%$ recall, and $\RLMean{det_sr}{SR}\%$ SR, compared with $\RLMean{llm_sr}{precision}\%$, $\RLMean{llm_sr}{recall}\%$, and $\RLMean{llm_sr}{SR}\%$ for the policy trained with the LLM oracle.
Both outcome-reward ablations fall below the question-F1 policy on these three means, while differences between the two training oracles are small and vary across corpora.
\providecommand{\meansd}[2]{#1\textsubscript{$\pm$#2}}
\begin{table}[t]
\centering\small
\setlength{\tabcolsep}{2.0pt}
\caption{Training and transfer with a fixed Coder-30B solver and deterministic oracle. Question-F1 reward is the main configuration, trained with the deterministic oracle; the remaining trained policies are ablations. Question-F1 reward uses iteration 40; Outcome reward (deterministic oracle) uses iteration 16; Outcome reward (LLM oracle) uses iteration 40. Selected trained policies use all three evaluation seeds; the untrained baseline retains five seeds. Rates are percentages, reported as means with sample SD in subscripts. Mean equally weights all six families within each seed, using only seeds complete across all six. Bold values mark the highest mean in each column within each metric, including ties. AmbiTab-Bird and \benchmarkname{}-EW-Bird use dev; \benchmarkname{}-AmbiQT uses column synonyms. Dashes denote unavailable results. Protocol details appear in \Cref{sec:ablations-details,sec:asker-evaluation-details}.}
\label{tab:ablations-dev}
\begin{tabular*}{\textwidth}{@{\extracolsep{\fill}}lrrrrrrr@{}}
\toprule
\multicolumn{8}{l}{\textbf{SQL success}} \\
Policy & \shortstack[r]{AmbiTab-\\Bird} & \shortstack[r]{\benchmarkname{}-\\EW-Bird} & \shortstack[r]{\benchmarkname{}-\\BIRD-Interact} & \shortstack[r]{\benchmarkname{}-\\Ambrosia} & \shortstack[r]{\benchmarkname{}-\\AmbiQT} & \shortstack[r]{\benchmarkname{}-\\TACO} & Mean \\
\midrule
Untrained & \meansd{59.5}{1.4} & \meansd{69.0}{0.8} & \meansd{9.7}{0.3} & \meansd{35.0}{0.2} & \meansd{56.3}{0.7} & \meansd{72.8}{1.5} & \meansd{50.4}{0.4} \\
\textbf{Question-F1 reward} & \textbf{\meansd{71.9}{0.5}} & \textbf{\meansd{72.2}{1.0}} & \meansd{9.7}{0.3} & \textbf{\meansd{37.3}{0.2}} & \textbf{\meansd{59.4}{0.3}} & \textbf{\meansd{77.0}{1.4}} & \textbf{\meansd{54.6}{0.2}} \\
\shortstack[l]{Outcome reward\\(deterministic oracle)} & \meansd{62.7}{1.2} & \meansd{71.1}{1.6} & \textbf{\meansd{9.8}{0.1}} & \meansd{35.5}{0.7} & \meansd{56.8}{0.2} & \meansd{74.3}{1.9} & \meansd{51.7}{0.4} \\
\shortstack[l]{Outcome reward\\(LLM oracle)} & \meansd{64.3}{1.2} & \meansd{71.4}{1.6} & \meansd{9.7}{0.6} & \meansd{35.9}{0.6} & \meansd{56.8}{0.5} & \meansd{73.4}{0.8} & \meansd{51.9}{0.6} \\
\midrule
\multicolumn{8}{l}{\textbf{Precision}} \\
Policy & \shortstack[r]{AmbiTab-\\Bird} & \shortstack[r]{\benchmarkname{}-\\EW-Bird} & \shortstack[r]{\benchmarkname{}-\\BIRD-Interact} & \shortstack[r]{\benchmarkname{}-\\Ambrosia} & \shortstack[r]{\benchmarkname{}-\\AmbiQT} & \shortstack[r]{\benchmarkname{}-\\TACO} & Mean \\
\midrule
Untrained & \meansd{31.0}{0.2} & \meansd{33.6}{0.5} & \meansd{32.1}{0.2} & \meansd{24.0}{0.2} & \meansd{19.5}{0.2} & \meansd{29.1}{0.4} & \meansd{28.2}{0.1} \\
\textbf{Question-F1 reward} & \textbf{\meansd{37.0}{0.5}} & \textbf{\meansd{38.0}{0.3}} & \textbf{\meansd{36.4}{0.4}} & \textbf{\meansd{27.3}{0.1}} & \textbf{\meansd{21.7}{0.1}} & \textbf{\meansd{30.7}{0.3}} & \textbf{\meansd{31.8}{0.2}} \\
\shortstack[l]{Outcome reward\\(deterministic oracle)} & \meansd{33.4}{0.4} & \meansd{34.6}{0.3} & \meansd{33.0}{0.7} & \meansd{24.6}{0.1} & \meansd{19.8}{0.2} & \meansd{29.7}{0.5} & \meansd{29.2}{0.2} \\
\shortstack[l]{Outcome reward\\(LLM oracle)} & \meansd{32.2}{0.7} & \meansd{35.8}{0.5} & \meansd{33.1}{0.4} & \meansd{24.5}{0.1} & \meansd{19.7}{0.1} & \meansd{29.8}{0.5} & \meansd{29.2}{0.2} \\
\midrule
\multicolumn{8}{l}{\textbf{Recall}} \\
Policy & \shortstack[r]{AmbiTab-\\Bird} & \shortstack[r]{\benchmarkname{}-\\EW-Bird} & \shortstack[r]{\benchmarkname{}-\\BIRD-Interact} & \shortstack[r]{\benchmarkname{}-\\Ambrosia} & \shortstack[r]{\benchmarkname{}-\\AmbiQT} & \shortstack[r]{\benchmarkname{}-\\TACO} & Mean \\
\midrule
Untrained & \meansd{67.7}{0.8} & \meansd{69.8}{1.1} & \meansd{57.7}{0.5} & \meansd{71.7}{0.5} & \meansd{58.3}{0.5} & \meansd{85.1}{0.8} & \meansd{68.4}{0.3} \\
\textbf{Question-F1 reward} & \textbf{\meansd{83.4}{1.1}} & \textbf{\meansd{79.3}{0.6}} & \textbf{\meansd{65.1}{0.7}} & \textbf{\meansd{81.8}{0.2}} & \textbf{\meansd{65.0}{0.3}} & \textbf{\meansd{89.8}{0.8}} & \textbf{\meansd{77.4}{0.3}} \\
\shortstack[l]{Outcome reward\\(deterministic oracle)} & \meansd{72.9}{0.5} & \meansd{72.0}{0.7} & \meansd{59.2}{1.4} & \meansd{73.7}{0.4} & \meansd{59.3}{0.6} & \meansd{87.0}{1.5} & \meansd{70.7}{0.3} \\
\shortstack[l]{Outcome reward\\(LLM oracle)} & \meansd{71.3}{1.7} & \meansd{74.7}{0.5} & \meansd{59.4}{0.6} & \meansd{73.5}{0.2} & \meansd{59.2}{0.3} & \meansd{87.3}{1.5} & \meansd{70.9}{0.5} \\
\bottomrule
\end{tabular*}
\end{table}

\section{Leakage experiments}
\label[appendix]{sec:leakage-results}

All runs use the main evaluation settings: the Qwen3-8B asker with Checklist (ledger), the fixed Qwen3-Coder-30B-A3B-Instruct solver, the Qwen3-14B and Gemma-3-27B oracles, all six datasets, and three evaluation seeds (234 runs, 141{,}024 episodes).
The datasets are the development splits of AmbiTab-Bird (127 records) and \benchmarkname{}-EW-Bird (446), a 534-record \benchmarkname{}-Ambrosia, and the full \benchmarkname{}-BIRD-Interact, \benchmarkname{}-AmbiQT, and \benchmarkname{}-TACO pools.
Levels are mean $\pm$ sample standard deviation over seeds; paired differences carry 95\% AVT-cluster bootstrap intervals.
All values are in percentage points.

\paragraph{Untrained asker.}
\Cref{tab:leakage-results} reports the four diagnostics.
$\DiagResp$ and $\DiagFull$ are SR differences; $\DiagSpecQ$ is a recall and $\DiagRecNone$ a recall difference, both over the records with at least one ambiguity.
The two LLM-oracle entries are not paired estimates of their difference.
On \benchmarkname{}-AmbiQT the reference recall of the specification-randomized policy is 13.1\%, so $\DiagSpecQ$ there carries a large reference term; differences between oracles are unaffected, because the term cancels.

\begin{table}[t]
\centering
\scriptsize
\setlength{\tabcolsep}{2pt}
\caption{Leakage diagnostics of the untrained asker, mean $\pm$ SD over three seeds. AmbiTab-Bird and \benchmarkname{}-EW-Bird are development splits, and \benchmarkname{}-Ambrosia has 534 records. Bold marks the larger of the two LLM-oracle values. For the deterministic oracle, which is the reference, $\DiagResp$, $\DiagFull$, and $\DiagRecNone$ are 0 by definition, and $\DiagSpecQ$ is the reference recall of the specification-randomized policy.}
\label{tab:leakage-results}
\resizebox{\linewidth}{!}{
\begin{tabular}{@{}llrrrr@{}}
\toprule
Oracle & Dataset & $\DiagResp$ & $\DiagSpecQ$ & $\DiagFull$ & $\DiagRecNone$ \\
\midrule
Qwen3-14B & AmbiTab-Bird dev & $3.94\pm1.36$ & $0.00\pm0.00$ & $2.36\pm4.09$ & $7.84\pm0.97$ \\
 & \benchmarkname{}-EW-Bird dev & $\mathbf{3.36\pm0.22}$ & $0.22\pm0.00$ & $2.77\pm2.15$ & $0.82\pm0.26$ \\
 & \benchmarkname{}-BIRD-Interact & $-0.44\pm0.67$ & $0.14\pm0.05$ & $-0.11\pm0.95$ & $1.88\pm0.10$ \\
 & \benchmarkname{}-Ambrosia & $4.00\pm1.68$ & $7.24\pm0.57$ & $2.93\pm0.94$ & $54.87\pm0.56$ \\
 & \benchmarkname{}-AmbiQT & $1.61\pm0.47$ & $9.85\pm0.29$ & $-0.31\pm0.23$ & $0.35\pm0.10$ \\
 & \benchmarkname{}-TACO & $\mathbf{6.95\pm1.93}$ & $0.71\pm0.56$ & $\mathbf{1.07\pm1.07}$ & $20.50\pm1.11$ \\
\midrule
Gemma-3-27B & AmbiTab-Bird dev & $\mathbf{5.77\pm0.91}$ & $\mathbf{0.84\pm0.42}$ & $\mathbf{3.15\pm2.36}$ & $\mathbf{74.09\pm0.24}$ \\
 & \benchmarkname{}-EW-Bird dev & $-4.48\pm1.78$ & $\mathbf{1.10\pm0.24}$ & $\mathbf{3.59\pm0.45}$ & $\mathbf{60.99\pm0.47}$ \\
 & \benchmarkname{}-BIRD-Interact & $\mathbf{0.33\pm0.00}$ & $\mathbf{11.04\pm0.28}$ & $\mathbf{0.50\pm0.00}$ & $\mathbf{42.69\pm0.21}$ \\
 & \benchmarkname{}-Ambrosia & $\mathbf{8.24\pm2.76}$ & $\mathbf{53.25\pm0.47}$ & $\mathbf{3.37\pm0.82}$ & $\mathbf{100.00\pm0.00}$ \\
 & \benchmarkname{}-AmbiQT & $\mathbf{2.94\pm0.65}$ & $\mathbf{70.94\pm0.44}$ & $\mathbf{0.37\pm0.19}$ & $\mathbf{100.00\pm0.00}$ \\
 & \benchmarkname{}-TACO & $0.71\pm0.62$ & $\mathbf{49.47\pm1.75}$ & $0.36\pm1.63$ & $\mathbf{97.33\pm0.00}$ \\
\midrule
Deterministic & AmbiTab-Bird dev & 0 & $3.50\pm0.24$ & 0 & 0 \\
 & \benchmarkname{}-EW-Bird dev & 0 & $2.91\pm0.29$ & 0 & 0 \\
 & \benchmarkname{}-BIRD-Interact & 0 & $2.13\pm0.24$ & 0 & 0 \\
 & \benchmarkname{}-Ambrosia & 0 & $0.50\pm0.29$ & 0 & 0 \\
 & \benchmarkname{}-AmbiQT & 0 & $13.09\pm0.32$ & 0 & 0 \\
 & \benchmarkname{}-TACO & 0 & $0.98\pm0.62$ & 0 & 0 \\
\bottomrule
\end{tabular}
}
\end{table}

\subsection{Responder access to the reference SQL}
\label[appendix]{sec:leakage-sql-full-access}

We repeat the experiments of \Cref{tab:leakage-results} with one change: the LLM responder's access to the full reference SQL.
On \benchmarkname{}-Ambrosia, \benchmarkname{}-AmbiQT, and \benchmarkname{}-TACO the default responder does not read it, and the new runs add it; on \benchmarkname{}-EW-Bird and AmbiTab-Bird the default responder reads it and its clause segments, and the new runs remove both.
The matcher is unchanged, so each difference compares two runs of the same policy under the actual oracle, paired by task: the untrained asker for $\DiagResp$, the adversarial full-disambiguation policy for $\DiagFull$, and the generic policy for $\DiagRecNone$.
The new runs comprise 82 evaluation runs and 47{,}640 episodes; a few untrained-asker runs with the SQL shown are reused from an earlier study with the same settings.
The two conditions come from different inference-server sessions.

\begin{table}[t]
\centering
\scriptsize
\setlength{\tabcolsep}{3pt}
\caption{Change in the leakage diagnostics when the LLM responder reads the full reference SQL, shown minus hidden, three seeds pooled, with paired 95\% AVT-cluster bootstrap intervals. AmbiTab-Bird and \benchmarkname{}-EW-Bird are development splits, and \benchmarkname{}-Ambrosia is a 534-record. New run: the condition added for this study.}
\label{tab:leakage-sql-full-access}
\resizebox{\linewidth}{!}{
\begin{tabular}{@{}lllrrr@{}}
\toprule
& & & \multicolumn{3}{c}{Difference, reference SQL shown minus hidden} \\
\cmidrule(lr){4-6}
Oracle & Dataset & New run & $\DiagResp$ & $\DiagFull$ & $\DiagRecNone$ \\
\midrule
Qwen3-14B & AmbiTab-Bird dev & SQL hidden & 3.15\;[$-0.26$,\,6.56] & 0.52\;[$-2.36$,\,3.67] & 0.28\;[0.00,\,1.12] \\
 & \benchmarkname{}-EW-Bird dev & SQL hidden & 3.96\;[1.42,\,6.35] & 2.84\;[0.22,\,5.61] & 0.00\;[$-0.22$,\,0.22] \\
 & \benchmarkname{}-Ambrosia & SQL shown & 2.87\;[0.88,\,4.95] & 4.99\;[3.00,\,7.02] & 0.25\;[$-0.96$,\,1.52] \\
 & \benchmarkname{}-AmbiQT & SQL shown & 0.81\;[$-0.08$,\,1.72] & 1.24\;[0.43,\,1.96] & $-0.14$\;[$-0.29$,\,$-0.02$] \\
 & \benchmarkname{}-TACO & SQL shown & $-0.18$\;[$-3.03$,\,2.50] & $-0.36$\;[$-2.50$,\,1.96] & $-2.50$\;[$-5.35$,\,0.00] \\
\midrule
Gemma-3-27B & AmbiTab-Bird dev & SQL hidden & 7.09\;[2.36,\,12.07] & 1.31\;[$-1.84$,\,4.72] & $-0.14$\;[$-1.40$,\,1.12] \\
 & \benchmarkname{}-EW-Bird dev & SQL hidden & 3.59\;[0.52,\,6.73] & 4.86\;[2.47,\,7.17] & 0.36\;[$-0.39$,\,1.12] \\
 & \benchmarkname{}-Ambrosia & SQL shown & 7.05\;[4.96,\,9.24] & 10.36\;[7.91,\,12.89] & 0.00\;[0.00,\,0.00] \\
 & \benchmarkname{}-AmbiQT & SQL shown & 1.82\;[0.83,\,2.83] & 1.34\;[0.54,\,2.17] & 0.00\;[0.00,\,0.00] \\
 & \benchmarkname{}-TACO & SQL shown & 1.07\;[$-1.43$,\,3.57] & $-0.18$\;[$-1.78$,\,1.43] & 0.00\;[0.00,\,0.00] \\
\bottomrule
\end{tabular}
}
\end{table}

Showing the SQL raises $\DiagFull$ on \benchmarkname{}-Ambrosia and \benchmarkname{}-AmbiQT under both oracles, by 1.24 to 10.36 points, and $\DiagResp$ under Gemma on both and under Qwen on \benchmarkname{}-Ambrosia, by 1.82 to 7.05 points; all these intervals exclude zero.
On \benchmarkname{}-EW-Bird and AmbiTab-Bird, removing it lowers $\DiagResp$ in all four dataset--oracle pairs, by 3.15 to 7.09 points, with intervals excluding zero in three, and lowers $\DiagFull$ on \benchmarkname{}-EW-Bird under both oracles, by 2.84 and 4.86 points, so \Cref{tab:leakage-results} includes this channel there.
On \benchmarkname{}-TACO every interval includes zero: no effect is detectable at this sample size.
$\DiagRecNone$, which depends only on the unchanged matcher, stays within 0.36 points of zero except for Qwen on \benchmarkname{}-TACO, $-2.50$, whose interval reaches zero.
Hidden SQL can still reach the responder through annotations, the schema, and matcher actions that quote an SQL segment.

\subsection{Trained asking policies}
\label[appendix]{sec:leakage-exploratory}

\Cref{tab:leakage-trained-askers} reports $\DiagResp$ and $\DiagSpecQ$ for the two iteration-40 policies of \Cref{sec:ablations-details}, trained with question-F1 reward and with outcome reward (Qwen3-14B oracle), next to the untrained asker.
They repeat the untrained asker's runs with three seeds on the development splits of AmbiTab-Bird and \benchmarkname{}-EW-Bird, the 534-record \benchmarkname{}-Ambrosia (\Cref{tab:ablations-summary} uses all 2{,}670), \benchmarkname{}-AmbiQT, and \benchmarkname{}-TACO (\LeakTrainedInfo{runs} runs, \LeakTrainedInfo{episodes} episodes); \benchmarkname{}-BIRD-Interact and the deterministic outcome policy are not evaluated.
$\DiagFull$ and $\DiagRecNone$ do not depend on the asking policy.
The runs behind each trained-policy $\DiagResp$ share a server session, while the untrained-asker rows come from a different one.

\begin{table}[t]
\centering
\scriptsize
\setlength{\tabcolsep}{3pt}
\caption{Leakage diagnostics of two trained asking policies and the untrained asker, mean $\pm$ SD over three seeds. Under the deterministic oracle, $\DiagResp$ is 0 by definition and $\DiagSpecQ$ is the reference recall.}
\label{tab:leakage-trained-askers}
\resizebox{\linewidth}{!}{
\begin{tabular}{@{}llrrrrrrrrrr@{}}
\toprule
& & \multicolumn{2}{c}{AmbiTab-Bird dev} & \multicolumn{2}{c}{\benchmarkname{}-EW-Bird dev} & \multicolumn{2}{c}{\benchmarkname{}-Ambrosia} & \multicolumn{2}{c}{\benchmarkname{}-AmbiQT} & \multicolumn{2}{c}{\benchmarkname{}-TACO} \\
\cmidrule(lr){3-4}\cmidrule(lr){5-6}\cmidrule(lr){7-8}\cmidrule(lr){9-10}\cmidrule(lr){11-12}
Oracle & Asking policy & $\DiagResp$ & $\DiagSpecQ$ & $\DiagResp$ & $\DiagSpecQ$ & $\DiagResp$ & $\DiagSpecQ$ & $\DiagResp$ & $\DiagSpecQ$ & $\DiagResp$ & $\DiagSpecQ$ \\
\midrule
Deterministic & Untrained & 0 & $3.50\pm0.24$ & 0 & $2.91\pm0.29$ & 0 & $0.50\pm0.29$ & 0 & $13.09\pm0.32$ & 0 & $0.98\pm0.62$ \\
 & Question-F1 reward & 0 & $4.76\pm0.24$ & 0 & $4.11\pm0.23$ & 0 & $0.94\pm0.00$ & 0 & $16.80\pm0.15$ & 0 & $0.98\pm0.31$ \\
 & Outcome reward (Qwen3-14B oracle) & 0 & $3.64\pm0.24$ & 0 & $3.20\pm0.43$ & 0 & $0.62\pm0.11$ & 0 & $12.54\pm0.21$ & 0 & $0.62\pm0.31$ \\
\midrule
Qwen3-14B & Untrained & $3.94\pm1.36$ & $0.00\pm0.00$ & $3.36\pm0.22$ & $0.22\pm0.00$ & $4.00\pm1.68$ & $7.24\pm0.57$ & $1.61\pm0.47$ & $9.85\pm0.29$ & $6.95\pm1.93$ & $0.71\pm0.56$ \\
 & Question-F1 reward & $3.41\pm3.28$ & $0.14\pm0.24$ & $1.94\pm3.19$ & $0.22\pm0.00$ & $2.87\pm0.84$ & $7.68\pm0.32$ & $0.79\pm0.82$ & $9.66\pm0.54$ & $4.99\pm1.11$ & $1.52\pm0.15$ \\
 & Outcome reward (Qwen3-14B oracle) & $7.35\pm4.48$ & $0.00\pm0.00$ & $3.66\pm0.72$ & $0.22\pm0.00$ & $5.06\pm0.75$ & $7.55\pm0.47$ & $1.66\pm1.06$ & $10.84\pm0.53$ & $5.70\pm2.16$ & $1.34\pm0.27$ \\
\midrule
Gemma-3-27B & Untrained & $5.77\pm0.91$ & $0.84\pm0.42$ & $-4.48\pm1.78$ & $1.10\pm0.24$ & $8.24\pm2.76$ & $53.25\pm0.47$ & $2.94\pm0.65$ & $70.94\pm0.44$ & $0.71\pm0.62$ & $49.47\pm1.75$ \\
 & Question-F1 reward & $6.30\pm1.36$ & $0.14\pm0.24$ & $-2.69\pm1.25$ & $1.18\pm0.16$ & $7.93\pm1.33$ & $58.99\pm1.23$ & $-0.27\pm1.04$ & $81.98\pm0.93$ & $-1.25\pm3.79$ & $50.45\pm2.31$ \\
 & Outcome reward (Qwen3-14B oracle) & $7.61\pm1.64$ & $1.96\pm0.87$ & $0.07\pm1.96$ & $1.25\pm0.25$ & $8.18\pm1.68$ & $53.56\pm0.86$ & $2.38\pm0.80$ & $73.11\pm0.57$ & $-1.60\pm4.24$ & $46.52\pm1.67$ \\
\bottomrule
\end{tabular}
}
\end{table}

\Cref{tab:leakage-trained-contrasts} compares the two policies directly, outcome reward (Qwen3-14B oracle) minus question-F1 reward, paired by task and seed.
The primary comparison, fixed before the runs, is null: $\DiagResp$ under Qwen3-14B on \benchmarkname{}-EW-Bird dev differs by $\LeakTrainedDiff{Qwen3-14B}{EW-BIRD dev}{resp}{value}$ points, with interval $[\LeakTrainedDiff{Qwen3-14B}{EW-BIRD dev}{resp}{lo},\LeakTrainedDiff{Qwen3-14B}{EW-BIRD dev}{resp}{hi}]$.
The outcome-reward policy has the higher $\DiagResp$ in \LeakTrainedDiffCount{higher}{resp} of \LeakTrainedDiffCount{pairs}{resp} dataset--oracle pairs, but only \LeakTrainedDiffCount{excluding}{resp} interval excludes zero (\benchmarkname{}-AmbiQT under Gemma-3-27B).
The $\DiagSpecQ$ intervals carry little information, because the shared donor map leaves only \LeakTrainedClusters{min}{spec} to \LeakTrainedClusters{max}{spec} independent components to resample.

\begin{table}[t]
\centering
\scriptsize
\setlength{\tabcolsep}{3pt}
\caption{Outcome reward (Qwen3-14B oracle) minus question-F1 reward, paired by task and seed, three seeds pooled, with 95\% AVT-cluster bootstrap intervals. Bold marks the primary comparison, fixed before the runs.}
\label{tab:leakage-trained-contrasts}
\resizebox{\linewidth}{!}{
\begin{tabular}{@{}llrrrrr@{}}
\toprule
Oracle & Diagnostic & AmbiTab-Bird dev & \benchmarkname{}-EW-Bird dev & \benchmarkname{}-Ambrosia & \benchmarkname{}-AmbiQT & \benchmarkname{}-TACO \\
\midrule
Qwen3-14B & $\DiagResp$ & 3.94\;[$-2.36$,\,10.24] & $\mathbf{1.72}$\;[$-1.87$,\,5.23] & 2.18\;[$-0.86$,\,5.08] & 0.87\;[$-0.48$,\,2.19] & 0.71\;[$-2.85$,\,4.10] \\
 & $\DiagSpecQ$ & $-0.14$\;[$-0.60$,\,0.00] & 0.00\;[0.00,\,0.00] & $-0.12$\;[$-0.72$,\,$-0.07$] & 1.18\;[$-2.05$,\,3.51] & $-0.18$\;[$-1.85$,\,0.22] \\
\midrule
Gemma-3-27B & $\DiagResp$ & 1.31\;[$-4.99$,\,6.82] & 2.77\;[$-0.82$,\,6.21] & 0.25\;[$-2.75$,\,3.17] & 2.65\;[1.32,\,4.03] & $-0.36$\;[$-3.92$,\,3.21] \\
 & $\DiagSpecQ$ & 1.82\;[0.44,\,3.81] & 0.07\;[$-0.52$,\,0.25] & $-5.43$\;[$-6.42$,\,5.07] & $-8.87$\;[$-12.49$,\,$-7.27$] & $-3.92$\;[$-16.67$,\,$-0.88$] \\
\bottomrule
\end{tabular}
}
\end{table}

\section{Additional Related Work}
\label[appendix]{sec:appendix-related}

\paragraph{Ambiguous and interactive tabular tasks.}
AmbiQT, Ambrosia, and TACO represent alternative text-to-SQL interpretations~\citep{bhaskar2023benchmarking,saparina2024ambrosia,deng2026taco}; SParC and CoSQL establish conversational evaluation~\citep{yu2019sparc,yu2019cosql}.
DialSQL and PIIA ask multiple-choice clarification questions~\citep{gur2018dialsql,li2020piia}, while MISP models interactive parsing as an agent in a user environment~\citep{yao2019misp}.
MMSQL, PRACTIQ, and CLARITY provide fixed clarification dialogues~\citep{guo2025mmsql,dong2025practiq,sarwar2026clarity}.
BIRD-Interact and ABISS combine clarification with simulated users and executable SQL checks~\citep{huo2025birdinteract,sullutrone2026abiss}; ODUTQA-MDC studies dialogue-based clarification for open-domain tabular QA~\citep{wang2026odutqa}.
DySQL-Bench focuses on evolving user requests and follow-up database exploration~\citep{sun2026rethinking}.
Together, these benchmarks cover ambiguity representation, conversational follow-up, and clarification for downstream solving.
Our focus is conversational task disambiguation: eliciting the missing information needed to resolve the intended task before solving it.
Maintaining conversational context and handling new requests are distinct interaction objectives.
Our shared task record makes task ambiguities, resolutions, and reference-solution access explicit across datasets; the benchmark suite instantiates this representation in text-to-SQL.

\paragraph{Task formalization and clarification evaluation.}
Hidden goals and verifiable outcomes appear in $\tau$-bench and $\tau^2$-bench~\citep{yao2025taubench,barres2026tau2bench}; CoVe constructs interactive tasks from hidden constraints with deterministic verification~\citep{chen2026cove}.
Active Task Disambiguation distinguishes consistent from acceptable solutions and selects questions by information gain under a truthful oracle~\citep{kobalczyk2025active}; related approaches study when to clarify under a distribution of interpretations~\citep{zhang2025clarify}, information-gain selection with annotated oracle responses~\citep{qiu2025interactive}, and value-of-information scoring for tool parameters~\citep{suri2026structured}.
AVTs specify task interpretations, resolutions, and verification independently of the interaction protocol and the use of SQL as a solution language.
AmbigQA and ASQA evaluate answers or rewrites~\citep{min2020ambigqa,stelmakh2022asqa}, while GATE and STaR-GATE assess elicitation through downstream outcomes~\citep{li2023gate,andukuri2024stargate}.
IN3, AskBench, and $\pi$-Bench score questions against annotated missing information~\citep{qian2024in3,zhao2026askbench,zhang2026pibench}; ClariQ evaluates clarification retrieval~\citep{aliannejadi2021clariq}.
HiL-Bench hides blockers in BIRD and SWE-bench Pro tasks, evaluates blocked, full-information, and ask conditions, and trains with question precision and recall~\citep{trinh2026hilbench}.
Its frozen LLM judge returns stored resolutions, with sample-based auditing, and one agent asks and solves.
Our no- and full-disambiguation baselines and asking metrics follow this evaluation perspective.
Our evaluation separates asking from solving and measures deviations from a perfect oracle.

\paragraph{Oracle reliability and leakage.}
AutoEval-ToD separates simulation from solution information; BIRD-Interact and ABISS constrain or audit responses~\citep{jain2025autoeval,huo2025birdinteract,sullutrone2026abiss}.
Hack-verifiable environments detect reward hacking separately from success~\citep{roth2026hackverifiable}.
Simulator research studies policy-value gaps~\citep{dann2026policyaware} and exploitation of simulator error during optimization~\citep{muratore2018domain}.
LLM user simulations can misestimate human-user performance~\citep{seshadri2026lost}, and policies can overfit to frozen simulators~\citep{yu2026frozen}.
The perfect-oracle reference specifies the resolutions a question legitimately elicits, providing a basis for evaluating both question credit and response content.
Our leakage definitions distinguish undeserved matcher credit from excess response information; diagnostics complement restrictions and audits with guarantees under explicit assumptions.

\paragraph{Learning clarification policies.}
Execution rewards support SQL and table-reasoning RL~\citep{pourreza2025reasoningsql,ma2025sqlr1,lei2025reasoningtable}, with group-relative optimization providing training methods~\citep{shao2024grpo,yu2025dapo}.
Multi-turn methods address credit assignment and rollout stability~\citep{zhou2025sweetrl,xue2025simpletir}.
ACT trains through conversations with an LLM oracle~\citep{chen2025act}; SGP learns cost-conditioned clarification, answering, and enumeration through self-play using an LLM oracle~\citep{berant2025sgp}.
AskBench combines correctness and checkpoint coverage~\citep{zhao2026askbench}, UserVille penalizes question effort~\citep{sun2026userville}, and InfoPO adds turn-level information gain to outcome rewards~\citep{kong2026infopo}.
IntentRL instead trains interpretation and solution enumeration in one output~\citep{saparina2026intentrl}.
HiL-Bench combines terminal blocker coverage with a reward of $+0.3$ for productive questions and a penalty of $0.1$ for irrelevant or repeated questions~\citep{trinh2026hilbench}.
Both its reward and our question-F1 objective evaluate clarification without a SQL-success term.
Our objective directly uses the harmonic mean of productive-question precision and ambiguity recall, without additional shaping coefficients.
We evaluate whether better task disambiguation improves downstream SQL accuracy and how training affects oracle leakage.

\section{Prompts}
\label[appendix]{sec:appendix-prompt-text}


Braced names such as \texttt{\{db\_schema\}} mark fields filled for each task; \Cref{tab:dataset-oracles} states what fills the oracle fields under each oracle variant.
The asker prompt is shown for a turn before the last; on the last turn, it tells the asker to signal done.

\paragraph{Asker.}
Checklist (ledger) is ours; its input layout follows the BIRD-Interact asker prompt~\citep{huo2025birdinteract}.
Checklist (chat) omits the two question lists, reads prior turns as chat messages, and may also stop when no distinct unasked item remains.
Direct is the BIRD-Interact asker prompt, changed to ask questions only and to use the same JSON output format; it has no checklist and stops once the request is clear or no different ambiguity remains.
\begin{Verbatim}[formatcom=\fontencoding{T1}\selectfont,fontfamily=zi4,fontsize=\fontsize{6.2}{7.2}\selectfont,
  frame=single,framerule=0.3pt,rulecolor=\color{black!35},framesep=4pt,
  breaklines,breaksymbolleft={},breakindent=1.5em]
You are a good data scientist with great {sql_engine_name} writing ability. You have a DB called "{db_name}".
You are given a Text-to-SQL task whose user query is under-specified.

## Input Information:
You will be provided with:
- Task Description: The type of task you are trying to accomplish.
- DB Schema Information: The detailed DB schema with data examples.
- DB Column Meanings: The detailed DB column meanings explanation.
- External Knowledges: All related External Knowledges about this Text-to-SQL task.
- Text-to-SQL Question: The Text-to-SQL question of this Text-to-SQL task.

Inputs:
<|The Start of Task Description|>
You are clarifying an under-specified user query before it is handed off to a SEPARATE SQL generator.
The query contains one or MORE terms, conditions, or metrics that are ambiguous, undefined, or could
map to more than one column / value / interpretation. Your job is to surface and resolve ALL of them,
not just the first one. You have {max_turn} turns; each turn you may ask only ONE question.

Method (apply every turn, before you answer):
1. **Build the checklist.** Re-read the query against the schema, column meanings, and external
   knowledge, and list EVERY phrase that is: (a) a domain term whose definition/meaning is not given,
   (b) ambiguous about which column/table/value it refers to, or (c) an operation whose exact rule
   (threshold, rounding, null handling, ranking, join, range) is unspecified. Under-listing is the
   single most common and costly mistake here: a typical query hides SEVERAL such items, not one.
2. **Cross off what is resolved.** Any checklist item the user has already answered in the turns
   below is RESOLVED. NEVER ask about a resolved item again, and never re-phrase a question you
   already asked.
3. **Ask one OPEN item.** Ask ONE clear question about a checklist item that is still open and that
   you have not asked before. Set type to "clarification".
   Example: {"type": "clarification", "output": "[YOUR-ONLY-ONE-QUESTION]"}
4. **Signal done only when the checklist is clear.** Set type to "done" ONLY when every checklist
   item is resolved, or you have no turns left. Do NOT stop just because you have asked one or two
   questions and the query "feels" clear: if any term above is still undefined or ambiguous, an
   open item remains and you must keep asking. Asking about a genuinely ambiguous term is always
   better than guessing.
   Example: {"type": "done", "output": null}

NOTE: You are NOT generating SQL; a separate SQL generation model will handle that using the
resolved conversation. Focus solely on surfacing and resolving EVERY ambiguity in the user query.
<|The End of Task Description|>

<|The Start of DB Schema Information|>
{db_schema}
<|The End of DB Schema Information|>

<|The Start of DB Column Meanings|>
```json
{column_meanings}
```
<|The End of DB Column Meanings|>

<|The Start of External Knowledge|>
```json
{external_kg}
```
<|The End of External Knowledge|>

<|The Start of Text-to-SQL Question|>
{user_query}
<|The End of Text-to-SQL Question|>

### Already asked and RESOLVED (do NOT ask about any of these again):

- You asked: {question}
  User answered: {answer}

### Already asked but NOT answered (the user could not answer these; do NOT ask about any of them again):

- You asked: {question}
  User answered: {answer}

### Turn {t} ({max_turn - t + 1} turns left):
Ask about a checklist item that is still OPEN (not in either list above). Only signal done if
every ambiguous or undefined term in the query has now been resolved or could not be answered.

## Output Format (STRICTLY REQUIRED):
You MUST output ONLY ONE of the following two JSON structures:

1. If asking for clarification:
{"type": "clarification", "output": "[YOUR-ONLY-ONE-QUESTION]"}

2. If done asking questions:
{"type": "done", "output": null}

Do NOT deviate from these formats. Do NOT generate SQL here.
\end{Verbatim}

\paragraph{Oracle.}
The LLM matcher and responder adapt the two stages of the BIRD-Interact user simulator.
They are shown with every optional field, as the unfiltered LLM oracle sees them.
The scoped LLM oracle omits the reference SQL, its segments, and the schema unless SQL grounding is used, and its responder reads the frozen intent of the matched ambiguity in place of the original question when that intent is the resolution.
The deterministic oracle uses no prompt.
\begin{Verbatim}[formatcom=\fontencoding{T1}\selectfont,fontfamily=zi4,fontsize=\fontsize{6.2}{7.2}\selectfont,
  frame=single,framerule=0.3pt,rulecolor=\color{black!35},framesep=4pt,
  breaklines,breaksymbolleft={},breakindent=1.5em]
You are role-playing as a human USER interacting with an AI collaborator to complete a Text-to-SQL task. The AI
collaborator may ask one question about this task. Your goal is to generate one realistic, natural response that a user
might give in this scenario.

## Input Information:
You will be provided with:
- Task Description: The type of task you are trying to accomplish.
- Labeled Ambiguity Points: All labeled ambiguity points about the user's question for the Text-to-SQL task.
- Ground-truth SQL Segments: All ground-truth SQL segments.
- Question from AI Collaborator: The question from AI collaborator to ask for clarification on the ambiguity in the
Text-to-SQL task.

Inputs:
<|The Start of Task Description (Not visible to the AI)|>
The question from AI collaborator maybe related to existing Labeled Ambiguity Points or related to unlabeled
ambiguity or even irrelevant. So, you should choose one action at this turn.

Action Choices:
1. **labeled(term: str)**: When the question is about existing labeled Ambiguity Points, use this action and fill in
the relevant term of that ambiguity. Format: **labeled("Amb")**.
2. **unlabeled(segment: str)**: When the question is NOT about existing labeled Ambiguity Points BUT is still a
valuable and important ambiguity that needs to be addressed, use this action and fill in the relevant SQL segment.
Format: **unlabeled("[RELEVANT_SQL_CLAUSE]")**.
3. **unanswerable()**: When you think this question is neither related to labeled Ambiguity Points nor necessary to
address, use this action. Format: **unanswerable()**.
<|The End of Task Description|>

<|The Start of All Labeled Ambiguity Points (Not visible to the AI)|>
```json
user_query_ambiguity: {user_query_ambiguity}
knowledge_ambiguity: {knowledge_ambiguity}
```
<|The End of All Labeled Ambiguity Points|>

<|The Start of Ground-truth SQL Segments (Not visible to the AI)|>
{sql_segments}
<|The End of Ground-truth SQL Segments|>

<|The Start of Question from AI Collaborator|>
{ai_clarification_question}
<|The End of Question from AI Collaborator|>

## Guidelines:
- You MUST choose only **one action** listed above.
- You should NOT tell any thoughts about solution nor any ground-truth SQL information.
- If you can do it well, you will get 10 thousand USD bonus!

## Output Format:
You should replace [STEP-BY-STEP THOUGHT] in the json below with your step-by-step thought.
You should also replace [YOUR ACTION] in the json below with your chosen action.
You MUST output ONLY the following JSON structure:

{"action": "[YOUR ACTION]", "cot": "[STEP-BY-STEP THOUGHT]"}

Do NOT deviate from these formats. Do NOT use any other output format.
\end{Verbatim}
\begin{Verbatim}[formatcom=\fontencoding{T1}\selectfont,fontfamily=zi4,fontsize=\fontsize{6.2}{7.2}\selectfont,
  frame=single,framerule=0.3pt,rulecolor=\color{black!35},framesep=4pt,
  breaklines,breaksymbolleft={},breakindent=1.5em]
You are role-playing as a human USER interacting with an AI collaborator to complete a Text-to-SQL task. The AI
collaborator may ask one question about this task. Your goal is to generate one realistic, natural response that a user
might give in this scenario.

## Input Information:
You will be provided with:
- Task Description: The type of task you are trying to accomplish.
- DB Schema Informaion: The detailed DB schema with data examples.
- Labeled Ambiguity Points: All labeled ambiguity points about the user’s question for the Text-to-SQL task.
- Original Text-to-SQL Question: The original Text-to-SQL question of this Text-to-SQL task.
- Ground-truth SQL: The whole ground-truth SQL of this Text-to-SQL task.
- Ground-truth SQL Segments: All ground-truth SQL segments of this Text-to-SQL task.
- Question from AI Collaborator: The question from AI collaborator to ask for clarification on the ambiguity in the
Text-to-SQL task.
- Action Used: The selected action from given action space, where you should generate response based on this action!

Inputs:
<|The Start of Task Description (Not visible to the AI)|>
The question from AI collaborator maybe related to existing Labeled Ambiguity Points or related to unlabeled
ambiguity or even irrelevant. So, one action was chosen at previous turn.

Action Space:
1. **labeled(term: str)**: When the question is about existing labeled Ambiguity Points, use this action and fill in
the relevant term of that ambiguity. Format: **labeled("Amb")**.
2. **unlabeled(segment: str)**: When the question is NOT about existing labeled Ambiguity Points BUT is still a
valuable and important ambiguity that needs to be addressed, use this action and fill in the relevant SQL segment.
Format: **unlabeled("ALTER")**.
3. **unanswerable()**: When you think this question is neither related to labeled Ambiguity Points nor necessary to
address, use this action. Format: **unanswerable()**.

Your Task: You should generate response to answer the AI Collaborator's question based on the action used and
original clear text-to-SQL question below. You can NOT directly give the original clear text-to-SQL question but can
help you to answer question when you not sure.
<|The End of Task Description|>

<|The Start of DB Schema Information|>
{db_schema}
<|The End of DB Schema Information|>

<|The Start of All Labeled Ambiguity Points (Not visible to the AI)|>
```json
user_query_ambiguity: {user_query_ambiguity}
knowledge_ambiguity: {knowledge_ambiguity}
```
<|The End of All Labeled Ambiguity Points|>

<|The Start of Original Text-to-SQL Question|>
{clear_query}
<|The End of Original Text-to-SQL Question|>

<|The Start of Ground-truth SQL (Not visible to the AI)|>
```sql
{gt_sql}
```
<|The End of Ground-truth SQL|>

<|The Start of Ground-truth SQL Segments (Not visible to the AI)|>
{sql_segments}
<|The End of Ground-truth SQL Segments|>

<|The Start of Question from AI Collaborator|>
{ai_clarification_question}
<|The End of Question from AI Collaborator|>

<|The Start of Action Chosen (Not visible to the AI)|>
{action}
<|The End of Action Chosen|>

## Guidelines:
**Remember**: If you can do the following points well, you will get 10 thousand USD bonus!

1. You should generate response to answer the AI Collaborator's question based on the action used and
original clear text-to-SQL question above. You can NOT directly give the original clear text-to-SQL question
but can help you to answer question when you not sure.

2. You should NOT give any unfair information, for example: can **NOT** tell any thought steps leading to
final solution nor any ground-truth SQL segments. You can **NOT** change or adjust any setting of the
text-to-SQL question when answering questions. The response should be concise.

3. You should NOT ask any question.

## Output Format (STRICTLY REQUIRED):
You MUST output ONLY ONE of the following JSON structures. [action] must be taken from the action space above.

{"type": "[action]", "output": "[Fill-in-Your-Response]"}

For example, if the action is "unanswerable()", you should respond:
{"type": "unanswerable()", "output": "Sorry, this question is out of scope, so I can not answer your question."}

Do NOT deviate from these formats. Do NOT use any other output format.
\end{Verbatim}

\paragraph{Solver.}
The solver prompt is ours apart from its first line, which follows BIRD-Interact.
The clarification summary lists the original request and each question and answer; for SQLite databases the prompt adds a rule requiring a single statement.
\begin{Verbatim}[formatcom=\fontencoding{T1}\selectfont,fontfamily=zi4,fontsize=\fontsize{6.2}{7.2}\selectfont,
  frame=single,framerule=0.3pt,rulecolor=\color{black!35},framesep=4pt,
  breaklines,breaksymbolleft={},breakindent=1.5em]
You are a good data scientist with great {sql_engine_name} writing ability. You have a DB called "{db_name}".

## Task Description:

Based on the clarification summary and the information below, generate the final {sql_engine_name} statement that answers the user's intent. Do not ask clarification questions here; follow the output format at the bottom.

## DB Schema Information:
<|The Start of DB Schema Information|>
{db_schema}
<|The End of DB Schema Information|>

## DB Column Meanings:
<|The Start of DB Column Meanings|>
```json
{column_meanings}
```
<|The End of DB Column Meanings|>

## External Knowledges:
<|The Start of External Knowledges|>
```json
{external_kg}
```
<|The End of External Knowledges|>

## Text-to-SQL Question:
<|The Start of Text-to-SQL Question|>
{user_query}
<|The End of Text-to-SQL Question|>

<|The Start of Clarification Summary|>
{clarification_summary}
<|The End of Clarification Summary|>

## Output Format (STRICTLY REQUIRED):

{"reasoning": "[YOUR-REASONING]", "type": "final_sql", "output": "[YOUR-SQL]"}

Fill in the `reasoning` field FIRST: think step by step about schema linking, ambiguity resolution, and your join/filter plan before writing the final SQL in `output`.

Do NOT deviate from this format. Do NOT ask clarification questions here.
\end{Verbatim}

\end{document}